\def\year{2021}\relax
\documentclass[letterpaper]{article} 
\usepackage{aaai21}  
\usepackage{times}  
\usepackage{helvet} 
\usepackage{courier}  
\usepackage[hyphens]{url}  
\usepackage{graphicx} 
\usepackage{natbib}  
\usepackage{caption} 

\usepackage{latexsym}

\usepackage{subfigure} 
\usepackage{multirow}
\usepackage{float}
\usepackage{rotfloat}
\usepackage{epstopdf}
\usepackage{microtype}
\usepackage{mathtools} 

\usepackage{booktabs}       
\usepackage{amsfonts}       
\usepackage{nicefrac}       
\usepackage{algorithm}
\usepackage{algorithmic}
\usepackage[algo2e,ruled,linesnumbered]{algorithm2e} 
\usepackage{soul}
\usepackage{bm}
\usepackage{cite}
\usepackage{color}
\usepackage{amsmath}
\usepackage{amsthm}
\usepackage[switch]{lineno}
\usepackage{fancyhdr}  

\newtheorem{theorem}{Theorem}

\newtheorem{definition}{Definition}

\newtheorem{lemma}{Lemma}

\title{EvaLDA: Efficient Evasion Attacks Towards Latent Dirichlet Allocation}

\author {

     Qi Zhou,\textsuperscript{\rm 1}\thanks{These authors contributed equally to this work.}
     Haipeng Chen, \textsuperscript{\rm 2}\footnotemark[1]\thanks{Corresponding authors.}
     Yitao Zheng, \textsuperscript{\rm 1}
     Zhen Wang \textsuperscript{\rm 1}\footnotemark[2] \\
}
\affiliations {
 \textsuperscript{\rm 1} School of Cyberspace, Hangzhou Dianzi University, Hangzhou 310018, China \\
 \textsuperscript{\rm 2} Center for Research on Computation and Society \& Department of Computer Science,\\ Harvard University, Cambridge 02138, MA, USA\\
 zhouqi@hdu.edu.cn, hpchen@seas.harvard.edu, zhengyitao@hdu.edu.cn, wangzhen@hdu.edu.cn
}

\begin{document}
\maketitle

\begin{abstract}
	As one of the most powerful topic models, Latent Dirichlet Allocation (LDA) has been used in a vast range of tasks, including document understanding, information retrieval and peer-reviewer assignment. Despite its tremendous popularity, the security of LDA has rarely been studied. This poses severe risks to security-critical tasks such as sentiment analysis and peer-reviewer assignment that are based on LDA. In this paper, we are interested in knowing whether LDA models are vulnerable to adversarial perturbations of benign document examples during inference time. 
	We formalize the evasion attack to LDA models as an optimization problem and prove it to be \textit{NP-hard}. We then propose a novel and efficient algorithm, EvaLDA to solve it. We show the effectiveness of EvaLDA via extensive empirical evaluations. For instance, in the NIPS dataset, EvaLDA can averagely promote the rank of a target topic from 10 to around 7 by only replacing $1\%$ of the words with similar words in a victim document. Our work provides significant insights into the power and limitations of evasion attacks to LDA models.
\end{abstract}

\section{Introduction}
Latent Dirichlet Allocation (LDA) is one of the most powerful topic models. Due to its superiority in discovering the latent topics of documents, it has been the underlying technique in a vast range of tasks, including scientific publication understanding~\citep{griffiths2004finding,talley2011database}, information retrieval~\citep{harvey2013building}, and peer-reviewer assignment~\citep{charlin2013toronto,liu2014robust}. Despite the great successes, the security of LDA-based systems has rarely been investigated. Security-critical tasks based on LDA, such as sentiment analysis and peer-reviewer assignment, are therefore at high risk of being manipulated by carefully designed adversarial attacks. For instance, by promoting the rank of a target topic for an article, it may end up in hands of colluded reviewers who are experts in the target topic. Recent studies show that machine learning (ML) methods are highly vulnerable to adversarial samples with simple and yet evasive perturbations to the benign input data~\citep{szegedy2014intriguing,goodfellow2015explaining}. In this paper, we are interested in knowing -- are LDA models vulnerable to crafted adversarial samples?

There are in general two types of adversarial attacks to ML models, namely poisoning attacks and evasion attacks~\citep{barreno2010security}. Poisoning attacks aim to mislead a ML model by manipulating the training data, while evasion attacks craft malicious samples to mislead a trained ML model during test time. While there has been pioneer work~\citep{mei2015security} studying poisoning attacks to LDA models, no efforts have been made on investigating evasion attacks to LDA models. We present the first study of this kind. More specifically, we study evasion attacks to LDA models that are trained/inferred using the predominant Monte Carlo simulation based method called Collapsed Gibbs Sampling (CGS).\footnote{Another popular training/inference method for LDA models is the Variational Inference (VI) method. The two methods are distinct in the training/inference procedure. We focus on CGS based LDA models and leave VI based LDA models for future work.}

Though evasion attacks of general ML models have been massively studied~\citep{nguyen2015deep,carlini2017towards,xiao2018generating,ling2019deepsec}, the unique characteristics of LDA models poses several new challenges on performing evasion attacks to them. First, the input of an LDA model is a document, and is essentially textual data. Different from evasion attacks in domains like images~\citep{nguyen2015deep,carlini2017towards}, textual data are intrinsically discrete, thus incurring higher difficulties to find optimal perturbations to benign samples. Second, because the input of an LDA model is an entire document, the strategy space of crafting a malicious document is huge, especially with a large document. Last and most importantly, unlike general ML models which usually infer the output via a simple forward pass (e.g., in neural networks and decision trees) from the input, the inference procedure of CGS-based LDA models involves a much more complicated Monte Carlo simulation process. This makes the gradient-based style attacks infeasible and trial-and-error style attacks extremely inefficient. \footnote{Note that the inference procedure for VI based LDA models is essentially an optimization. Therefore evasion attack to VI based LDA models is equally, if not more challenging.}
	
In recent years, studies on evasion attacks to ML models for NLP tasks have drawn great attention from both the ML and NLP communities~\citep{papernot2016crafting,ren2019generating,jin2020bert,behjati2019universal}. 
However, most of them focus on sentence level adversarial samples, where the strategy space is much smaller than document level inputs of LDA models. More importantly, as discussed above, a key difference of LDA and general ML models is the computationally expensive inference procedure of a new input. Due to these reasons, previous works on evasion attacks to ML based NLP models cannot be applied to our problem.

We present the first study on evasion attacks to LDA. Note that our goal is not to attack the LDA models, but to provide better insights of the power and limitations of such attacks, so that appropriate defense mechanisms can be designed to prevent the attacks. We first formulate the optimal adversarial attack to LDA models as a combinatorial optimization problem. We prove that the formulated problem is \textit{NP-hard} by reducing the combinatorial optimization problem with cardinality constraints (COPCC)~\cite{bruglieri2006annotated} to it. 

To efficiently solve the formulated optimization problem, we propose a novel algorithm, efficient Evasion attacks towards LDA models (EvaLDA), with two key novelties. First, to handle the high computational cost of the inference procedure and thus the difficulty of estimating an adversarial perturbation's marginal contribution to the attack objective, we derive an efficient analytical estimate. The analytical estimate builds upon a carefully designed surrogate procedure for the CGS-based LDA inference. Second, due to the underlying combinatorial nature of the problem and document-level large attack strategy space, we design an efficient greedy algorithm in EvaLDA to solve the optimization problem.

\noindent\textbf{Our contributions} i) We are the first to study evasion attacks to LDA models. The study has significant practical impact as it provides early warning on the security vulnerability of LDA models to users and service providers. The derivation of the attack strategies provides insights into designing defense strategies against such attacks. ii) We formulate the evasion attack to LDA models as a combinatorial optimization problem and prove its \textit{NP-hardness}. iii) We propose EvaLDA, a novel evasion attack algorithm to LDA models with key technical novelties in handling the various new challenges in the new problem. iv) We conduct extensive empirical evaluations that prove the effectiveness of EvaLDA on two distinct datasets (i.e., NIPS and AP\footnote{Links to the datasets are in the experiment section.}) and a large variety of problem settings. For instance, results on the NIPS dataset show that, by replacing only $1\%$ of the words to similar words in a victim document, EvaLDA can averagely promote the rank of a target topic from 10 to around 7. 
We also show the effectiveness of EvaLDA via a case study. The code of this paper can be found at \url{https://github.com/tools-only/Evasion-Attack-against-LDA-Model}.

\section{Related Work}\label{sec:related_work}
Our work lies in the thread of research that studies evasion attacks to ML models for NLP tasks. Similar to most prior works, we aim at two goals in evasion attacks to textual data: to deteriorate the performance (e.g., accuracy) of the victim model and to keep the adversarial perturbations evasive. 
	
\noindent\textbf{Gradient-based attack.}
\citet{papernot2016crafting} make an early attempt at adversarial attacks to recurrent neural networks (RNNs) based NLP models. They calculate words' marginal contributions of classification using Fast Gradient Signed Method FGSM~\citep{goodfellow2015explaining} and Forward Derivative~\citep{papernot2016limitations}. 
FGSM is also used in \citet{samanta2018towards}. They consider replacement, insertion and deletion operations to generate adversarial samples. 
In finding word replacements, they build a candidate replacement pool using synonyms, typos and genre specific keywords. \citet{liang2018deep} calculate the character-level marginal contribution of text classification from cost gradient, and treat a word as multiple characters. This type of attacks are called ``white-box'' attacks as computing gradients requires knowledge about the victim model.

\noindent\textbf{Black-box attack.} Opposed to white-box attacks which require full knowledge about the victim model or its gradient information, black-box attacks only require knowledge about the outputs of the victim model via queries.
\citet{ren2019generating} study black-box attacks on text classification models. They calculate words' marginal contribution of text classification by masking out the underlying words. They also propose a scheme called probability weighted word saliency (PWWS), which considers lexical, grammatical, and semantic constraints in generating adversarial samples. 
Similar word selection method is used in \citep{jin2020bert}, where adversarial attacks to BERT model~\citep{devlin2019bert} are studied. They use word embeddings to extract synonyms of the selected words for replacement, subject to post-processing which ensures the semantics and syntax of the adversarial sentence. 
\citet{alzantot2018generating} propose an adversarial text generation approach based on genetic algorithm, while \citet{zhang2019generating} use a Markov chain Monte Carlo sampling approach. 
	
\noindent\textbf{Attack from embedding space.}
Compared with the above methods which work directly in discrete word or character space, \citet{miyato2017adversarial} propose to generate adversarial texts from the continuous embedding space. 
Similarly in~\citep{behjati2019universal}, adversarial perturbation is performed in the embedding space using loss gradient maximization. The obtained optimal perturbation is then projected back into the vocabulary space to find a feasible replacement. \citet{sato2018interpretable} improve over~\citep{miyato2017adversarial} by interpretable perturbation approach which restricts perturbation to be existing points in the embedding space. 
	

\noindent\textbf{Poisoning attack to LDA.}	
Most of the above studies on security aspects of NLP models focus on deep ML models such as RNNs~\citep{mikolov2010recurrent} and Transformers~\citep{vaswani2017attention}, while little attention has been paid to the security of LDA models. A more related work is~\citet{mei2015security}, which also studies adversarial attack to LDA models. However, it is notably different from our work. First, they study poisoning attack, where the entire training corpus are target documents, whereas we study evasion attack. From a practical point of view, poisoning attacks are more difficult than evasion attacks, as it requires edit access to the entire training corpus. Second, they focus on VI~\citep{blei2003latent} based parameter estimation, while we study CGS~\citep{griffiths2004finding} based parameter estimation.

\section{Preliminaries}
As a reminder, LDA is a generative model consisting of $K$ latent topics. It assumes the following generation process of a document $m$ with word length $N$: i) Draw a ``document-topic'' distribution $\theta_m$ from a Dirichlet prior $\theta_m\sim \text{Dir}(\alpha)$. ii) For each topic $z_k$ in the $K$ topics, draw a ``topic-word" distribution $\varphi_k$ from another Dirichlet prior: $\varphi_k \sim \text{Dir}(\beta)$. Here $\varphi_k$ specifies the probability distribution that the $v$-th word in a vocabulary $\mathcal{V}^{vocab}$ belongs to topic $z_k$. $\alpha$ and $\beta$ are hyperparameters implying prior knowledge about the shape of the Dirichlet distributions. iii) For each of the $N$ positions of the document, first sample a topic $k \sim \theta_m$, and then sample a word $w\sim \varphi_k$.
	
Given a set of $M$ training documents $\mathbf{w}=\langle \mathbf{w}_m \rangle, m=1,\ldots,M$, the training process works by finding the parameter values $\theta = \langle \theta_m \rangle, m=1,\ldots,M$ and $\varphi=\langle\varphi_k \rangle, k=1,\ldots,K$ that maximizes the posterior $P(\mathbf{z}|\mathbf{w};\theta,\varphi)$.  $\mathbf{z}=\langle z_k \rangle, k=1,\ldots, K$ is a vector of latent topics. Since calculating this posterior is computational intractable, two common approximations methods are used, namely variational inference (VI)~\citep{blei2003latent}
and collapsed Gibbs sampling (CGS)~\citep{griffiths2004finding}. In this paper, we focus on attacks to the CGS-based LDA models.
	
\noindent\textbf{CGS-based inference.} 
Different from training time where both the document-topic distribution $\theta$ and topic-word distribution $\varphi$ are updated, during inference time, topic-word distribution $\varphi$ is usually fixed~\citep{heinrich2005parameter}. This is because one input sample during inference has very limited effect to $\varphi$ compared with the entire training corpus. 
	
CGS-based inference is essentially a Markov Chain Monte Carlo (MCMC) method. Given a document $\mathbf{w}_m = \langle w_{mi} \rangle, i = 1,\ldots,N_m$, it works iteratively. In each iteration, it goes through each word $w_{mi}$ in the document $\mathbf{w}_m$. At word $w_{mi}$, it first samples the topic $z_k$ to be allocated at this position according to the following \textit{full conditional distribution}~\citep{griffiths2004finding}:
\begin{align}\label{eq:full_cond}
p\left(z^{i}\!=\!z_k | \mathbf{z}^{-i}\!,\! \mathbf{w}_m\right)\! \propto\! \varphi_{ki} \frac{{N}_{{m} {k}}+\alpha}{\sum_{{k'}=1}^{{K}}\!N_{{mk'}}\!+\!K\alpha}
\end{align}
where $\mathbf{z}^{-i}$ is the topic allocations of all the other words except $w_{mi}$. ${N}_{{m k}}$ is the count of topic $z_k$ being sampled in the document before word $w_{mi}$. 
After sufficient iterations (the ``burn-in'' period), the document-topic distribution $\theta_m$ is calculated:
\begin{align}\label{eq:topic_dist}
\theta_{mk} = \frac{N_{mk}+\alpha}{\sum_{k'=1}^{K}N_{mk'}+K\alpha},\ \forall k=1,\ldots, K
\end{align}
We refer to~\citep{griffiths2004finding,heinrich2005parameter} for a detailed description for the training and inference procedures of LDA models using CGS.

\section{Problem Formulation}\label{sec:formulation}
Recall that in inference time, LDA maps an unseen test document $\mathbf{w}_m$ into a document-topic distribution $\theta_m$. An evasion attack to LDA is to make perturbations to the victim document $\mathbf{w}^{vic}$ and generate an adversarial sample $\mathbf{w}^{adv}$, so that the inferred document-topic distribution $\theta^{vic}$ of the victim document is changed to $\theta^{adv}\neq \theta^{vic}$. More specifically, we consider the following adversarial word replacement attack.\footnote{We only consider word replacement operation to make the paper focused. We leave the extension to other types of operations (e.g., insertion and deletion) as future work.}

\begin{definition}\label{def:attack-lda}
The adversarial word replacement attack on LDA (\textit{\textbf{Attack-LDA}}) problem aims to replace a subset of words $\mathcal{W}$ by a new set of words $\mathcal{W}'$, so that the victim document $\mathbf{w}^{vic}$ is changed into an adversarial document $\mathbf{w}^{adv}$, and a certain attack objective $Q(\mathcal{W}, \mathcal{W}')$ is maximized.
\end{definition}
Here we refer to $(\mathcal{W},\mathcal{W}')$ as the \textit{attack strategy}. A word $w\in \mathcal{W}$ is called a \textit{target} word, and a word $w'\in \mathcal{W}'$ is called \textit{replacement} word. Note that the specific form of the attack objective $Q(\mathcal{W}, \mathcal{W}')$ is dependent on the goal of attack. For instance, in a \textit{rank promotion attack}, the goal is to maximize the increase of probability for a target topic $z_k$, i.e., 
\begin{align}\label{eq:promote}
Q(\mathcal{W}, \mathcal{W}') = \theta^{adv}_k-\theta^{vic}_k
\end{align} 
In an \textit{rank demotion attack}, the goal is to maximize the decrease of probability for the target topic $z_k$, i.e., 
\begin{align}\label{eq:demote}
Q(\mathcal{W}, \mathcal{W}') = \theta^{vic}_k-\theta^{adv}_k
\end{align}

\noindent\textbf{Attack budget.} As introduced in the Related Work Section, another key aspect of an evasion attack is to make minimal changes to the victim document, so as to make the attack \textit{evasive}. This introduces a budget constraint of the form $D_w(\mathbf{w}^{vic},\mathbf{w}^{adv}) \leq \delta$, where $\delta$ is a threshold, and $D_w(\mathbf{w}^{vic},\mathbf{w}^{adv})$ denotes the distance of the two documents $\mathbf{w}^{vic}$ and $\mathbf{w}^{adv}$.

In practical implementation, we break it down into two types of constraints. One type of constraint specifies that the distance $D(w,w')$ of the target word $w$ and the replacement word $w'$ cannot exceed a threshold $\sigma$: 
$D(w,w') \leq \sigma, \  \forall \ (w,w') \in (\mathcal{W}, \mathcal{W}').$
An example distance measure is the cosine distance of the two words $w$ and $w'$ in the vector space using word embeddings ~\citep{bojanowski2017enriching}. The other type of constraint restricts that the number of words being replaced cannot exceed a certain percentage $\kappa$ of the total number of words in the victim document: $
| \mathcal{W}|  \leq |\mathbf{w}^{vic}| \cdot  \kappa.$

Formally, we represent the \textit{Attack-LDA} problem in Definition~\ref{def:attack-lda} as the following optimization problem:
\begin{align}
\label{eq:obj}\max_{\mathcal{W}, \mathcal{W}'} \quad & Q(\mathcal{W}, \mathcal{W}') \\
\label{eq:constr_dist} \text{s.t.} \quad &  D(w,w') \leq \sigma, \quad \forall \ (w,w') \in (\mathcal{W}, \mathcal{W}')\\
\label{eq:constr_capacity}&| \mathcal{W}|  \leq |\mathbf{w}^{vic}| \cdot  \kappa
\end{align}

\begin{theorem}\label{thm:NP-hardness}
Given an oracle that tells the algorithm the explicit value of $Q(\mathcal{W},\mathcal{W}')$ for an attack strategy $(\mathcal{W},\mathcal{W}')$, the \textit{Attack-LDA} problem formulated above is \textit{NP-hard}.
\end{theorem}
\begin{proof}
	The key idea is to reduce the (binary) combinatorial optimization problem with cardinality constraints (COPCC, which is proven to be $\mathcal{NP}$-\textit{hard}~\cite{bruglieri2006annotated}) to our defined \textit{Attack-LDA} problem. 
	
	An \textit{arbitrary} instance of COPCC can be expressed as:
	\begin{align}\label{eq:copcc-instance}
	\min_x  f(x) \ \ 
	\text{s.t.} \ \  x\in \{0,1\}^d : ||x||_0\leq C,
	\end{align} where $x$ is a $d$-dimensional binary indicator vector which corresponds to the selection of items. That is, $x_i=1$ indicates the $i$-th item is selected and otherwise not. $|| x||_0$ denotes $l$-0 norm of $x$. 
	
	We then construct a \textit{special} instance of the \textit{Attack-LDA} problem as follows. We first let 
	\begin{align*}
	Q^*(\mathcal{W}) &= \max_{\mathcal{W}'} Q(\mathcal{W},\mathcal{W}')\\
	\text{s.t.} \quad &D(w,w')\leq \sigma, \quad \forall (w,w') \in (\mathcal{W}, \mathcal{W}').
	\end{align*} The \textit{Attack-LDA} problem is then transformed as 
	\begin{align}\label{eq:attack-lda-instance}
	\max_{\mathcal{W}} \quad Q^*(\mathcal{W}) \ \ 
	\text{s.t.} \ \  | \mathcal{W}| \leq |\mathbf{w}^{vic}| \cdot  \kappa
	\end{align}
	Because finding the optimal $w'$ for a given $w$ requires enumerating the vocabulary space only once, therefore finding $Q^*(\mathcal{W})$ is polynomial ($\mathcal{O}(|\mathbf{w}^{vic}|\cdot |\mathcal{V}^{vocab}|)$) in the document size $|\mathbf{w}^{vic}|$ and the vocabulary size $|\mathcal{V}^{vocab}|$. 
	
	We construct the correspondence as: $d \xleftrightarrow\ |\mathbf{w}^{vic}|$, $x_i=1 \xleftrightarrow\ w_i\in \mathcal{W}$, $f(x) \xleftrightarrow\ Q^*(\mathcal{W})$, $C \xleftrightarrow\ |\mathbf{w}^{vic}| \cdot  \kappa$. In this sense, we can easily get that if $x$ is an "yes" instance of Eq.\eqref{eq:copcc-instance}, then the corresponding $\mathcal{W}$ is an "yes" instance of Eq.\eqref{eq:attack-lda-instance} and vice-versa.
\end{proof}

\section{EvaLDA: Evasion Attack towards LDA}
We can see that a key challenge of solving Eqs.\eqref{eq:obj}-\eqref{eq:constr_capacity} is to obtain the objective value $Q(\mathcal{W},\mathcal{W}')$ given an attack strategy $(\mathcal{W},\mathcal{W}')$. For any type of attack objectives in Eqs.~\eqref{eq:promote}-\eqref{eq:demote}, the key is then to compute the document-topic distributions $\theta^{vic}$ and $\theta^{adv}$ of the victim and adversarial documents. However, this is computationally expensive for CGS-based inference procedure as it involves hundreds or thousands of simulation iterations. 
To handle this challenge, we design a surrogate sampling-based inference procedure, from which we derive an analytical estimate of the document-topic distribution. 

Another key challenge of solving the formulated optimization problem, as shown in the \textit{NP-hardness} of the problem in Theorem~\ref{thm:NP-hardness}, is the high computational complexity that arises from its combinatorial nature. To scale up EvaLDA, we propose a greedy algorithm, where we assume independence of effects on the objective function $Q(\mathcal{W},\mathcal{W}')$ for different target-replacement word pairs $(w,w')$.

\subsection{Efficient Estimate of Document-Topic Distribution}
\noindent\textbf{Surrogate inference procedure.}
Similar to the original CGS-based inference, the surrogate inference works iteratively, where each iteration goes over all the positions of the test document $\mathbf{w}_m$ once. The \textit{key difference} is that the topic sampling of each position happens simultaneously in each iteration. Denote the set of unique words (i.e., vocabulary) in the test document as $\mathcal{V}$, for each unique word $v\in \mathcal{V}$, it is updated $n_v$ times, where $n_v$ is the number of times $v$ appears in the test document.\footnote{By default, we use $w$ to denote a \textit{position} word in the document, and $v$ to denote a unique word in the vocabulary. Therefore in a document, there is a one to many correspondence from $v$ to $w$.} Consequently, the calculation of topic counts happens after the simultaneous sampling step.


\begin{lemma}\label{lemma:recursive_theta}
	In the above designed surrogate inference procedure, when $\alpha\rightarrow 0$,\footnote{Be reminded that the hyper-parameter $\alpha$ of the Dirichlet distribution can be interpreted as a regularization term of the document-topic distribution $\theta_m$ based on prior knowledge. In practice, $\alpha$ is a very small value that is approximately equal to $1/K$ where $K$ is the number of topics. Therefore the assumption is reasonably made.} there exists a recursive definition of the topic distribution $\theta_k^t$ for each topic $k=1,\dots,K$:
	\begin{align}\label{eq:recursive_theta}
	\theta_{k}^t = \frac{\theta_k^{t-1}}{N} \sum\nolimits_{v\in \mathcal{V}}\!n_v \varphi_{kv},
	\end{align} where $t$ is the number of iterations in the CGS procedure, and $N=|\mathbf{w}_m|$ is the number of words in the test document $\mathbf{w}_m$.
\end{lemma}	
\begin{proof}
	At iteration $t$, denote the full conditional probability of sampling a topic $k$ for word $v$ as $p_{kv}^t$, then $p_{kv}^t$ is the same at each of the $n_v$ sampling operations:
	\begin{align*}
	p_{kv}^t=\varphi_{kv} \cdot \frac{N_{k}^{t-1}+\alpha}{N^{t-1}+K\alpha}
	\end{align*}
	Here since only the test document $\mathbf{w}_m$ is involved, we have omitted the sub-script $m$ for clarity of notation. 
	Because the surrogate inference procedure goes over the entire document at each iteration $t$, the sum of topic count always equals the total number of words $N$ in the test document. When $\alpha\rightarrow 0$ and $K\alpha \rightarrow 1$, the above equation is re-written as:
	\begin{align*}
	p_{kv}^t=\varphi_{kv} \cdot \frac{N_{k}^{t-1}}{N}
	\end{align*}
	The approximation on the denominator holds as $N\gg 1$.
	When sampling repeats $n_v$ times, the \textit{expected} count of topic $k$ being sampled at word $v$ is $N_{kv}^t=n_v p_{kv}^t$.
	Note that we omit the expectation symbol for clarify of notation. Thus, the total \textit{expected} count of topic $k$ for the test document is:
	\begin{align*}
	N_{k}^t=\sum\nolimits_{v\in \mathcal{V}}n_v p_{kv}^t
	\end{align*}
	Combining the above two equations,
	\begin{align*}
	N_{k}^t=\sum\nolimits_{v\in \mathcal{V}}\!n_v \varphi_{kv}  \frac{N_{k}^{t-1}}{N} = \frac{N_{k}^{t-1}}{N} \sum\nolimits_{v\in \mathcal{V}}\!n_v \varphi_{kv}   
	\end{align*}
	The second equation holds because $\frac{N_{k}^{t-1}}{N}$ does not depend on $v$.
	According to Eq.(2) in the main text,
	\begin{align*}
	\theta_k = \frac{N_k^t+\alpha}{N + K\alpha} \rightarrow \frac{N_k^t}{N}
	\end{align*} when $\alpha \rightarrow 0$, the recursive equation in Eq.\eqref{eq:recursive_theta} holds.
\end{proof}

From Lemma~\ref{lemma:recursive_theta}, we can derive the following theorem:
\begin{theorem}
When $\alpha \rightarrow 0$ and the document-topic distribution is initialized as a discrete uniform distribution, i.e., for each $k=1,\ldots,K$, there is $\theta_k^0 = 1/K$, then
\begin{align}\label{eq:analytic_estimate}
\theta_{k}^t = \frac{1}{K}\cdot \Big(\frac{\sum\nolimits_{v\in \mathcal{V}}\!n_v \varphi_{kv}}{N}\Big)^t,\  \forall t\geq 0
\end{align}
\end{theorem}
The theorem can be easily proved by induction from $t=0$ using Lemma~\ref{lemma:recursive_theta} and is therefore omitted.
In practical implementation, different $t$ values (i.e., different levels of approximation) are tested. 
With the analytical estimate of document-topic distribution, we can efficiently compute any form of attack objective value $Q(w,w')$ in Eqs\eqref{eq:promote}-\eqref{eq:demote} for a target-replacement word pair $(w,w')$.

\subsection{Attack Strategy Design}\label{sec:attack_design}
Even with an efficient estimate of the attack objective given an attack strategy, the optimization problem in Eqs.\eqref{eq:obj}-\eqref{eq:constr_capacity} is essentially combinatorial, which is proven to be \textit{NP-hard}. To scale up the attack, we design a greedy algorithm which assumes the marginal contribution $Q(w,w')$ to the attack objective of different target-replacement word pairs $(w,w')$ is independent from each other:
\begin{align}
Q(\mathcal{W},\mathcal{W}')=\sum_{w\in \mathcal{W}}Q(w,w')
\end{align}
	
With this assumption, the problem is then to find the top ranked set of target-replacement word pairs which have the highest marginal contribution to the attack objective, subject to the two constraints in Eqs.\eqref{eq:constr_dist}-\eqref{eq:constr_capacity}. Algorithm~\ref{alg:evalda} shows the detailed description of EvaLDA. 

\noindent\textbf{Step 1: Get feasible set of target words.} It starts by getting the feasible set of target words $\mathcal{W}^f$ for the victim document $\mathbf{w}^{vic}$ (Line~\ref{line:get_vf}). This is done by removing unimportant words such as stop-words. With this step, the attack strategy space is sufficiently reduced with little sacrifice on the effectiveness of the attack. Simultaneously, we can get a feasible vocabulary $\mathcal{V}^f$ that corresponds to $\mathcal{W}^f$.

\begin{algorithm2e}[t]
\DontPrintSemicolon
\SetAlgoNoEnd
\SetInd{0.5em}{0.5em}
\caption{EvaLDA} \label{alg:evalda}
\KwIn{Victim document $\mathbf{w}^{vic}$, topic word distribution $\varphi$, attack type in Eqs.\eqref{eq:promote}\eqref{eq:demote}, approximation level $t$ in Eq.\eqref{eq:analytic_estimate}, 
word distance threshold $\sigma$, modification budget $\kappa$.}
\KwOut{Adversarial document $\mathbf{w}^{adv}$} \hrule

Get feasible target word set $\mathcal{W}^f$ and vocabulary $\mathcal{V}^f$\label{line:get_vf}
		
Get candidate replacement $\mathcal{R}(v),\ \forall v\in \mathcal{V}^f$ \label{line:get_Rv}
		
\For{$w\in \mathcal{W}^f$}{\label{line:for_start}
	$Q^*(w)\leftarrow 0$
			
	\For{$w'\in \mathcal{R}(w)$}{
				
		Compute $Q(w,w')$ according to Eq.\eqref{eq:analytic_estimate} and attack type
				
		\If{$Q(w,w')>Q^*(w)$} {$Q^*(w)\leftarrow Q(w,w')$
					
			$u(w)\leftarrow w'$
		}
	}
}
$\mathcal{W}^f_{sort} \leftarrow$ Sort $\mathcal{W}^f$ by $Q^*(w)$ \label{line:sort}
		
$\mathcal{W}^*\leftarrow$ first word of $\mathcal{W}^f_{sort}$ \label{line:generate_start}
		
		
\While{$|\mathcal{W}^*|<|\mathbf{w}^{vic}|\cdot \kappa$}{\label{line:while_start}
			
	$\mathcal{W}\leftarrow \mathcal{W}^*$
			
	$w \leftarrow$ next word of $\mathcal{W}^f_{sort}$
			
	$\mathcal{W}^*\leftarrow \mathcal{W}^*\cup \{w\}$ \label{line:while_end}
			
}
		
$\mathbf{w}^{adv}\leftarrow$ replace words $w \in\mathcal{W}$ with $u(w)$ in $\mathbf{w}^{vic}$ \label{line:generate_end}
		
\Return $\mathbf{w}^{adv}$
\end{algorithm2e}

\noindent\textbf{Step 2: Get candidate replacement set (Line~\ref{line:get_Rv}).} 
In addition to finding the top-ranked target-replacement word pairs so as to maximally deteriorate the model's performance, another critical aspect of a successful evasion attack is to make the perturbations evasive. Therefore, it is important to find a candidate set of replacements $\mathcal{R}(v)$ for each unique word $v\in\mathcal{V}^f$ that are ``close'' to $v$.
In implementation, we consider two ways of building the candidate replacement set, which corresponds to two different notions of ``closeness'' of words. The first way is to find the set of synonyms $\mathcal{S}_v$ for each word $v\in \mathcal{V}^f$ (e.g., using WordNet\footnote{\url{http://wordnet.princeton.edu/}}). Another way is to measure word ``closeness'' in the word embedding space, e.g.~\citep{bojanowski2017enriching}. More specifically, we use the cosine distance $1-\text{cos}(v,v')$ as the distance measure of a word pair $(v,v')$ in the embedding space and add words whose cosine distance w.r.t. the target word is smaller than a threshold $\sigma$. This is done similarly as in previous works, e.g.~\cite{devlin2019bert,abdibayev2021using}.
	
\noindent\textbf{Step 3: Sort target words $w\in \mathcal{W}^f$ by their marginal contributions to the objective ( Lines~\ref{line:for_start}-\ref{line:sort}).} 
This step contains two loops. The outer loop iterates over each target word $w$ in $\mathcal{W}^f$, while for each $w$, the inner loop iterates over all the possible replacement word $w'$ in $\mathcal{R}(w)$. Here $\mathcal{R}(w)$ is obtained by mapping $w$ to the unique word $v$ in the vocabulary. The inner loop finds the best replacement word that has the largest marginal contribution $Q(w,w'))$ to the objective. 
For each target word $w\in \mathcal{W}^f$, its best replacement word and the corresponding largest marginal contribution are respectively denoted as $u(w)$ and $Q^*(w)$, where $Q^*(w)=\max_{w'\in \mathcal{R}(w)} Q(w,w')$. After the two loops, $\mathcal{W}^f$ is sorted w.r.t. $Q^*(w)$ to obtain the top-ranked target-replacement word pairs. The sorted \textit{list} of words is represented as $\mathcal{W}_{sort}^f$.

\noindent\textbf{Step 4: Generate adversarial sample (Lines~\ref{line:generate_start}-\ref{line:generate_end}).} 
This step generates the attack strategy and the adversarial document. It initializes $\mathcal{W}^*$ with the first word in the sorted list $\mathcal{W}_{sort}^f$. In the while loop (Lines~\ref{line:while_start}-\ref{line:while_end}), it checks whether the current set of target words $\mathcal{W}^*$ exceed the budget constraint in Eq.\eqref{eq:constr_capacity}. It then repeatedly adds word from $\mathcal{W}_{sort}^f$ to $\mathcal{W}^*$ until the condition is violated. Note that because $\mathcal{W}$ is a backup of $\mathcal{W}^*$ latent by one step, it is the optimal target word set after the while loop. The algorithm finally replaces each word $w$ in $\mathcal{W}$ with $u(w)$ to obtain the adversarial document.

\begin{table*}[ht]
	\centering
	\begin{tabular}{cccccc}
	\hline
	Dataset& \#Train docs & \#Test docs& \#Train words& Vocab size &Avg test doc length\\
	\hline
	NIPS& 6,562 & 679 & 9,731,300 & 27,176 & 3,211\\
	AP& 1,571 & 675 & 304,357 & 10,473 & 192\\
	\hline
	\end{tabular}
\caption{Statistics of datasets.}\label{tab:dataset_stats}
\end{table*}

\section{Experimental Evaluations}
We conduct empirical experiments to evaluate EvaLDA. 

\noindent\textbf{Datasets}
We evaluate EvaLDA on 2 different datasets, NIPS\footnote{\url{https://www.kaggle.com/benhamner/nips-papers}} and AP\footnote{\url{https://github.com/Blei-Lab/lda-c/blob/master/example/ap.tgz}}. Statistics of the datasets are shown in Table~\ref{tab:dataset_stats}.

\noindent\textbf{Models \& hyperparameters}
We implement LDA-CGS using the lda package\footnote{\url{https://github.com/lda-project/lda}}. The hyperparameters of the two datasets are set as follows: the topic number is $120$ for NIPS dataset, and $75$ for AP dataset. The training iteration is $5,000$ which is enough to converge. We set the hyperparameters $\alpha$ and $\eta$ of the Dirichlet distribution as default values $0.1$ and $0.01$. Each test sample runs $500$ iterations. All experiments are run in machines with Intel E5-2678 v3 and 100GB RAM. Parameters about the EvaLDA algorithm are described after the Evaluation Metrics paragraph.

\noindent\textbf{Baselines}
Because this is the first work that studies evasion attack to LDA models, there are no prior works that are directly applicable to our task. In order to evaluate the performance of EvaLDA, we compare it with 4 reasonable heuristic baselines. \textbf{B1} selects target-replacement word pairs completely randomly. \textbf{B2} randomly selects the target words and selects replacements as the nearest neighbor in the word vector space. \textbf{B3} selects the top-k words related to the target topic while selecting replacements randomly. \textbf{B4} selects the top-k words related to the target topic and selects replacements as the nearest neighbor in the word vector space.

\begin{figure*}[t]
	\centering
	\subfigure{
		\begin{minipage}[t]{ 0.45\columnwidth}
			\centering
		\includegraphics[width=\textwidth]{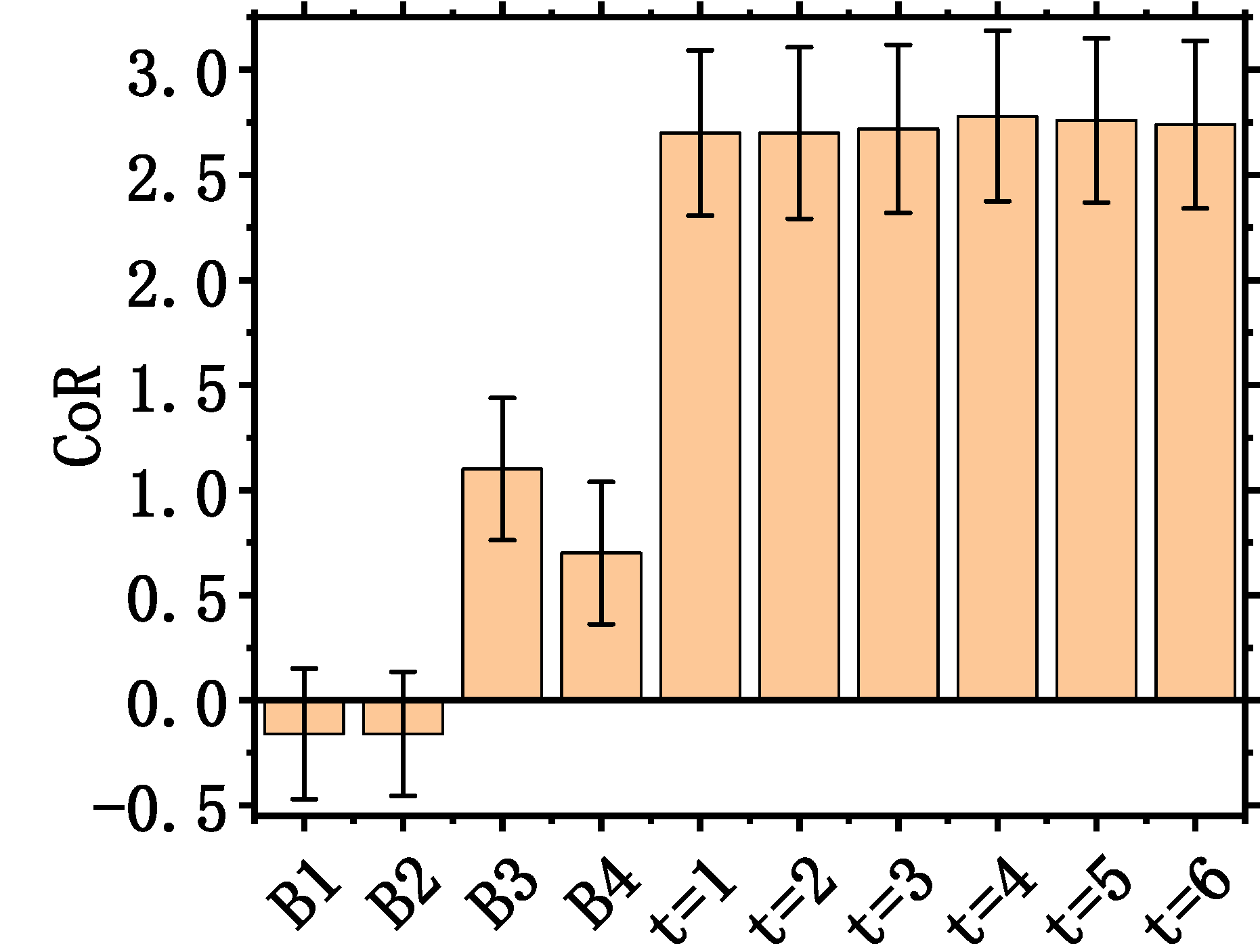}
		\end{minipage}%
	}%
	\subfigure{
		\begin{minipage}[t]{ 0.45\columnwidth}
			\centering			\includegraphics[width=\textwidth]{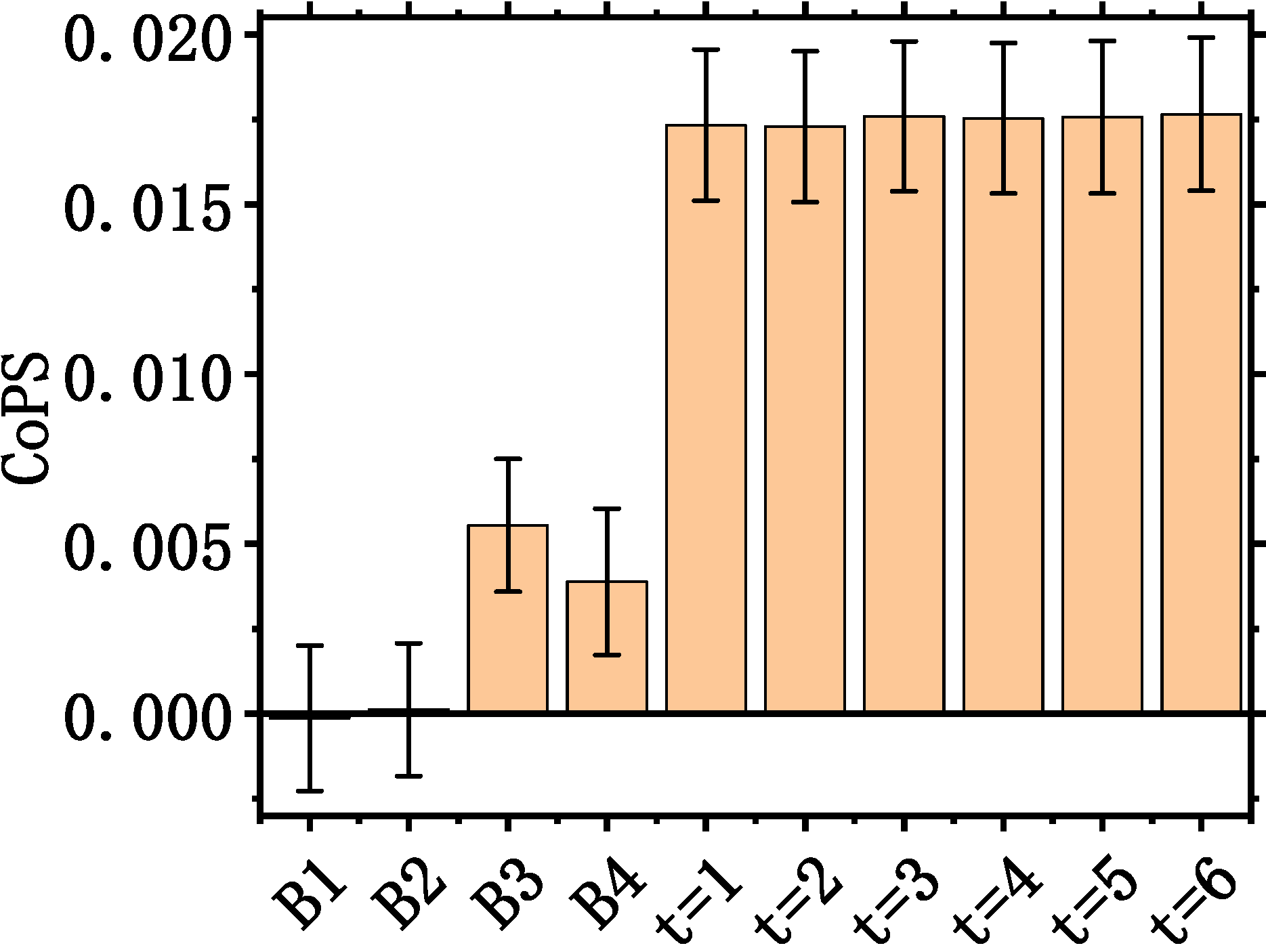}
		\end{minipage}%
	}%
	\subfigure{
	\begin{minipage}[t]{ 0.45\columnwidth}
			\centering
			\includegraphics[width=\textwidth]{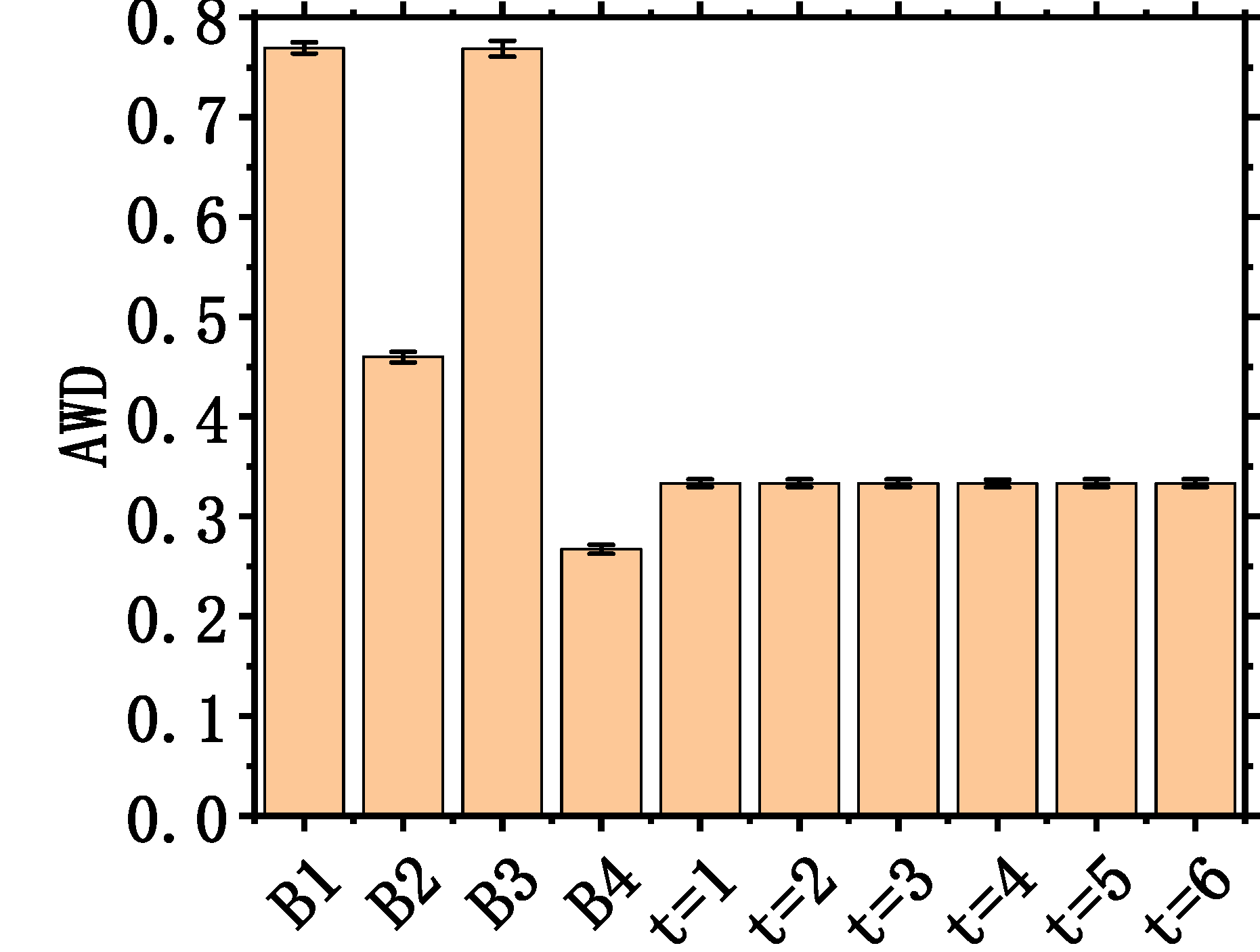}
		\end{minipage}
	}%
	\subfigure{
		\begin{minipage}[t]{0.45\columnwidth}
			\centering
			\includegraphics[width=\textwidth]{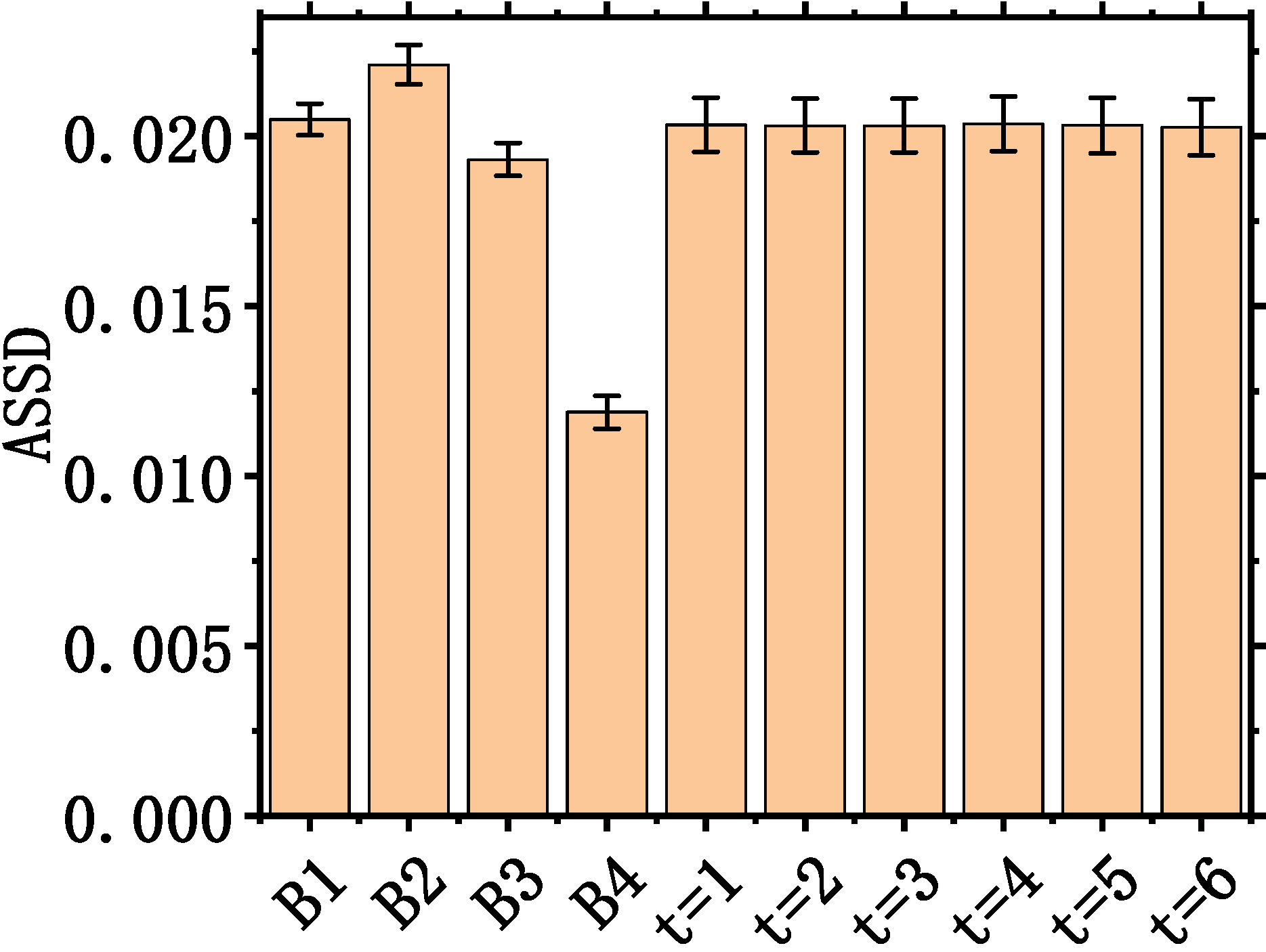}
		\end{minipage}
	}%
	
	\centering
	
	\caption{Promotion attack with varying approximate levels $t$, on the NIPS dataset, showing 95\% confidence interval.}
	\label{fig:nips_diff_level}
	
\end{figure*}

\begin{figure*}
	\centering
	\subfigure{
		\begin{minipage}[t]{0.45\columnwidth}
			\centering
			\includegraphics[width=\textwidth]{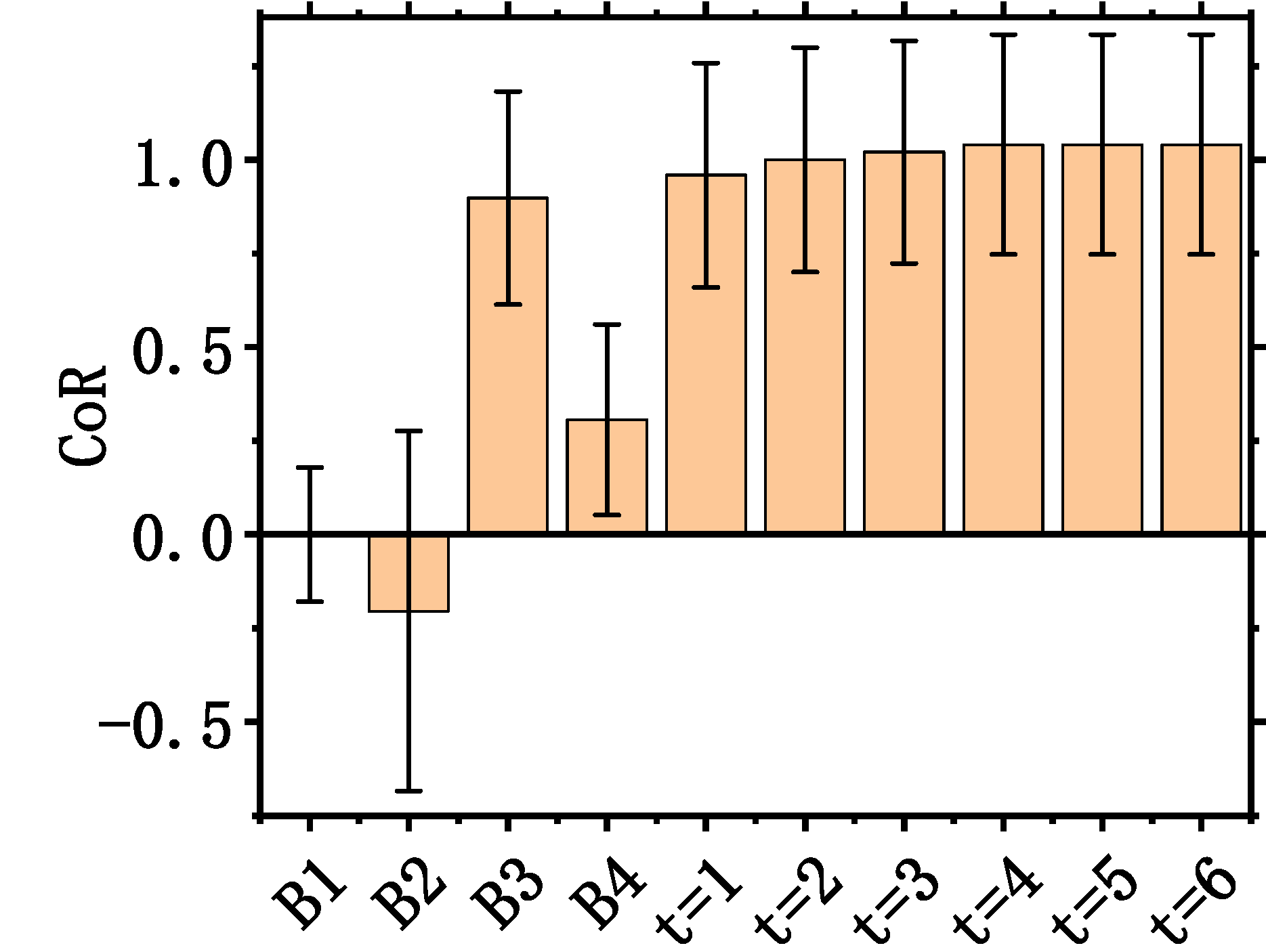}
		\end{minipage}%
	}%
	\subfigure{
		\begin{minipage}[t]{0.45\columnwidth}
			\centering
			\includegraphics[width=\textwidth]{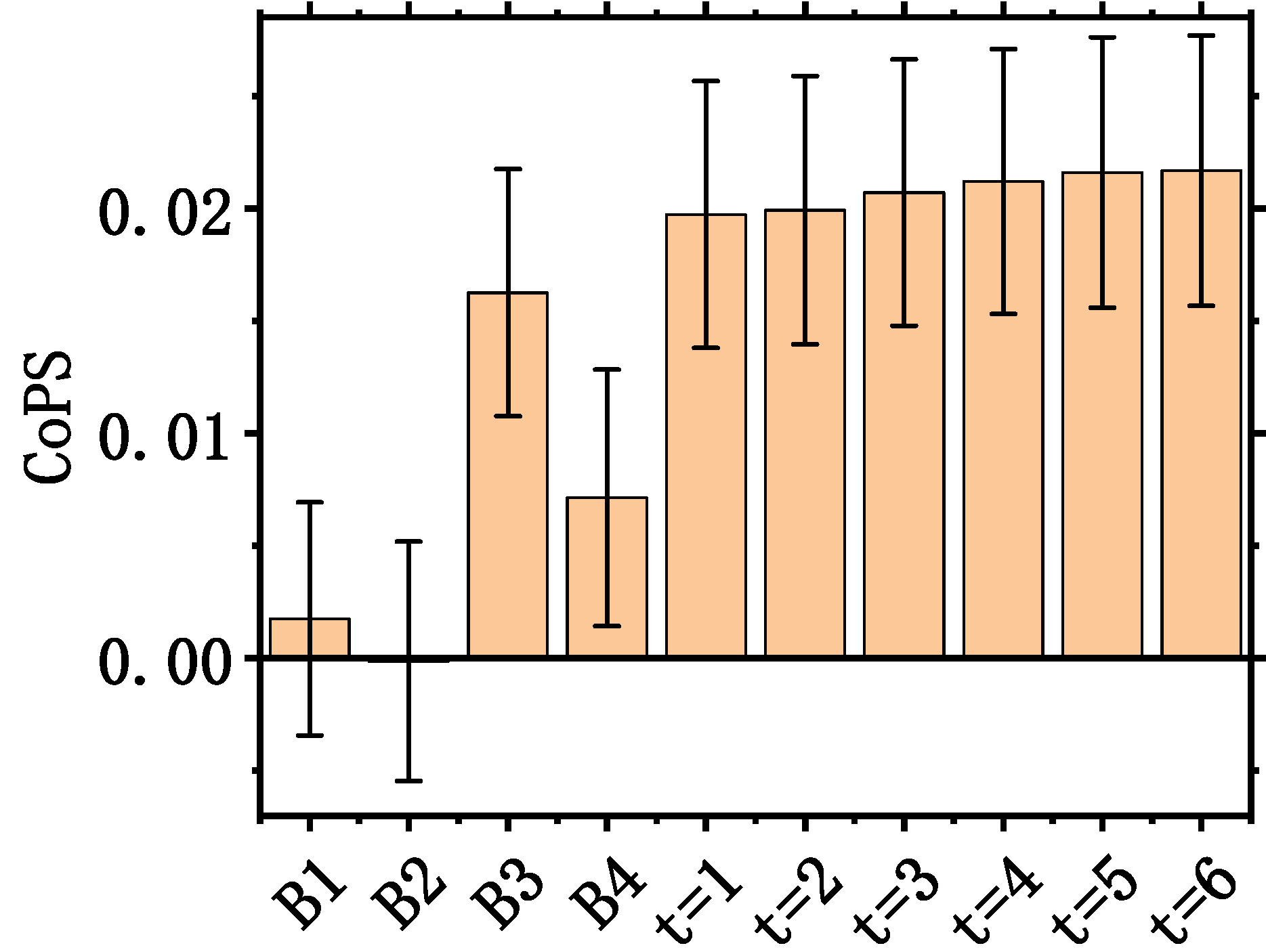}
		\end{minipage}%
	}%
	\subfigure{
		\begin{minipage}[t]{0.45\columnwidth}
			\centering
			\includegraphics[width=\textwidth]{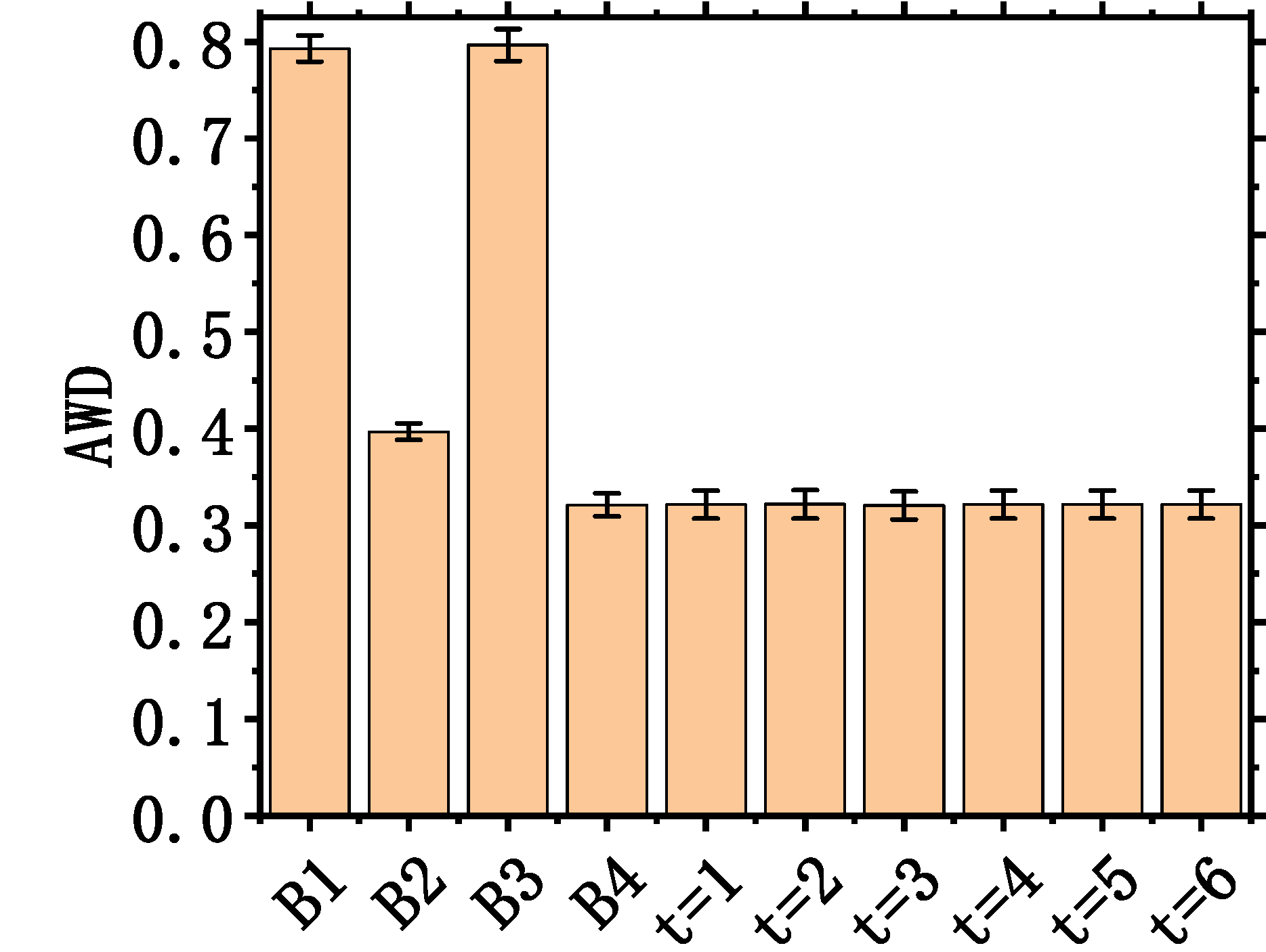}
		\end{minipage}
	}%
	\subfigure{
		\begin{minipage}[t]{0.45\columnwidth}
			\centering
			\includegraphics[width=\textwidth]{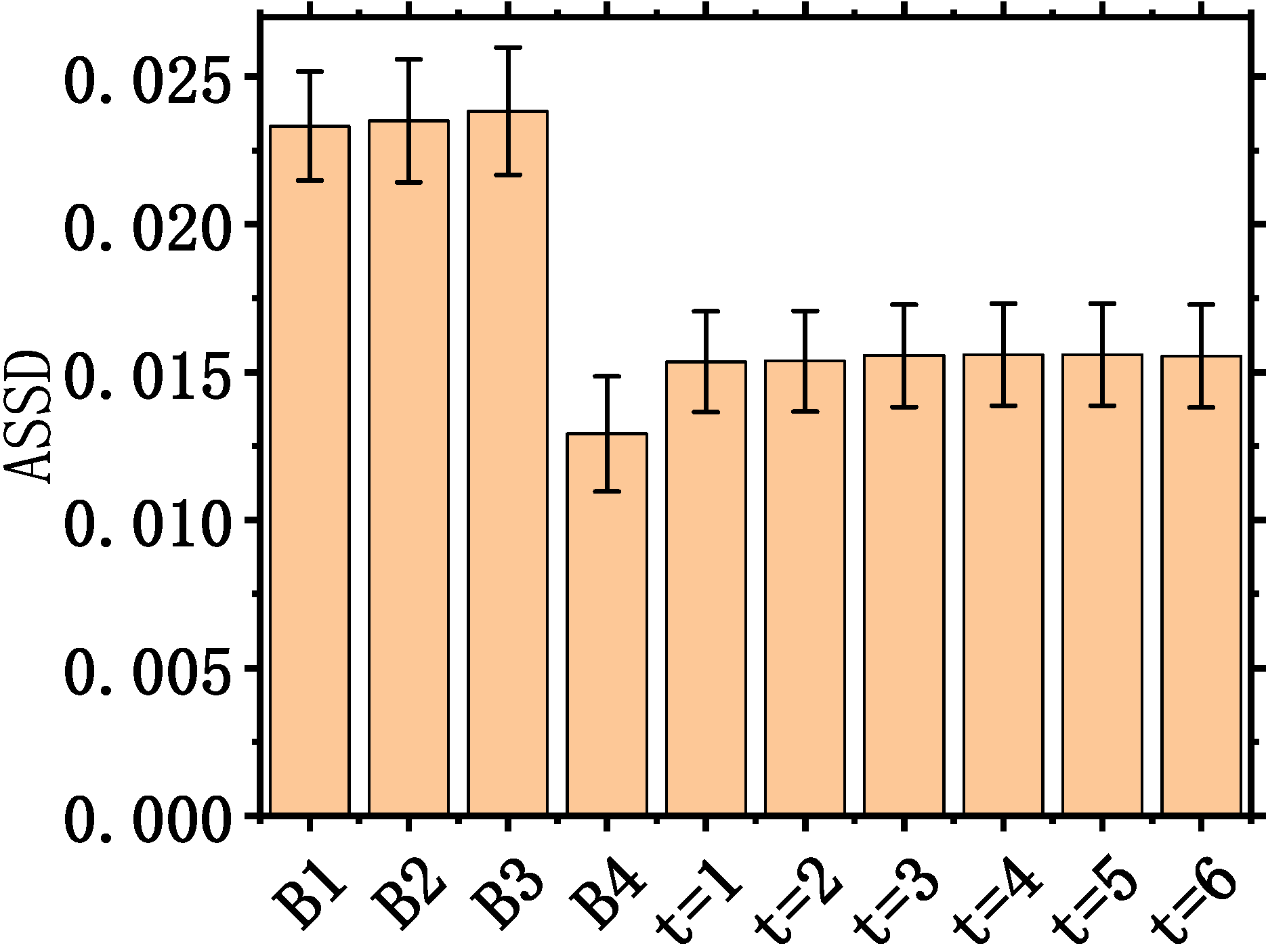}
		\end{minipage}
	}%
	
	\centering
	
	\caption{Promotion attack with varying approximate levels, on the AP dataset, showing 95\% confidence interval.}
	\label{ldac_diff_level}	
\end{figure*}


\begin{figure*}
         \centering
         \subfigure{
               \begin{minipage}[t]{ 0.48\columnwidth}
               \centering
               \includegraphics[width=\textwidth]{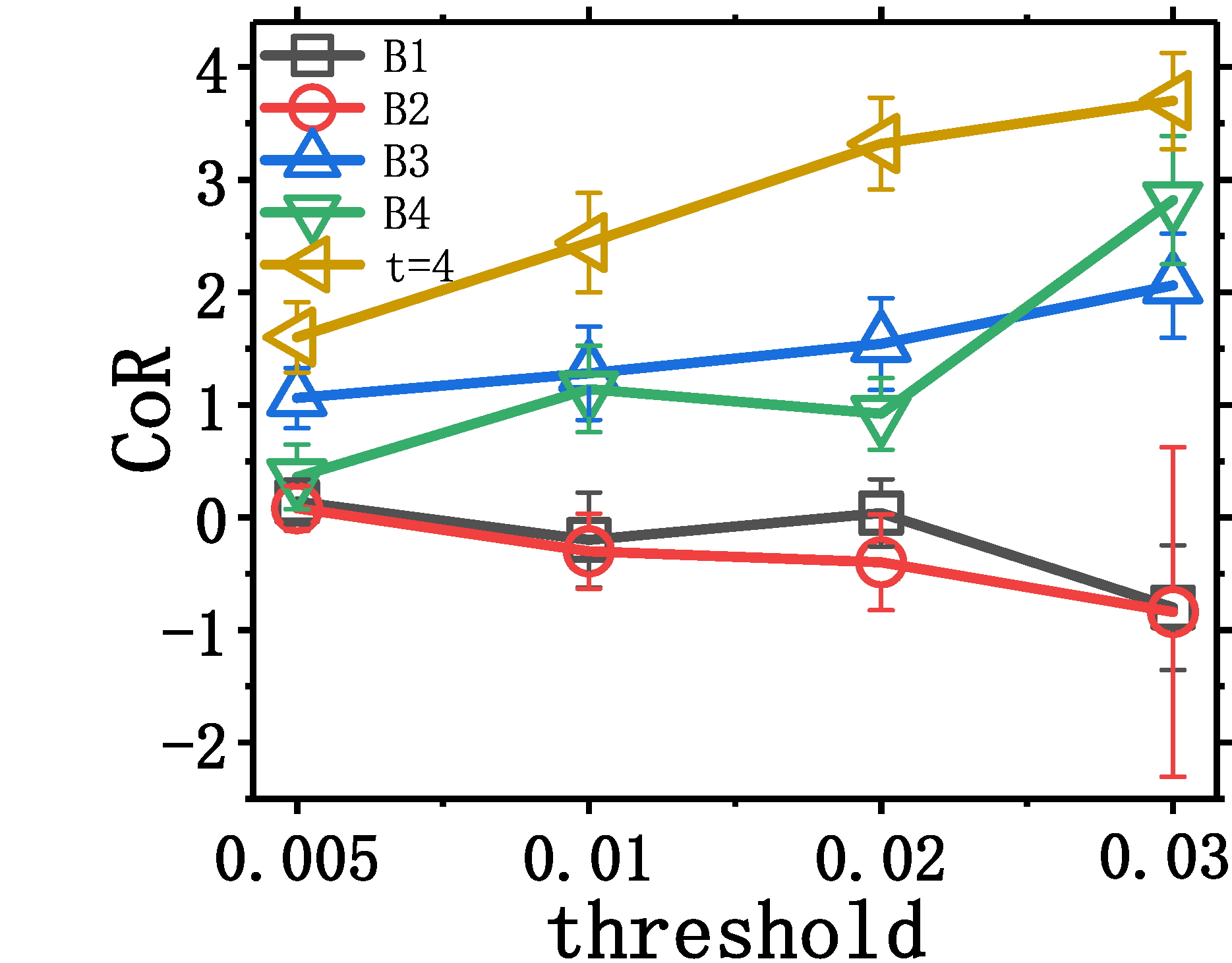}
               \end{minipage}%
           }%
           \subfigure{
                \begin{minipage}[t]{ 0.48\columnwidth}
                     \centering
                     \includegraphics[width=\textwidth]{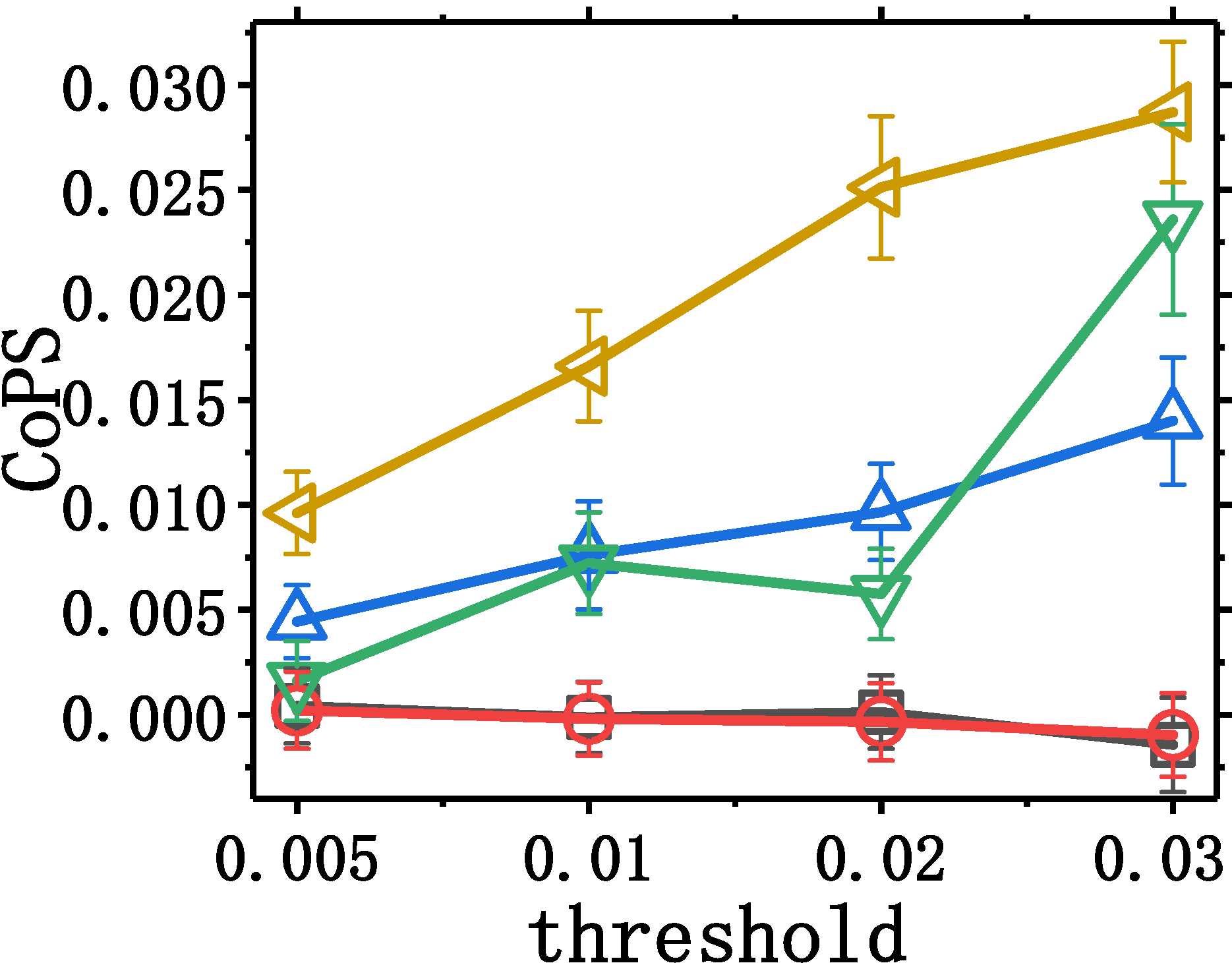}
                \end{minipage}%
           }%
           \subfigure{
                \begin{minipage}[t]{ 0.48\columnwidth}
                     \centering
                     \includegraphics[width=\textwidth]{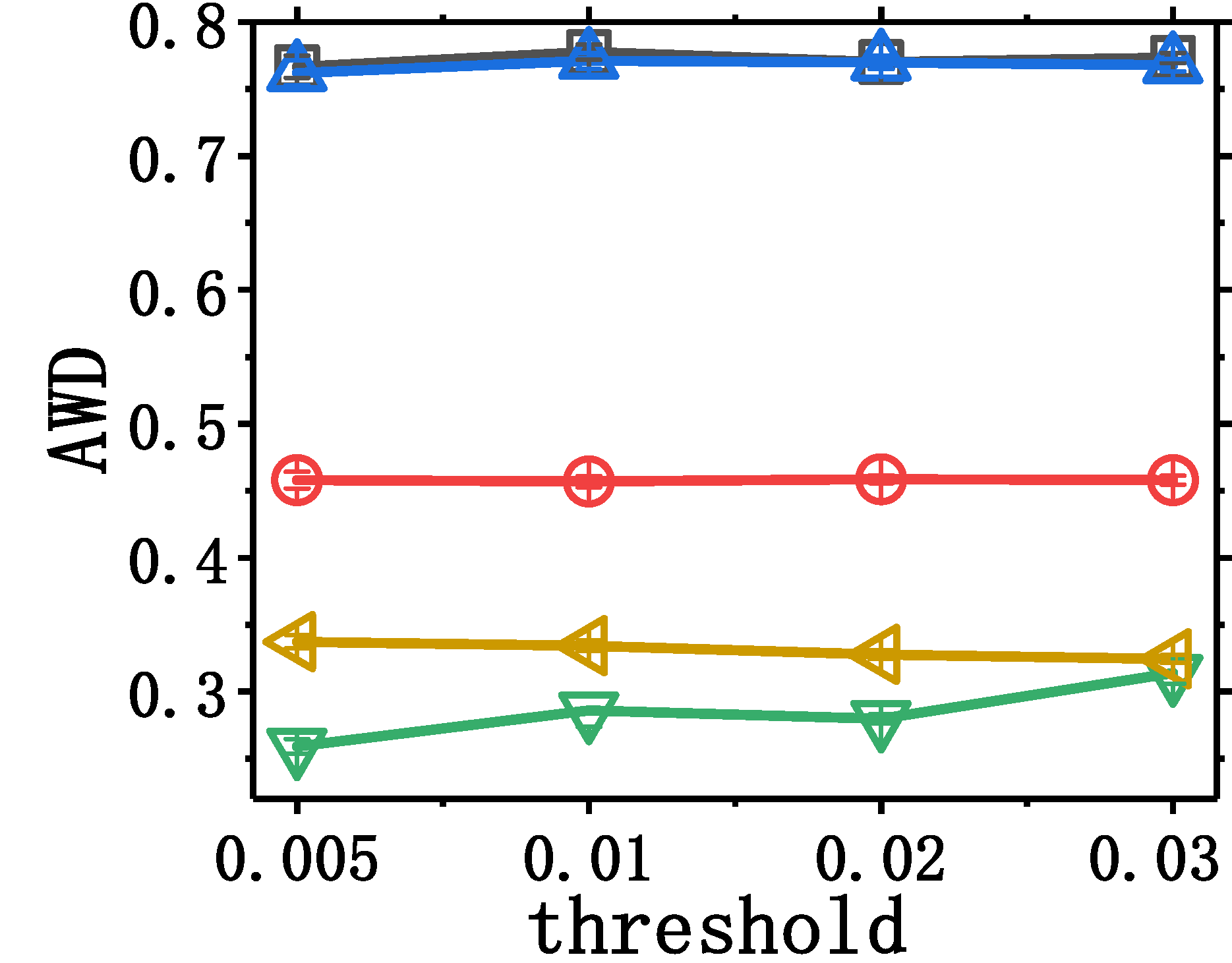}
                \end{minipage}
           }%
		\subfigure{
			\begin{minipage}[t]{ 0.48\columnwidth}
				\centering
				\includegraphics[width=\textwidth]{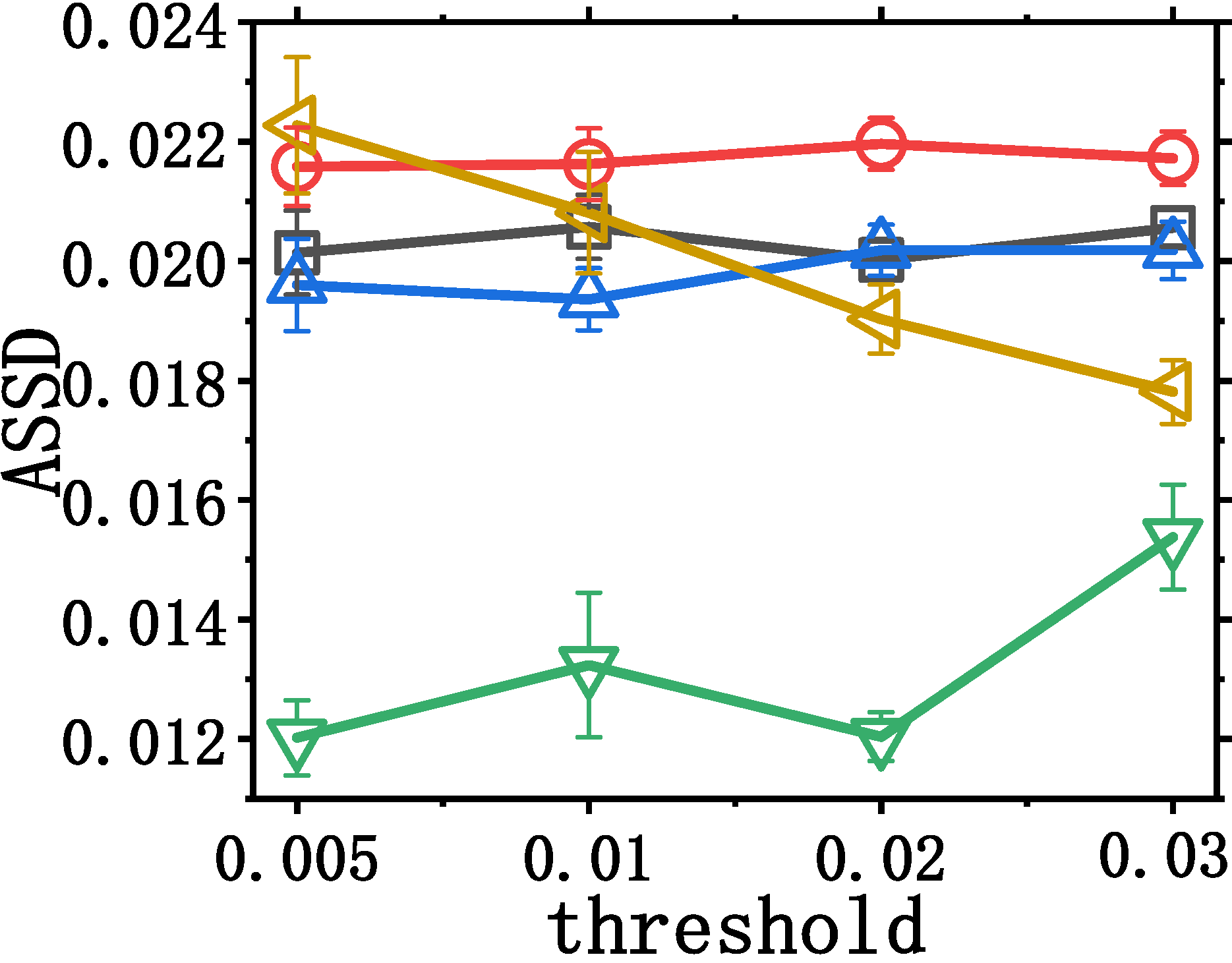}
			\end{minipage}
		}%
		\centering
		
		\caption{Promotion attack with varying perturbation threshold, on the NIPS dataset, showing 95\% confidence interval.}
		\label{nips_diff_thre}
		
\end{figure*}

\begin{figure*}
	\centering
	\subfigure{
		\begin{minipage}[t]{0.48\columnwidth}
			\centering
			\includegraphics[width=\textwidth]{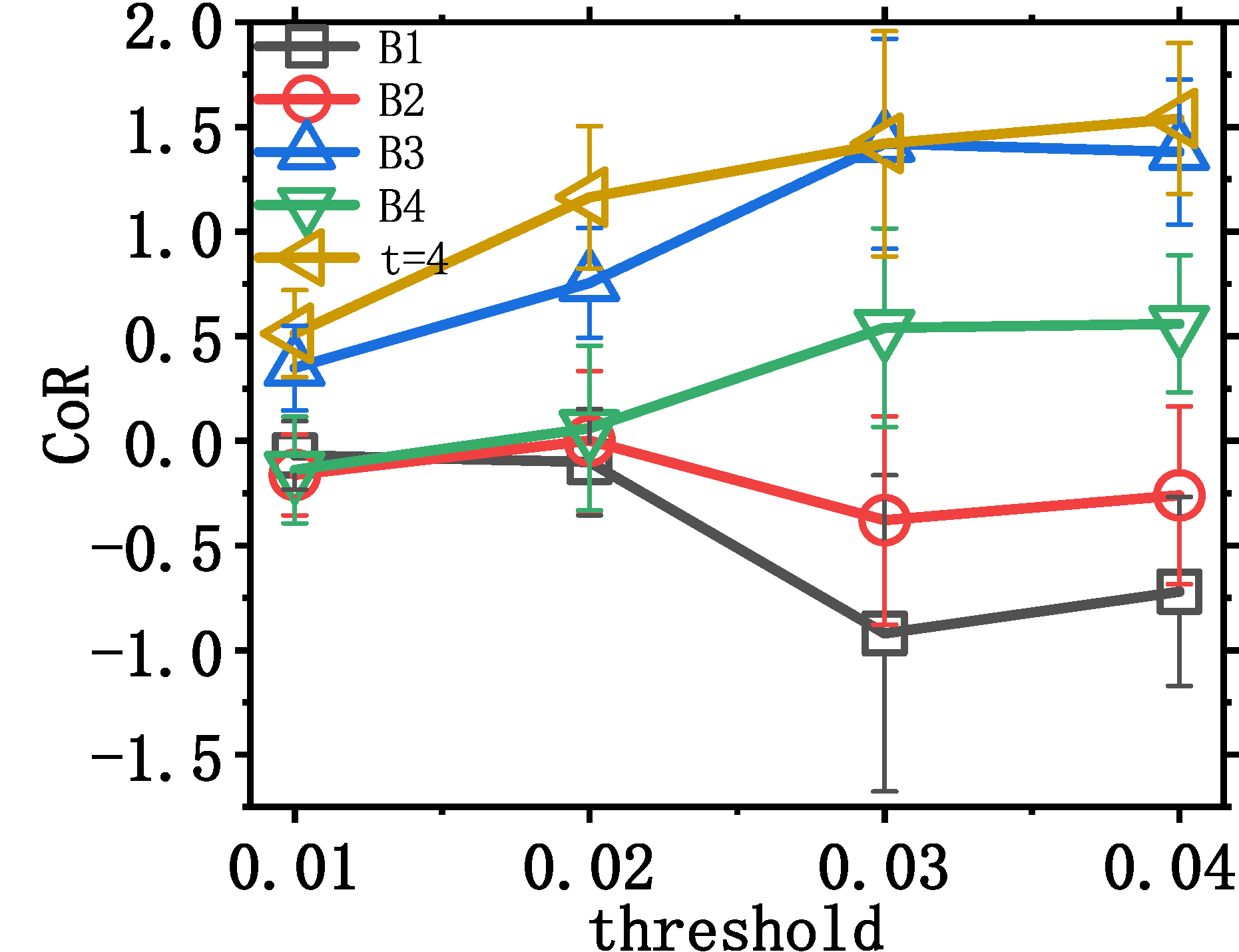}
		\end{minipage}%
	}%
	\subfigure{
		\begin{minipage}[t]{0.48\columnwidth}
			\centering
			\includegraphics[width=\textwidth]{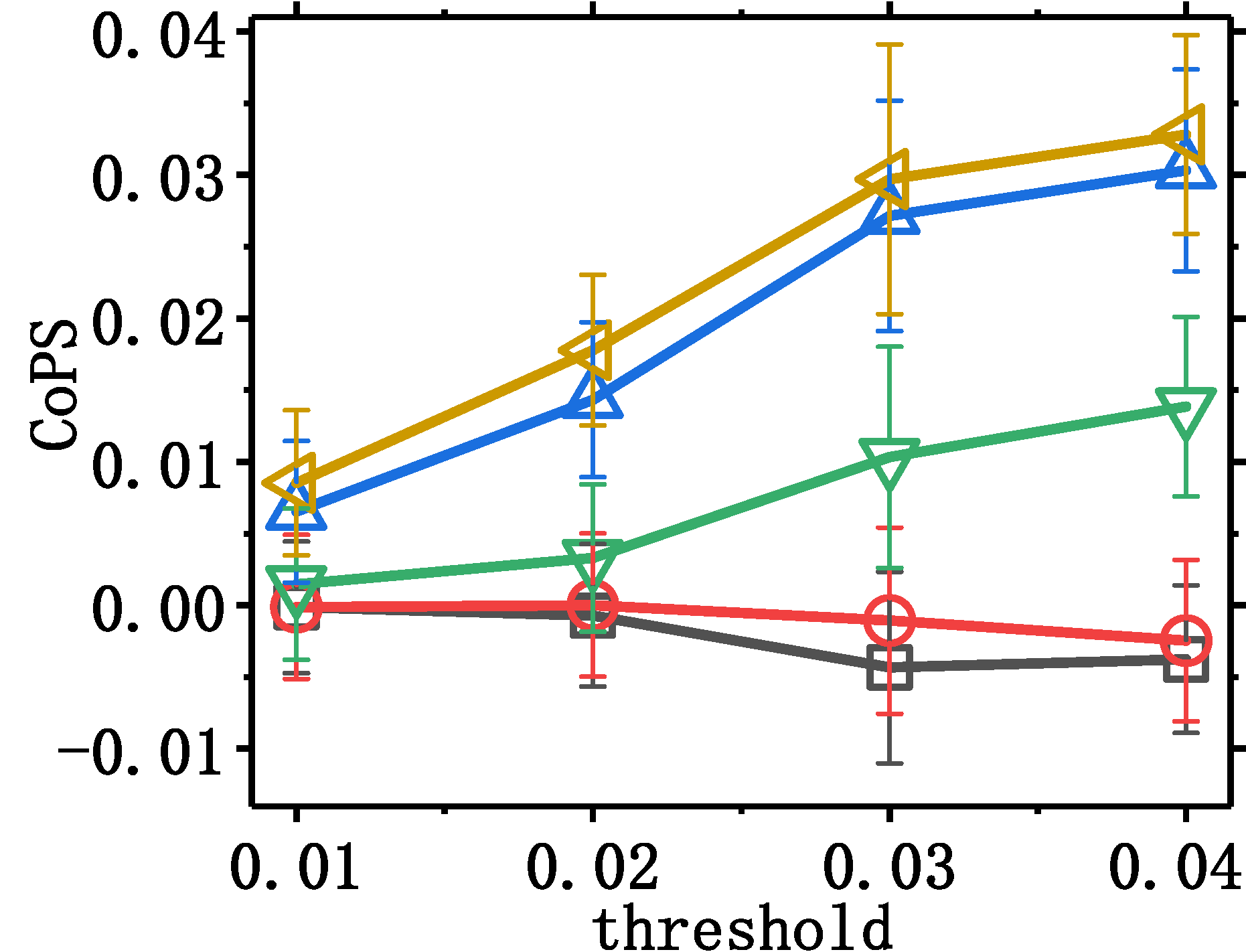}
		\end{minipage}%
	}%
	\subfigure{
		\begin{minipage}[t]{0.48\columnwidth}
			\centering
			\includegraphics[width=\textwidth]{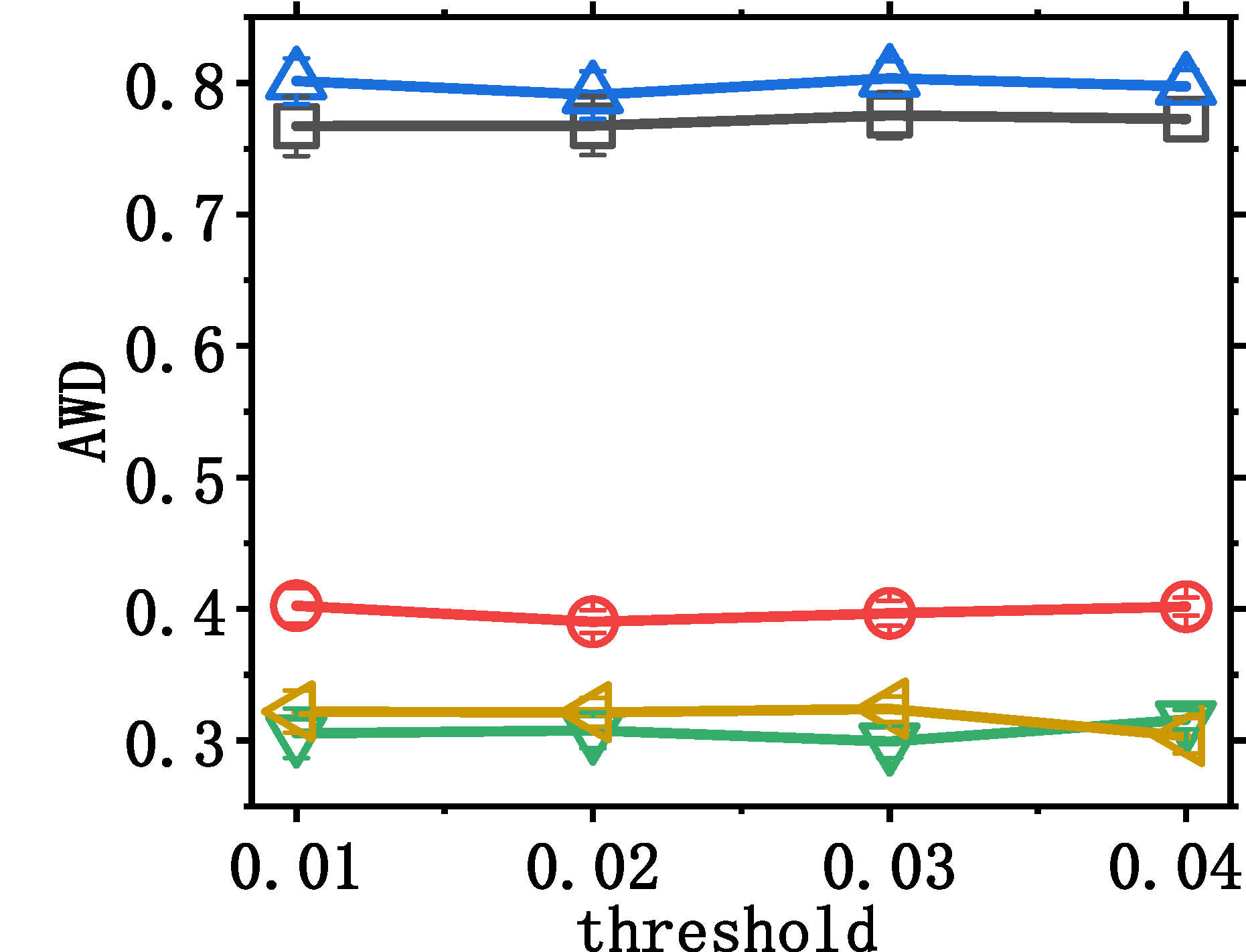}
		\end{minipage}
	}%
	\subfigure{
		\begin{minipage}[t]{0.48\columnwidth}
			\centering
			\includegraphics[width=\textwidth]{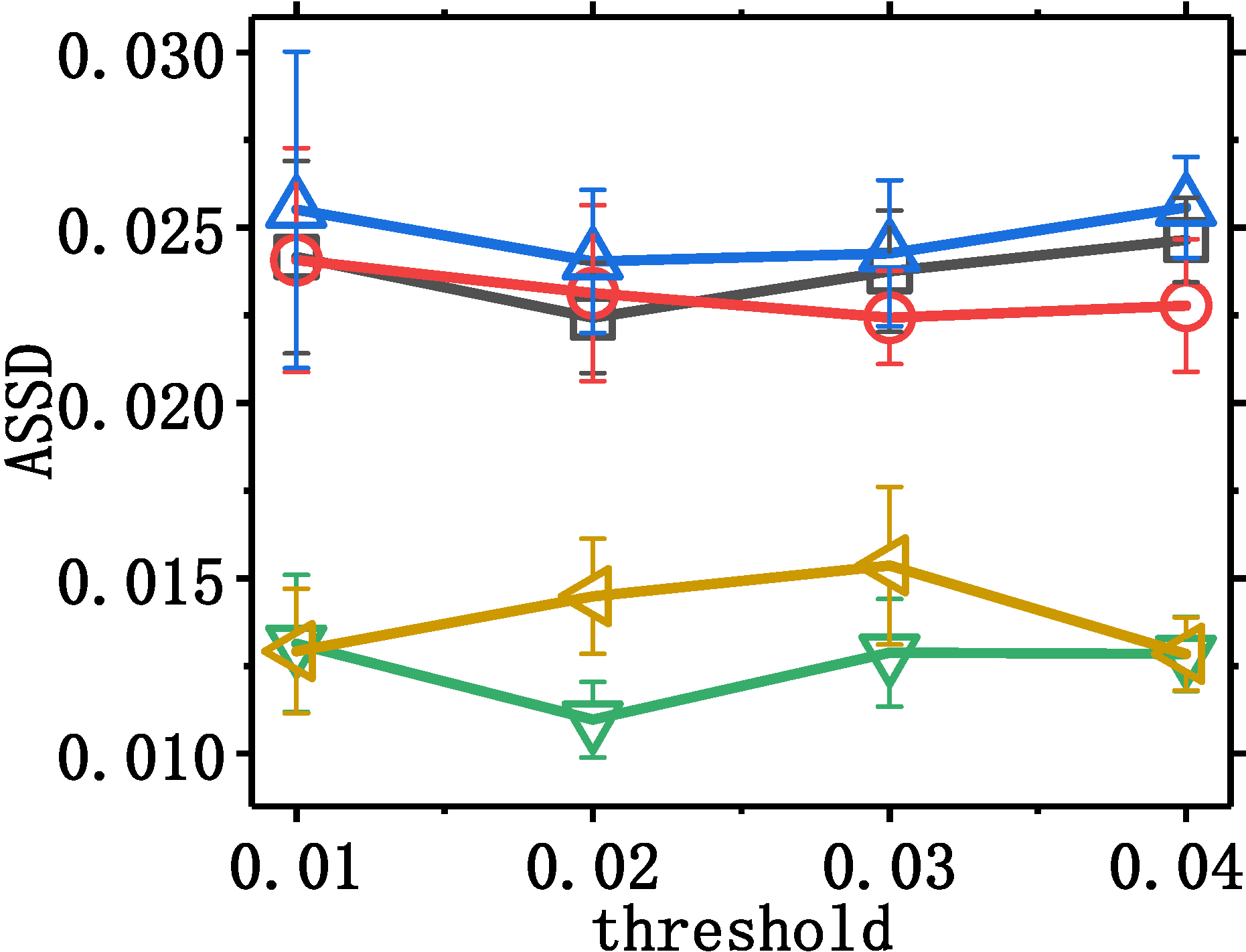}
		\end{minipage}
	}%
	
	\centering
	
	\caption{Promotion attack with varying perturbation threshold, on the AP dataset, showing 95\% confidence interval.}
	\label{ldac_diff_thre}
\end{figure*}

\begin{figure*}
		\centering
		\subfigure{
			\begin{minipage}[t]{0.48\columnwidth}
				\centering
				\includegraphics[width=\textwidth]{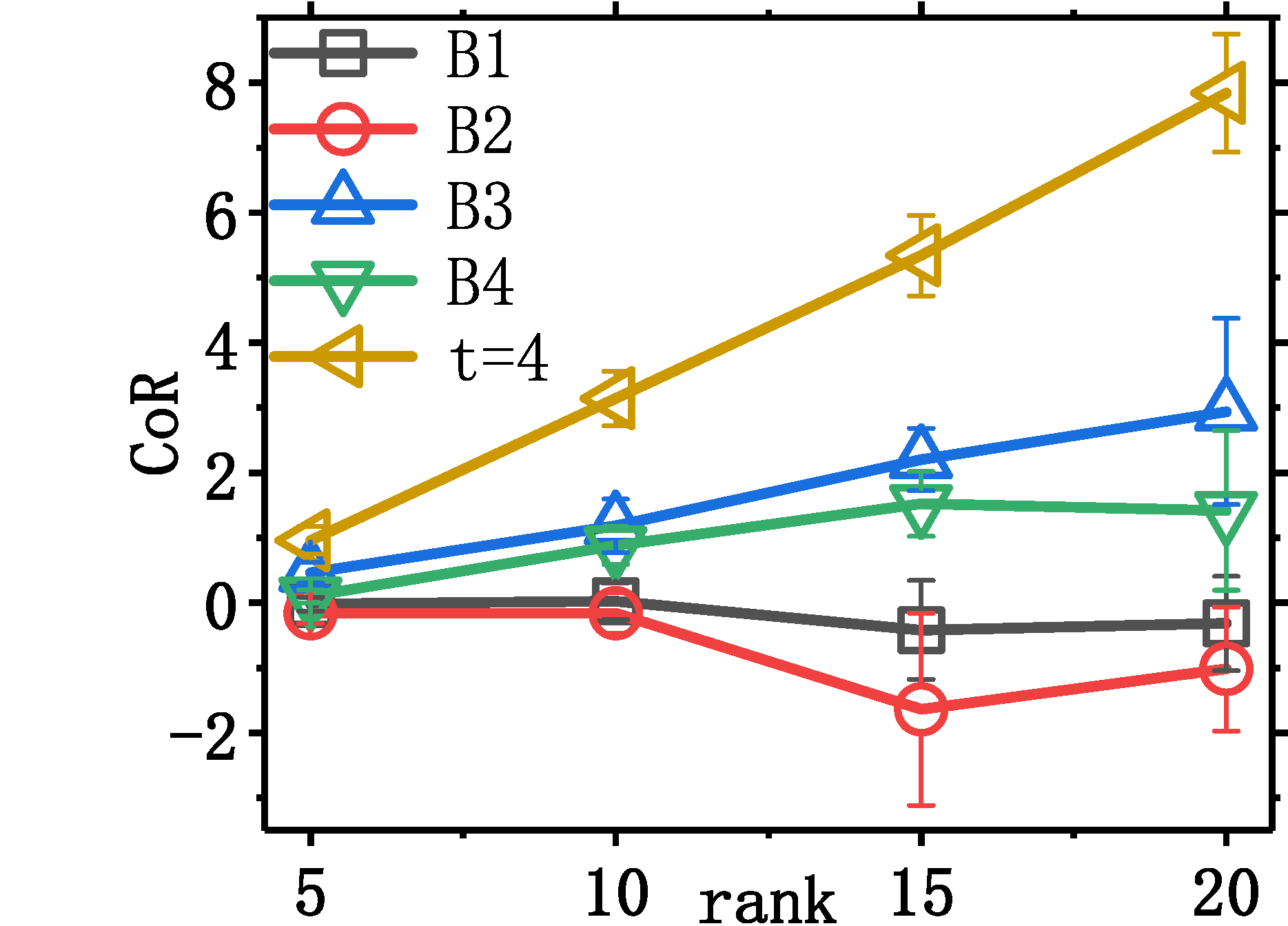}
			\end{minipage}%
		}%
		\subfigure{
			\begin{minipage}[t]{0.48\columnwidth}
				\centering
				\includegraphics[width=\textwidth]{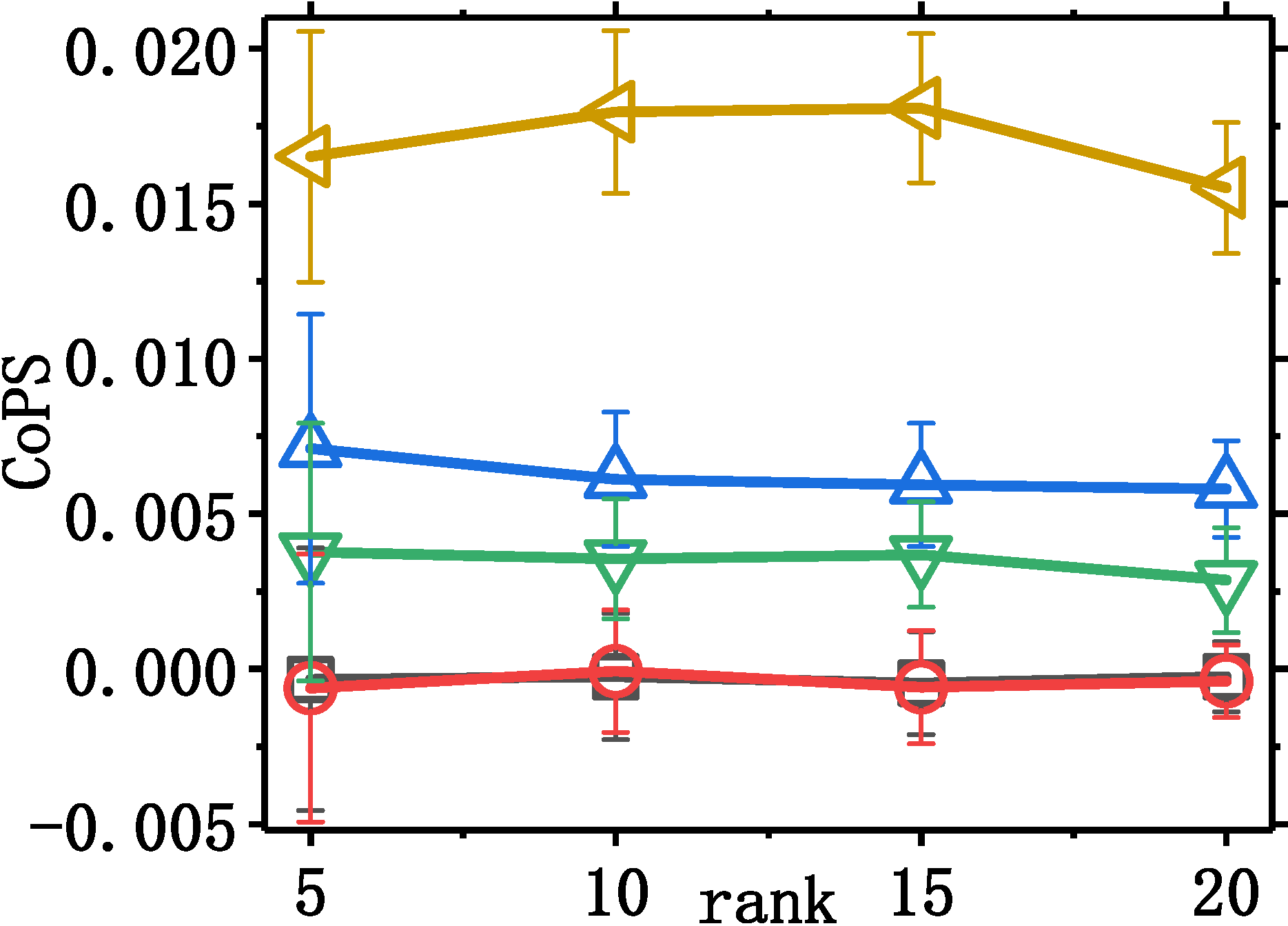}
			\end{minipage}%
		}%
		\subfigure{
			\begin{minipage}[t]{0.48\columnwidth}
				\centering
				\includegraphics[width=\textwidth]{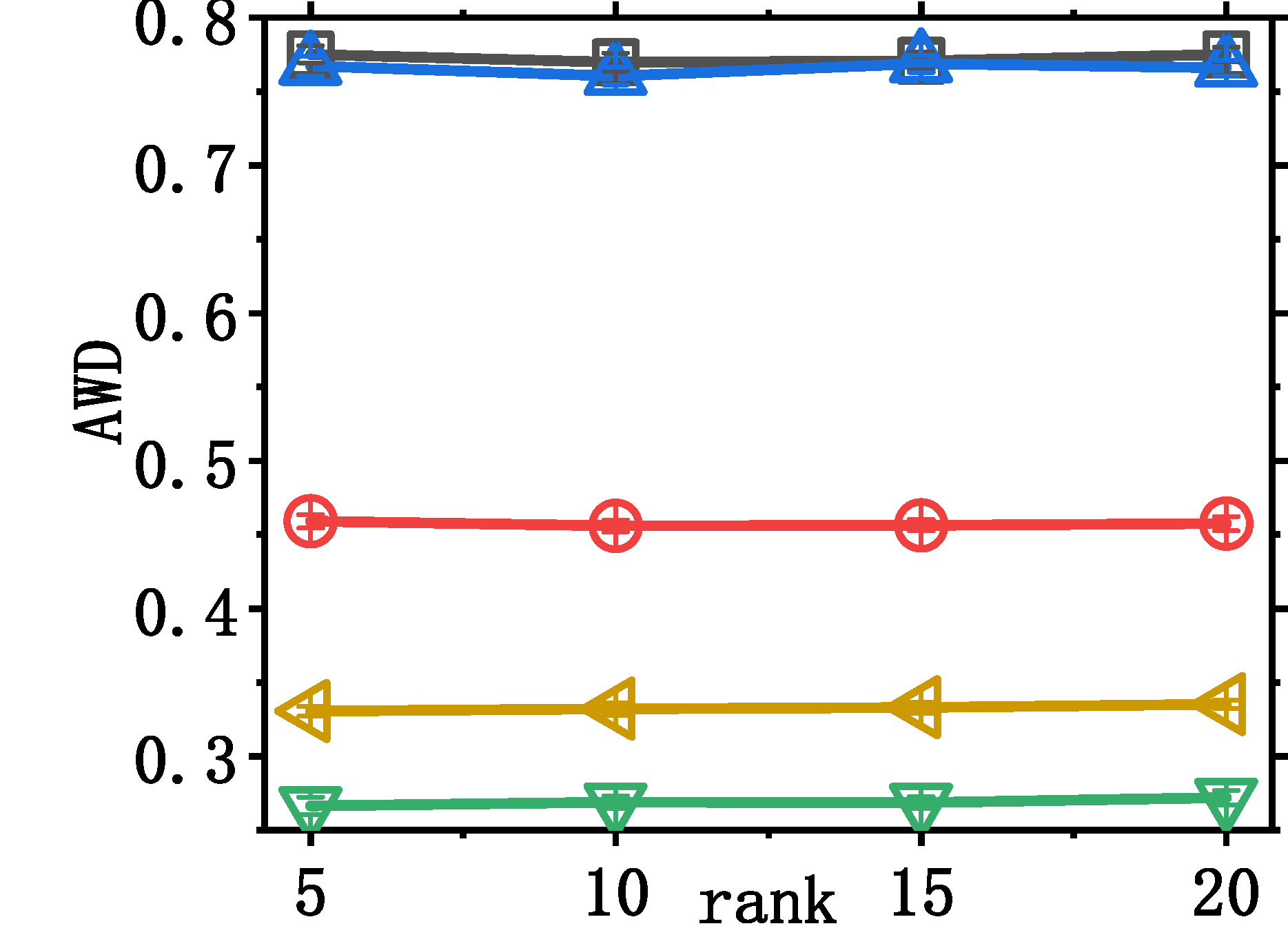}
			\end{minipage}
		}%
		\subfigure{
			\begin{minipage}[t]{0.48\columnwidth}
				\centering
				\includegraphics[width=\textwidth]{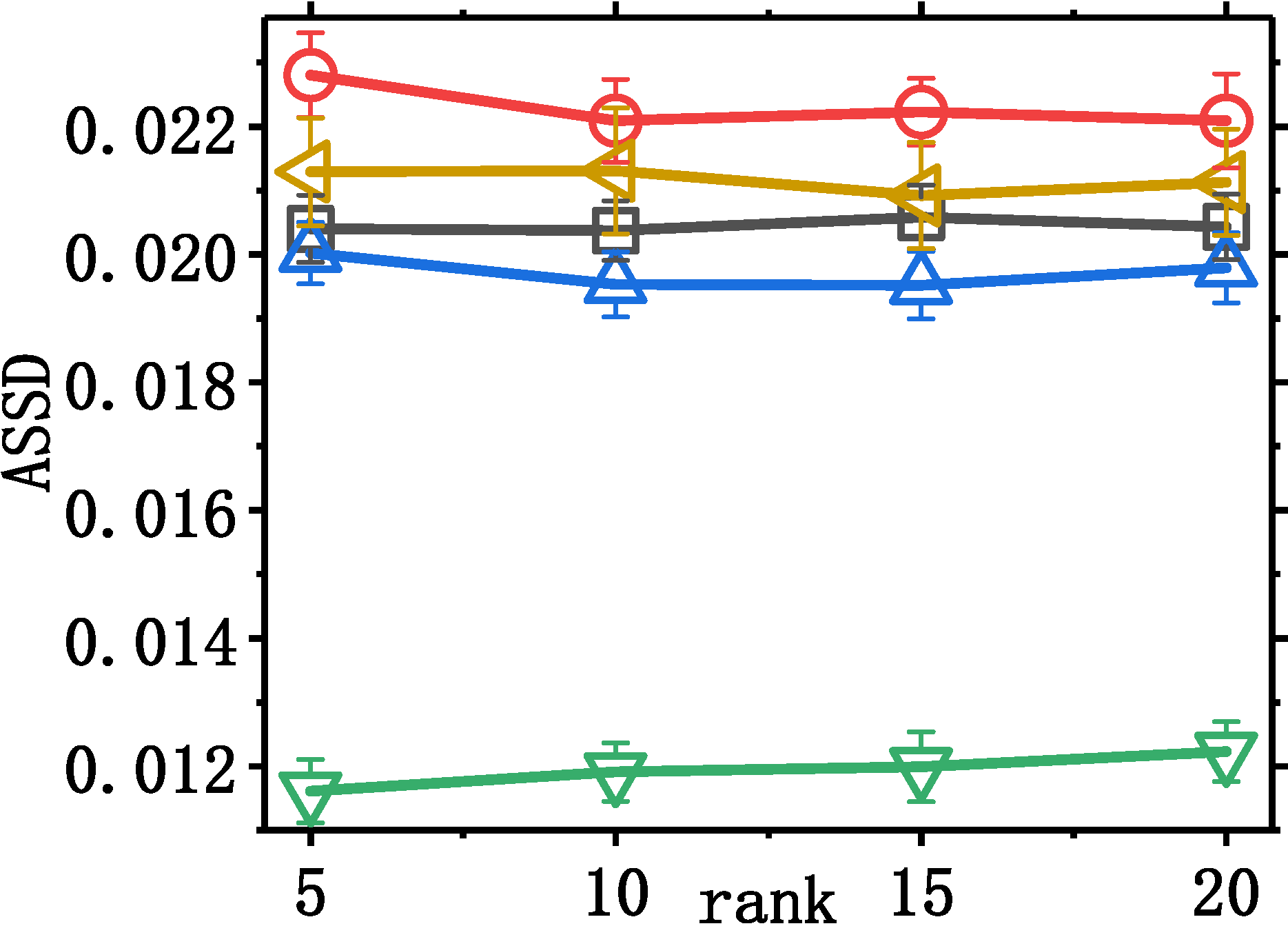}
			\end{minipage}
		}%
		
		\centering
		
		\caption{Promotion attack with varying original rank of the target topic, on the NIPS dataset, showing 95\% confidence interval.}
		\label{nips_diff_rank}
		
	\end{figure*}

\begin{figure*}
	\centering
	\subfigure{
		\begin{minipage}[t]{0.48\columnwidth}
			\centering
			\includegraphics[width=\textwidth]{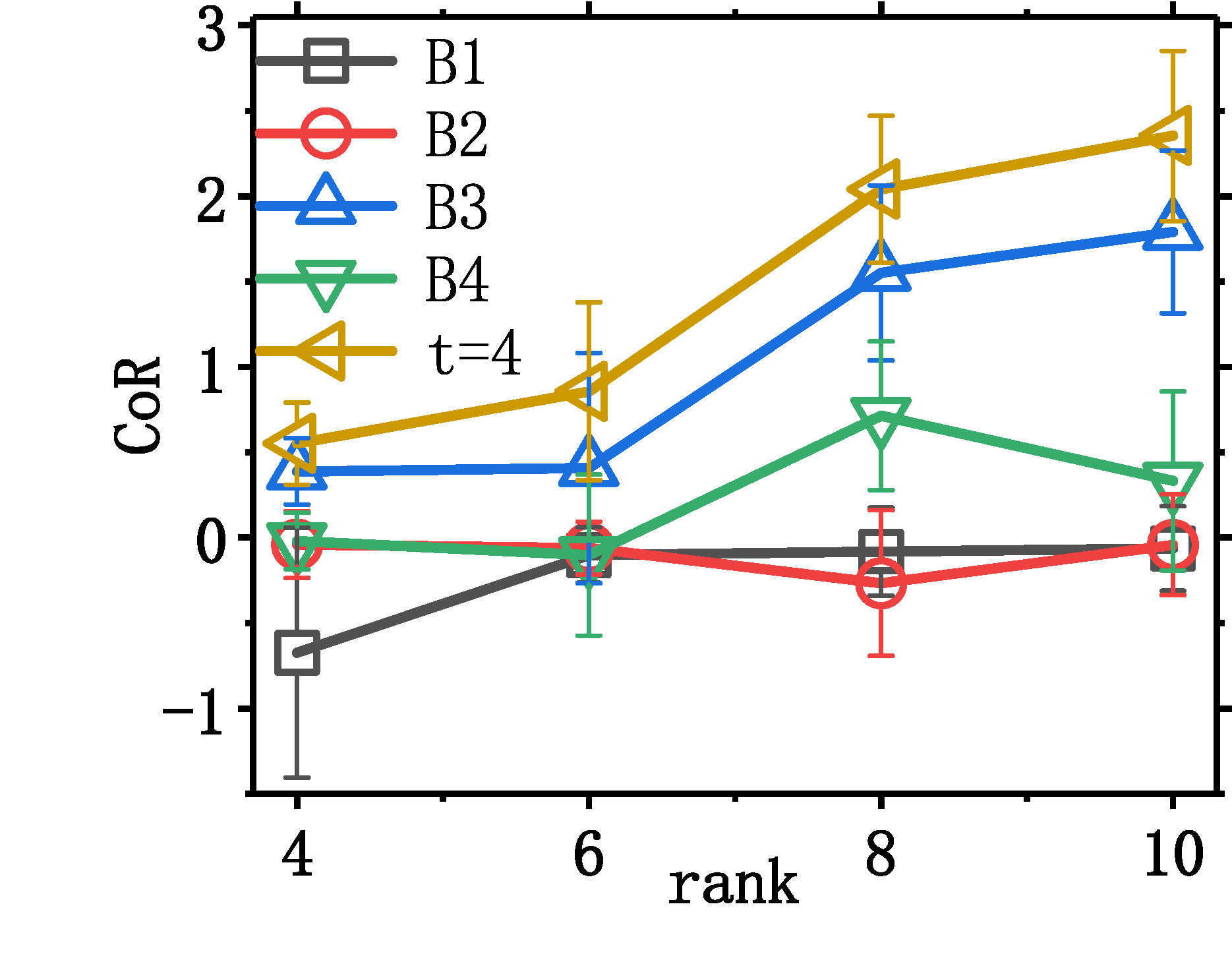}
		\end{minipage}%
	}%
	\subfigure{
		\begin{minipage}[t]{0.48\columnwidth}
			\centering
			\includegraphics[width=\textwidth]{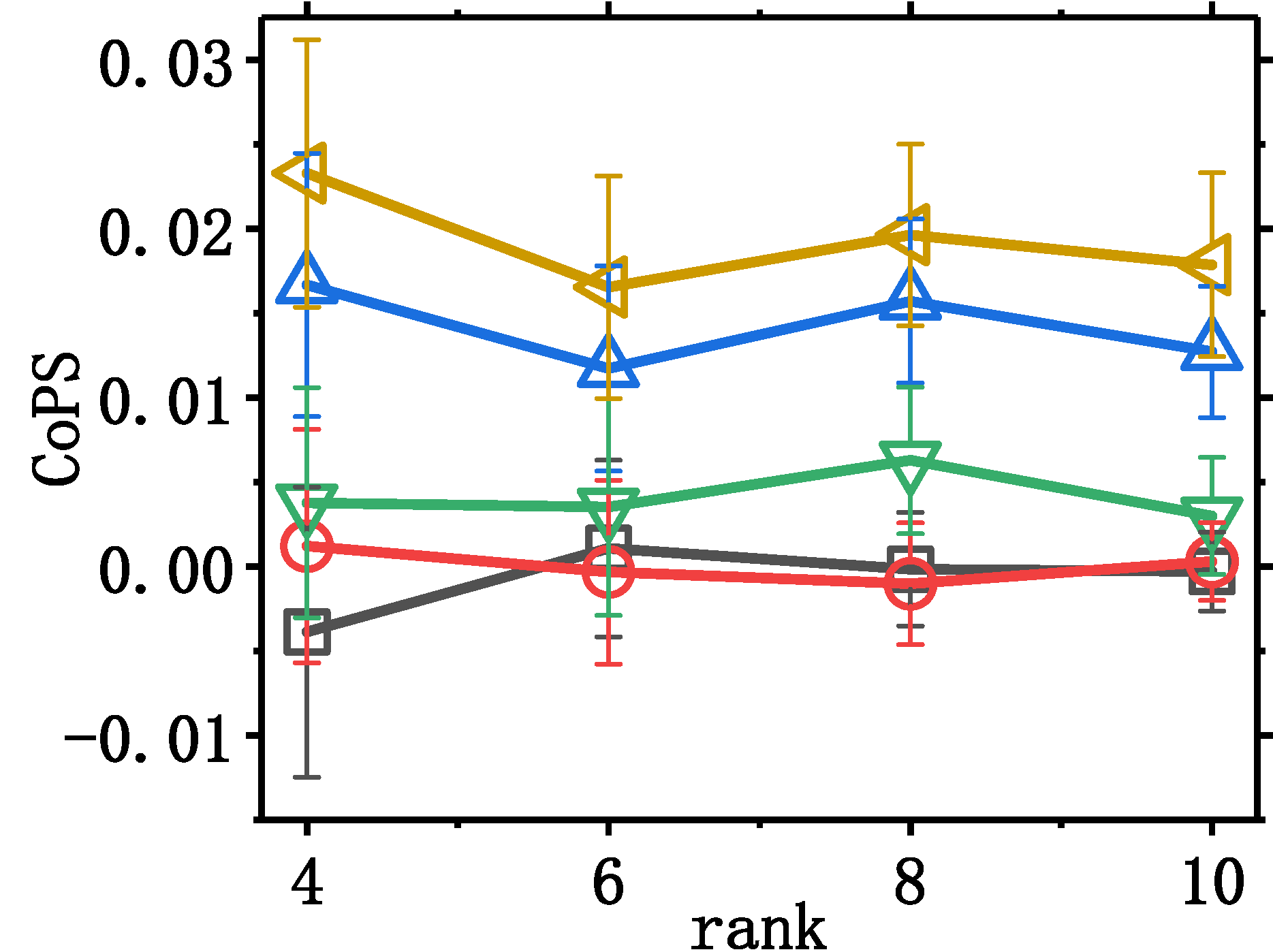}
		\end{minipage}%
	}%
	\subfigure{
		\begin{minipage}[t]{0.48\columnwidth}
			\centering
			\includegraphics[width=\textwidth]{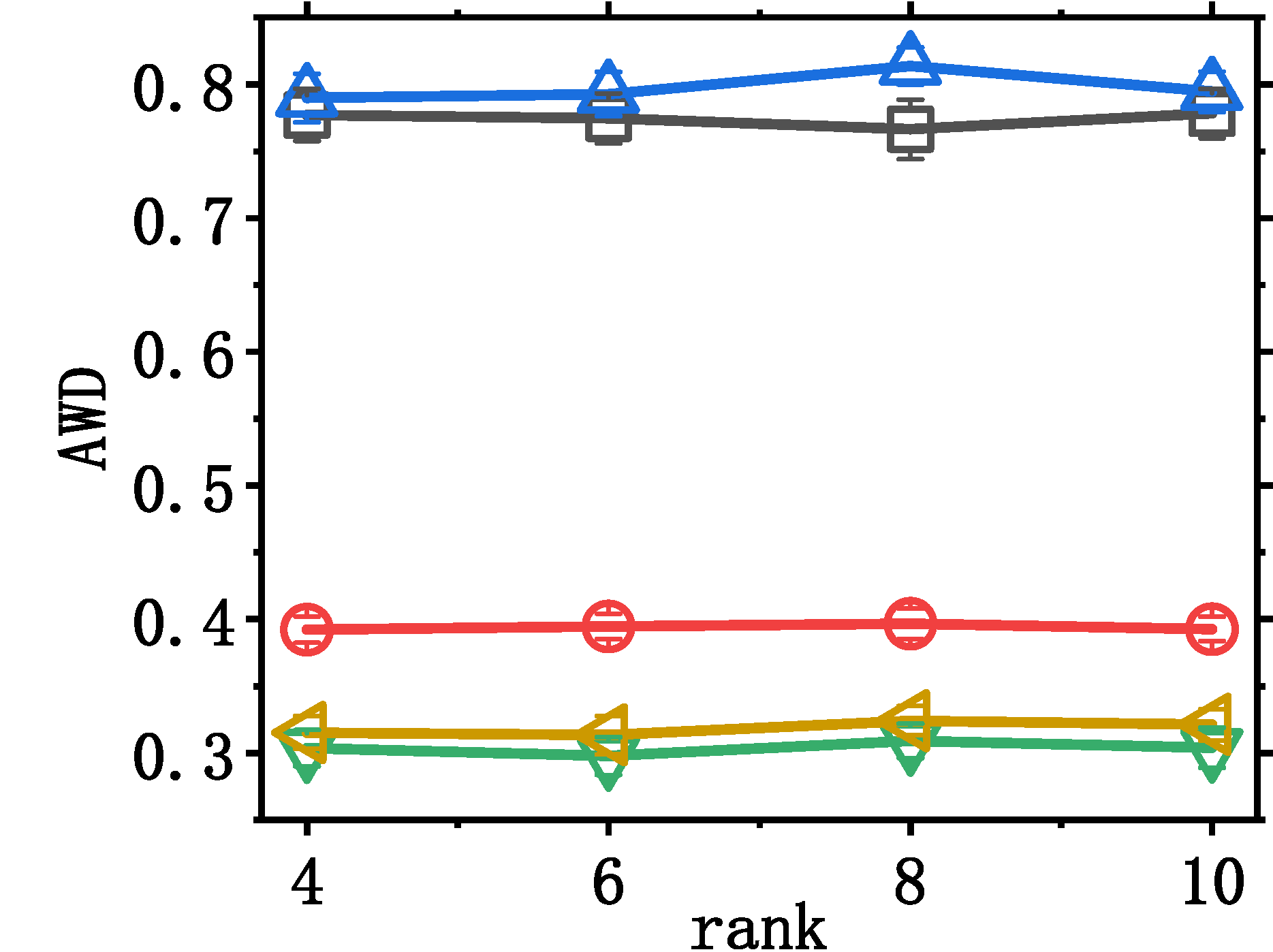}
		\end{minipage}
	}%
	\subfigure{
		\begin{minipage}[t]{0.48\columnwidth}
			\centering
			\includegraphics[width=\textwidth]{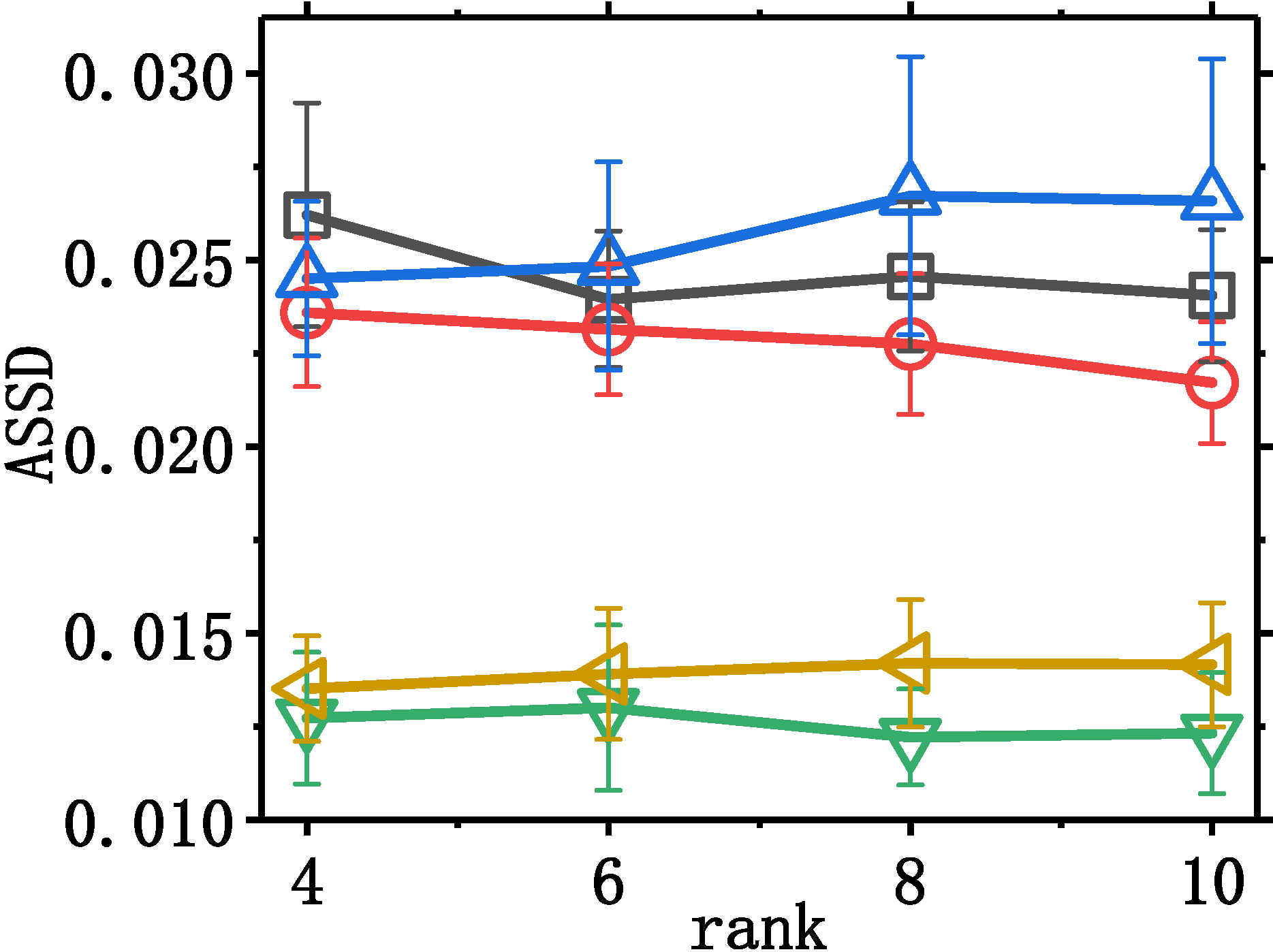}
		\end{minipage}
	}%
	
	\centering
    
\caption{Promotion attack with varying original rank of the target topic, on the AP dataset, showing 95\% confidence interval.}
	\label{ldac_diff_rank}
\end{figure*}

\noindent\textbf{Evaluation metrics}
We use the following $4$ metrics for evaluation, where the first $2$ measure the \textit{effectiveness} of attack, and the last $2$ measure the \textit{evasiveness}. i) \textit{Change of rank (\textbf{CoR})} is the change of rank for a topic $k$:
$\text{CoR}=|\text{Rank}'_k - \text{Rank}_k|$ where $\text{Rank}_k$ and $\text{Rank}_k'$ are respectively the rank of topic $k$ before and after the attack. ii) \textit{Change of probability score (\textbf{CoPS})} means the change in probability score of target topic $k$:
$\text{CoPS}=|\theta_k - \theta^{adv}_k|$. 
iii) \textit{Average word distance (\textbf{AWD})} is the average distance between the target-replacement word pairs in the word vector space: $\text{AWD}=1/|\mathcal{W}|\sum_{(w,w')\in (\mathcal{W},\mathcal{W}')}(1-\cos(w,w')).$ 
iv) \textit{Average sentence semantic distance} (\textbf{ASSD}).
We first use BERT~\citep{devlin2019bert} 
to encode sentences into high dimensional vectors. We then calculate the accumulated cosine distance of all sentences that are perturbed to measure the semantic distance of the victim and adversarial documents. Denote the original sentence and its perturbed version respectively as $\mathbf{s}$ and $\mathbf{s}^{adv}$, this is defined as: $\text{ASSD}=1/|\mathcal{S}|\sum_{\mathbf{s}\in \mathcal{S}}(1-\cos(\mathbf{s},\mathbf{s}^{adv})),$ where $\mathcal{S}=\{\mathbf{s}\ |\  \mathbf{s}\ni w \cap w\in \mathcal{W}\}$ denotes the set of sentences which contain adversarial words. $\mathbf{s}\ni w$ means sentence $s$ contains $w$. ASSD measures sentence level evasiveness.

\begin{figure*}
	\centering
	\subfigure{
		\begin{minipage}[t]{0.48\columnwidth}
			\centering
			\includegraphics[width=\textwidth]{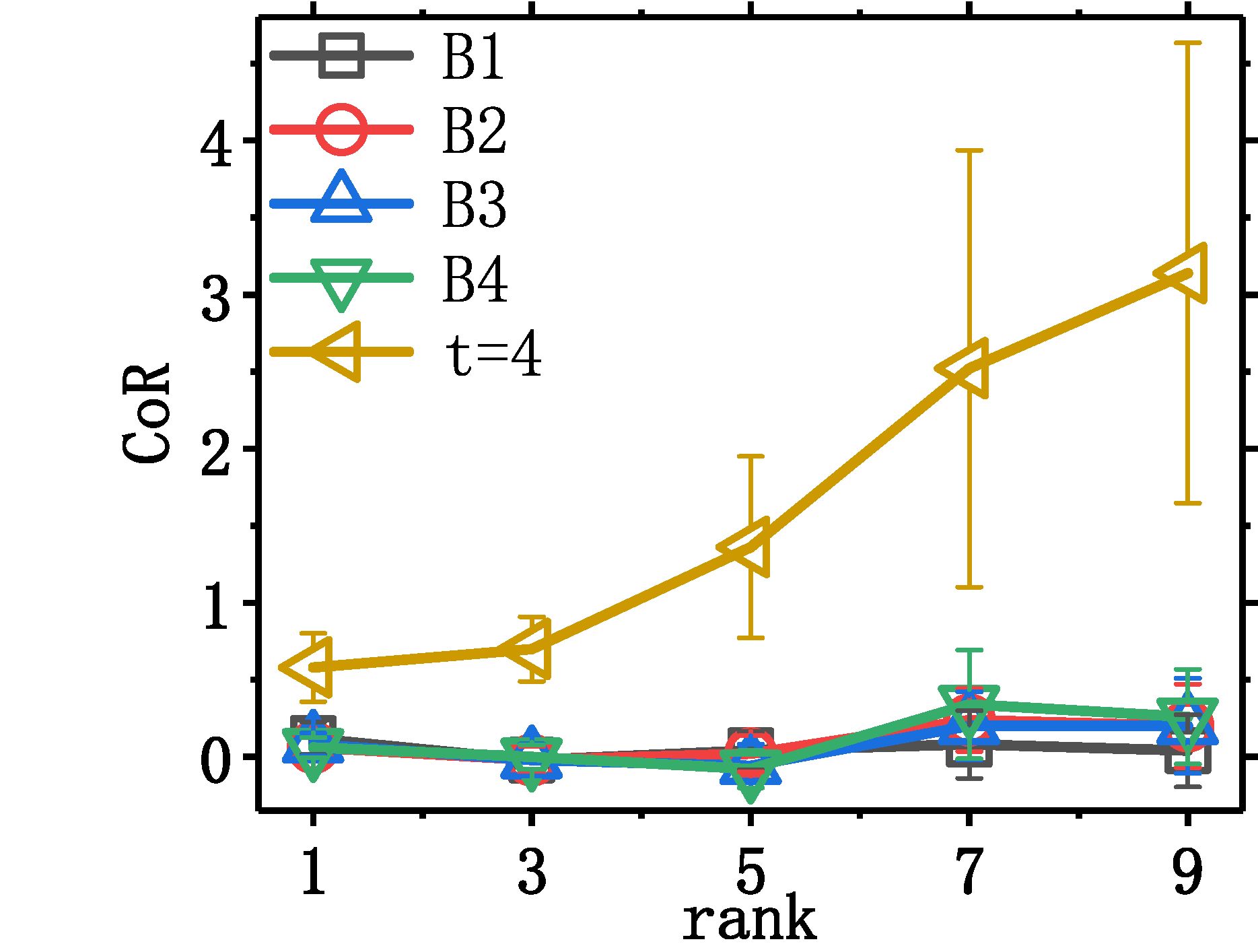}
		\end{minipage}%
	}%
	\subfigure{
		\begin{minipage}[t]{0.48\columnwidth}
			\centering
			\includegraphics[width=\textwidth]{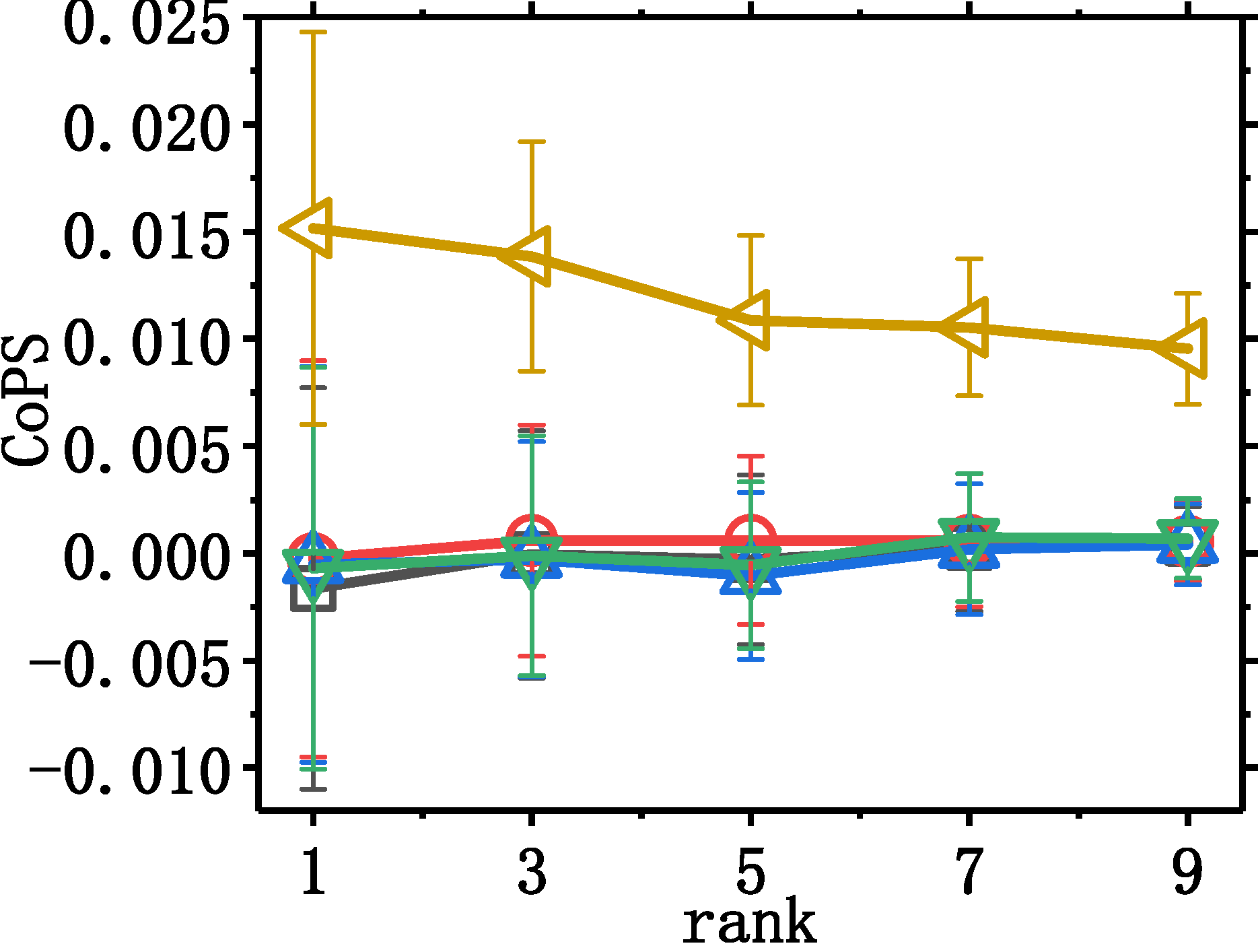}
		\end{minipage}%
	}%
	\subfigure{
		\begin{minipage}[t]{0.48\columnwidth}
			\centering
			\includegraphics[width=\textwidth]{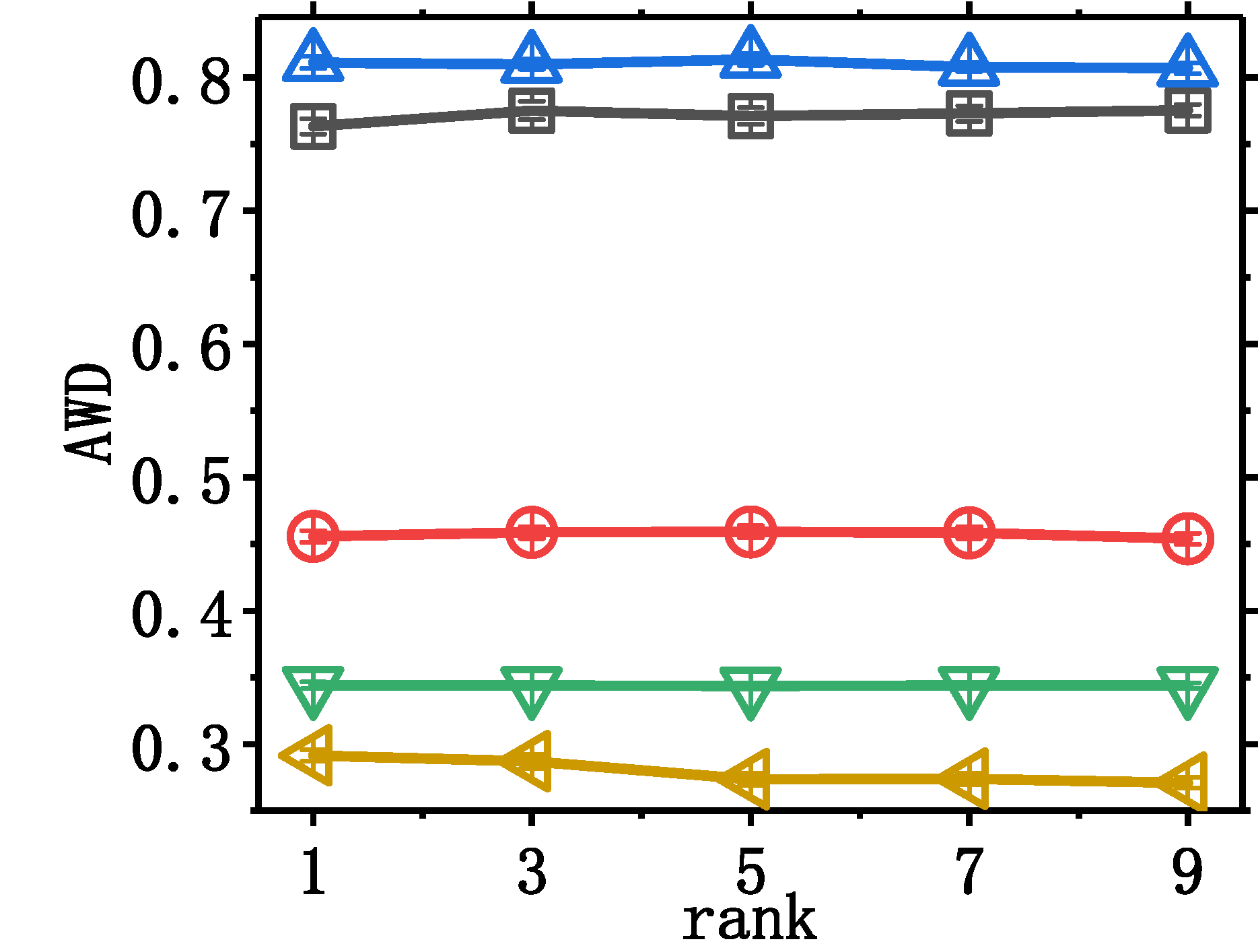}
		\end{minipage}
	}%
	\subfigure{
		\begin{minipage}[t]{0.48\columnwidth}
			\centering
			\includegraphics[width=\textwidth]{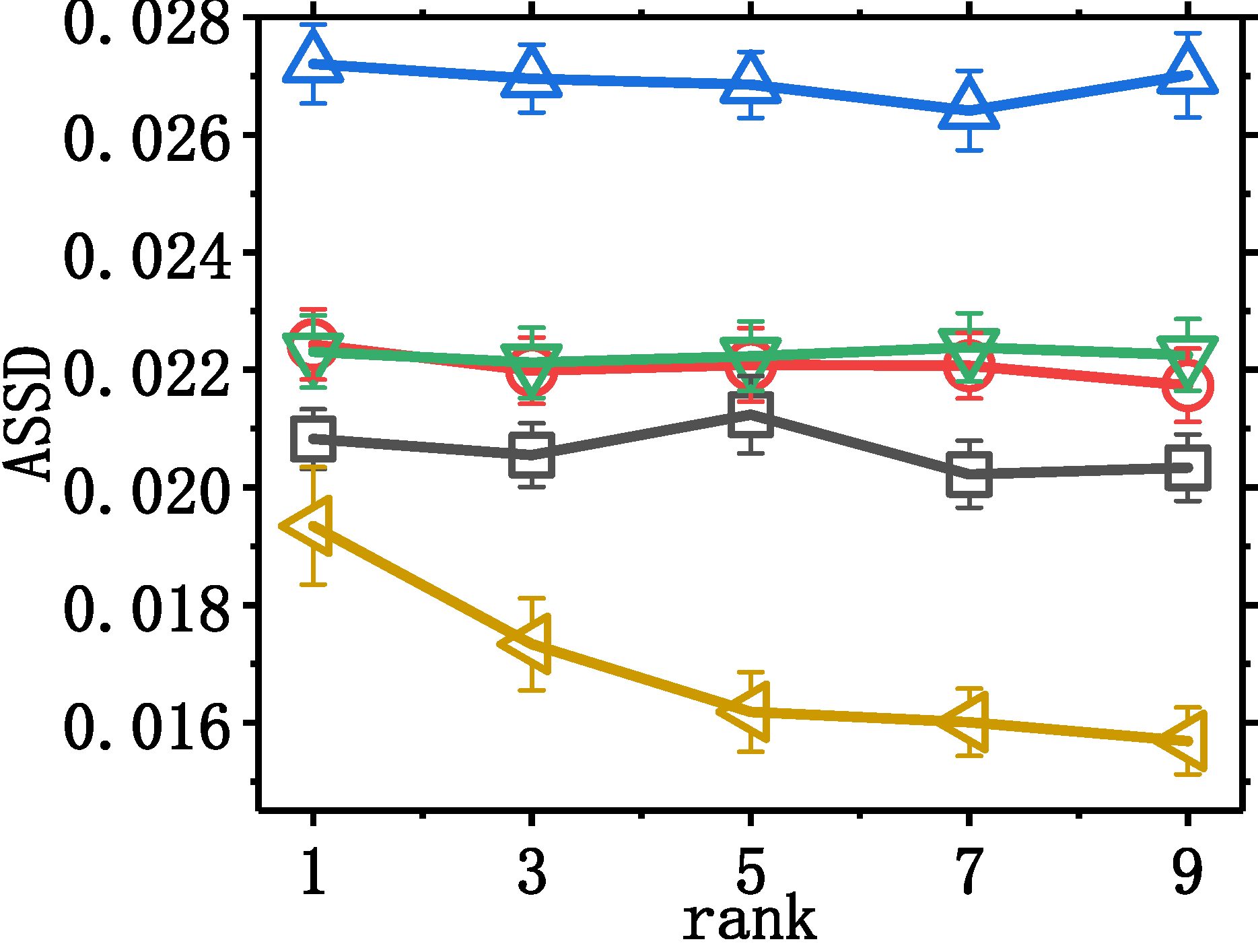}
		\end{minipage}
	}%
	
	\centering
	\caption{Demotion attack with varying original rank of target topic, on the NIPS dataset, showing 95\% confidence interval.}
	\label{fig:nips_demotion}
\end{figure*}

\begin{figure*}
	\centering
	\subfigure{
		\begin{minipage}[t]{0.48\columnwidth}
			\centering
			\includegraphics[width=\textwidth]{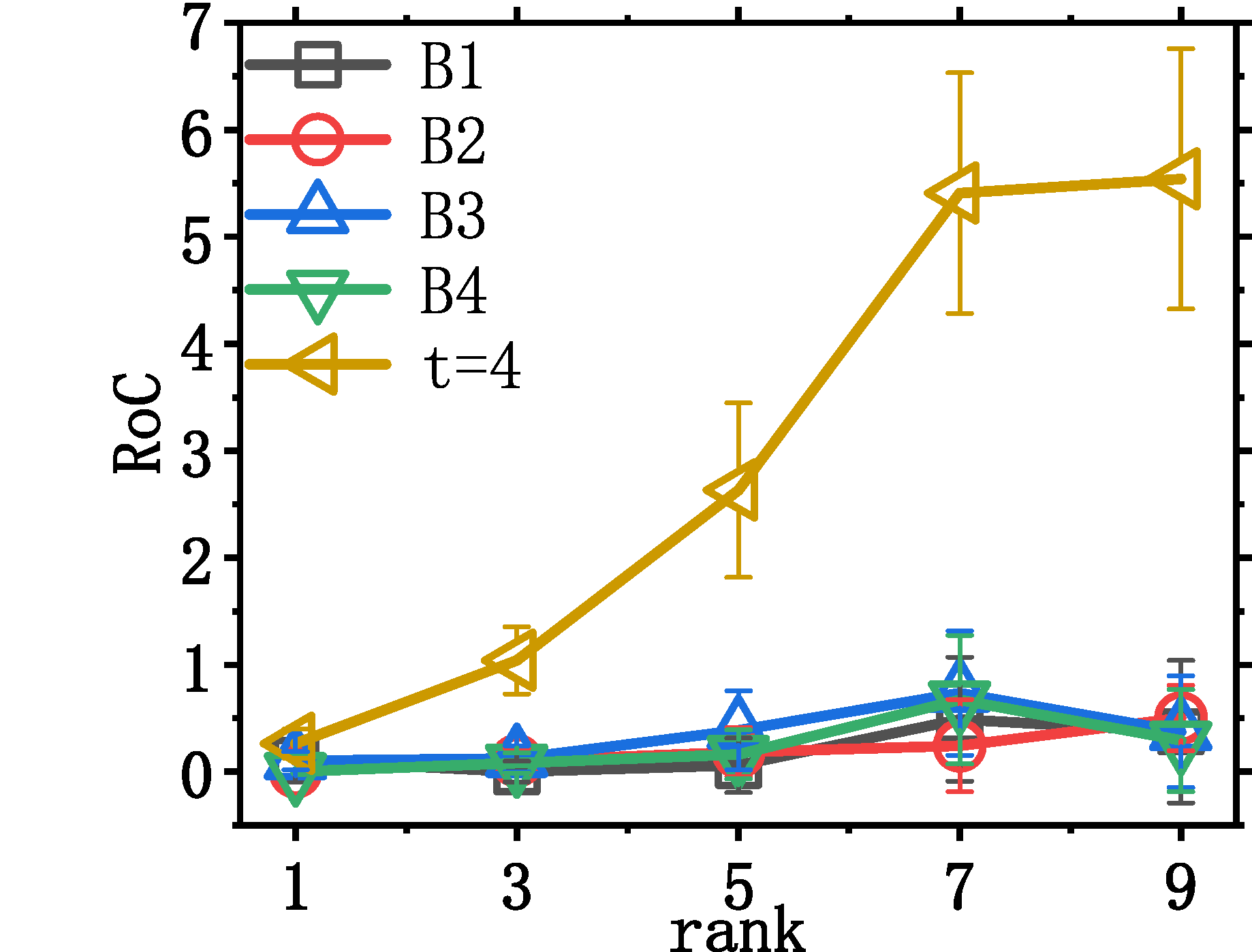}
		\end{minipage}%
	}%
	\subfigure{
		\begin{minipage}[t]{0.48\columnwidth}
			\centering
			\includegraphics[width=\textwidth]{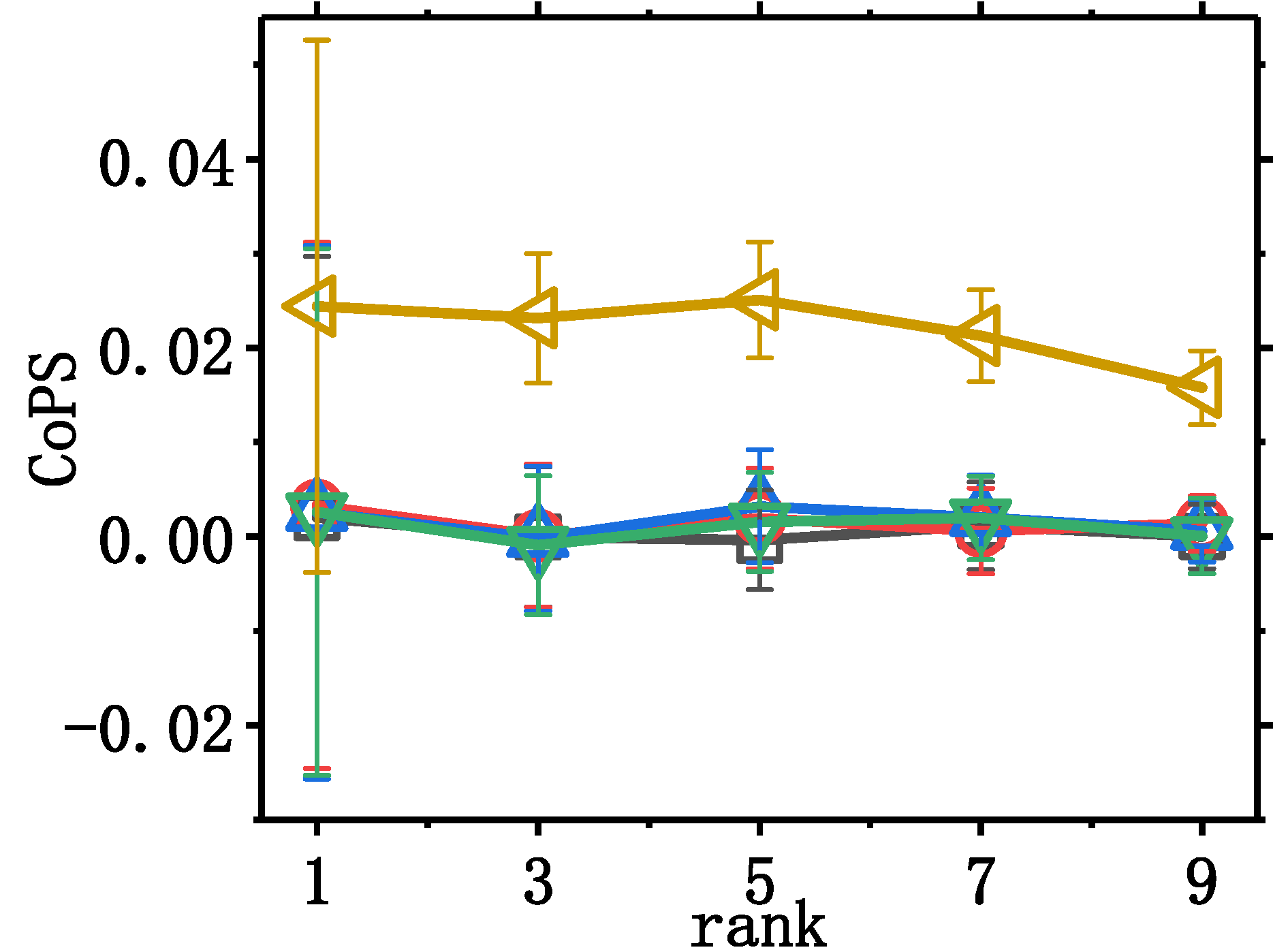}
		\end{minipage}%
	}%
	\subfigure{
		\begin{minipage}[t]{0.48\columnwidth}
			\centering
			\includegraphics[width=\textwidth]{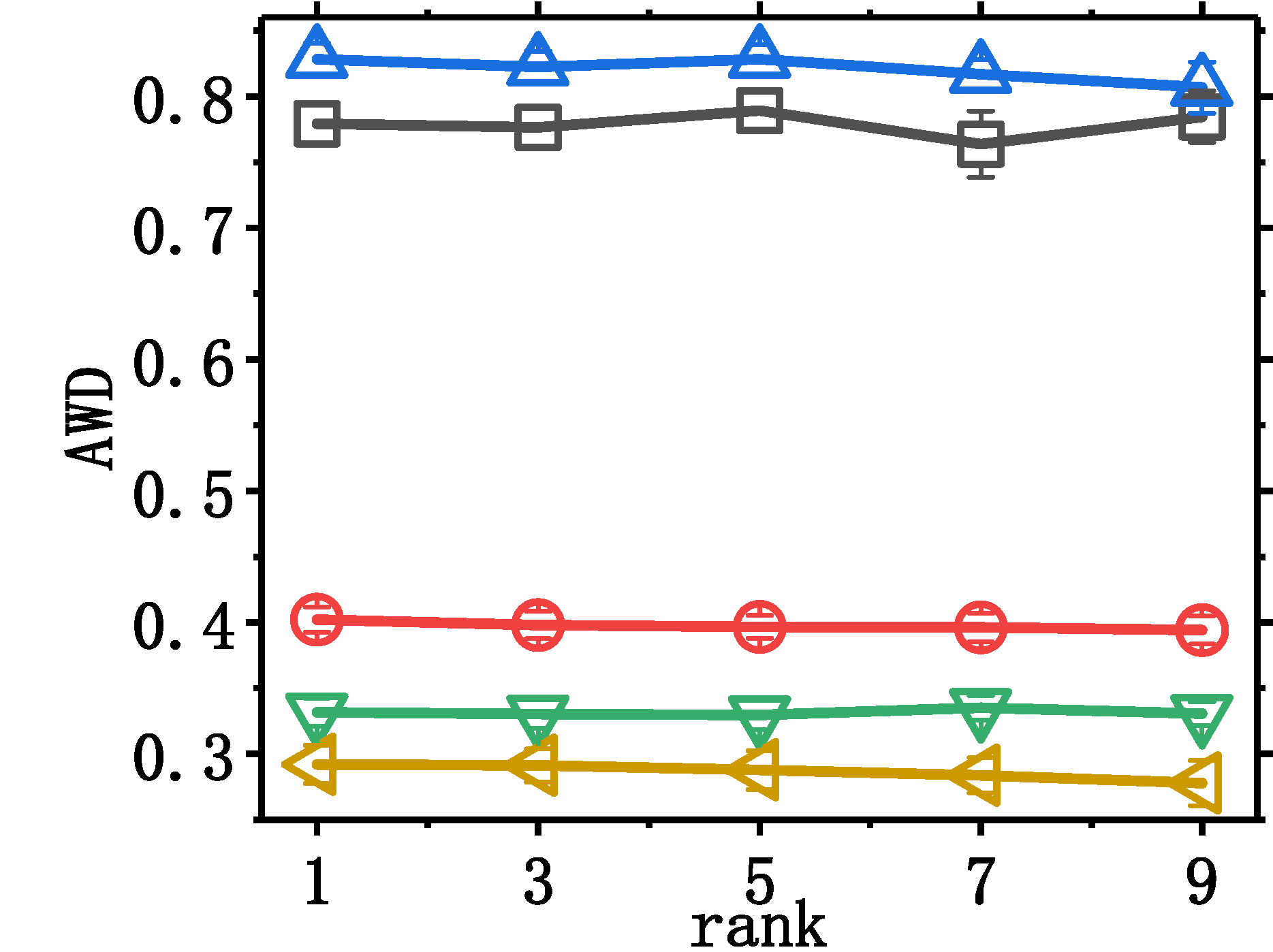}
		\end{minipage}
	}%
	\subfigure{
		\begin{minipage}[t]{0.48\columnwidth}
			\centering
			\includegraphics[width=\textwidth]{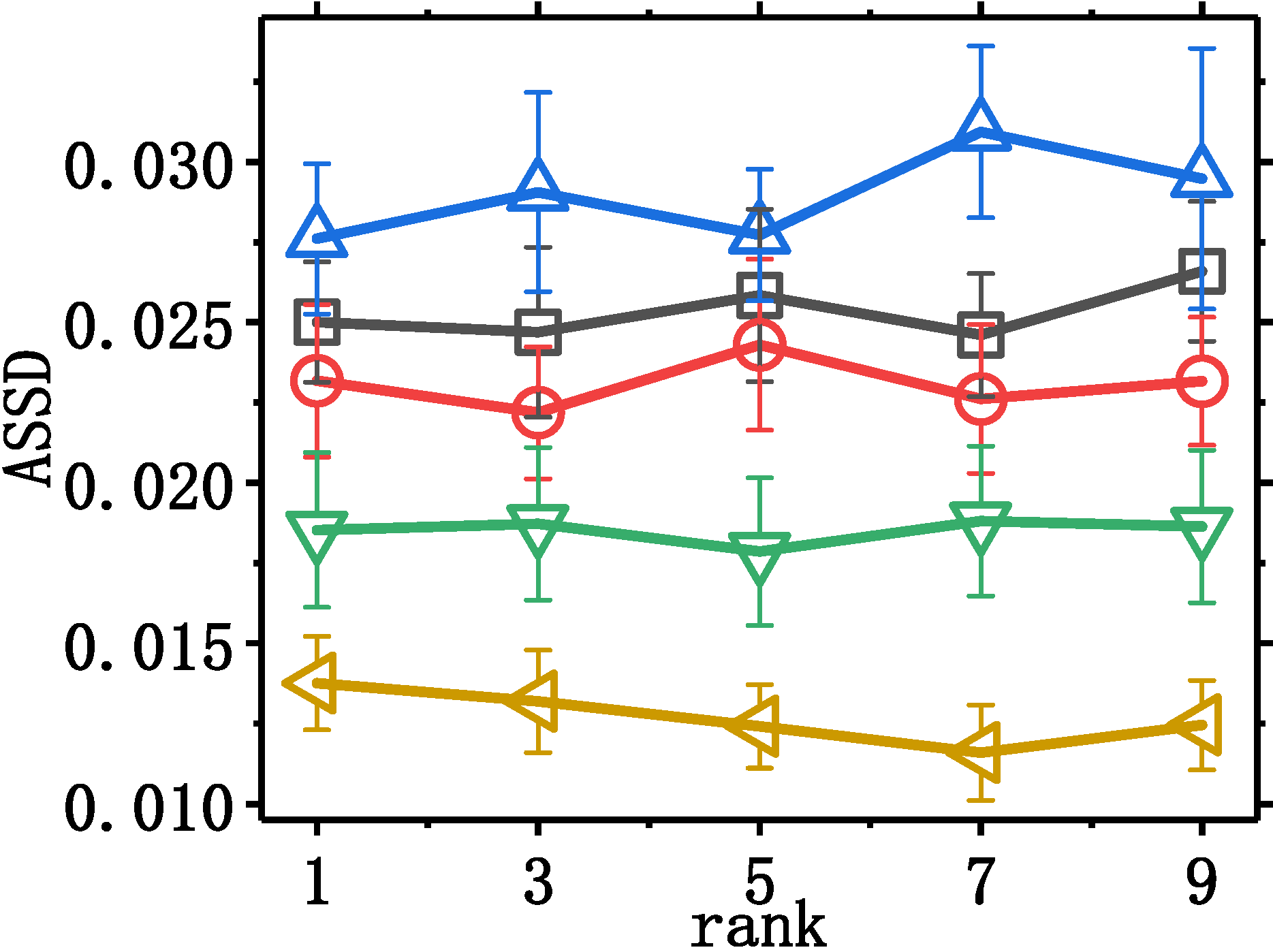}
		\end{minipage}
	}%
	
	\centering
	
	\caption{Demotion attack with varying original rank of target topic, on the AP dataset, showing 95\% confidence interval.}
	\label{ldac_demotion}
		
\end{figure*}

\subsection{Promotion Attack}
We first evaluate the performance of promotion attack on the two datasets using the above $4$ metrics, under the following settings: i) Approximation level $t$ (in Eq.\eqref{eq:analytic_estimate}) ranges over $1-6$ (default $4$). ii) Word substitution strategies (Step $1$ in the previous section), including synonym, word-embedding and a mixture of the two (default mixture).
iii) Perturbation threshold $\kappa$ (in Eq.\eqref{eq:constr_capacity}), ranges over $[0.5\%, 1\%, 2\%, 3\%]$ (default $\kappa=1\%$). (iv) Original rank of target topic, ranges over $[5, 10, 15, 20]$ (default $10$). When evaluating one parameter, we fix the other parameters with default values. We randomly choose $50$ test document samples as victim samples. For all settings, we set word distance threshold $\sigma=0.6$. For ASSD, we treat the pre- and proceeding $5$ words combined of the current target word as a sentence.
For the AP dataset, the different experiment settings for promotion attack are the same as before:i) Approximation level $t$ (default 4); ii) Word substituion strategies (default mix); iii) Perturbation threshold $\kappa$, which ranges over $[1\%, 2\%, 3\%, 4\%]$ (default $\kappa=2\%$).
(iv) Original rank of target topic, which ranges over $[4, 6, 8, 10]$ (default $6$). 
As shown in Table 1, documents in the AP data are much shorter (192 on average) than the NIPS dataset (3211 on average), using a perturbation threshold of $1\%$ means for some documents whose length is less than 200, only one word can be perturbed. This makes it too challenging for the attack algorithms including EvaLDA. Therefore we set larger default perturbation thresholds for the experiments on the AP dataset. Similarly, because the documents are much shorter in the AP dataset, the number of topics with non-zero probability scores are also much smaller than that in the NIPS dataset. Therefore, the original rank of the target topic is set higher and denser than that in the NIPS dataset.
We still randomly choose $50$ samples and set word distance threshold $\sigma = 0.6$.

\noindent\textbf{Approximation level}
Figure \ref{fig:nips_diff_level} shows the results with different $t$ values from $1$ to $6$ on the \textit{NIPS} dataset. We can see that by replacing only $1\%$ of the words, EvaLDA  significantly outperforms the baselines in promoting the rank and probability score of the target topic. In particular, the best CoR result is obtained when $t=4$, where EvaLDA promotes the original rank from $10$ to around $7$ on average. We also notice that the performance of EvaLDA in all metrics is not sensitive to $t$ values. This indicates the \textit{robustness} of EvaLDA. Meanwhile, the AWD and ASSD values are very small (approximately 0.3 for AWD and 0.02 for ASSD). This indicates that the average perturbation per word is reasonable and the average perturbation per sentence is almost negligible. Note that \textbf{B4} achieves the best performance in AWD and ASSD. This is expected as \textbf{B4} always selects replacements as the nearest neighbor of the target word in the vector space. However, the effectiveness performance of \textbf{B4} is far worse than EvaLDA. 
Figure \ref{ldac_diff_level} shows the results for the \textit{AP} dataset. For all the metrics, similar patterns are observed compared with experiments on the NIPS dataset. 
Note that we set the perturbation threshold as $2\%$, original rank $6$ and mixed word substitution strategy in this setting.
Based on this set of experiments, we choose level $4$ for the rest of the experiments on both datasets. 

\begin{table}[pht]
	\centering
	\setlength{\tabcolsep}{3pt}
	\small
	\begin{tabular}{ccccc}
		\hline
		& CoR & CoPS & AWD & ASSD \\ \hline
		EM & 2.60$\pm$0.45 & 0.017$\pm$0.003 & 0.334$\pm$0.004 & 0.021$\pm$0.001 \\
		Syno & 0.54$\pm$0.31 & 0.003$\pm$0.002 & \textbf{0.227$\pm$0.009} & \textbf{0.013$\pm$0.001} \\
		Mix & \textbf{2.60$\pm$0.43} & \textbf{0.017$\pm$0.003} & 0.332$\pm$0.003 & 0.021$\pm$0.001\\ \hline
	\end{tabular}
	\caption{Evaluate word substitution strategies on NIPS dataset }\label{tab:diff_similar}
\end{table}
\begin{table}[pht]
	
	\centering\small
	\setlength{\tabcolsep}{3pt}
	\begin{tabular}{lllll}
		\hline
		& CoR & CoPS & AWD & ASSD \\ \hline
		EM & 0.89$\pm$0.26 & 0.023$\pm$0.007 & 0.315$\pm$0.018 & 0.014$\pm$0.002 \\
		Syno & 0.43$\pm$0.25 & 0.012$\pm$0.008 & \textbf{0.260$\pm$0.018} & \textbf{0.010$\pm$0.001} \\
		Mix & \textbf{0.93$\pm$0.27} & \textbf{0.023$\pm$0.008} & 0.315$\pm$0.018 & 0.014$\pm$0.002 \\ \hline
	\end{tabular}
	\caption{Evaluate word substitution strategies on AP dataset} 
	\label{tab:ldac_diff_similar}
\end{table}


%

\noindent\textbf{Word substitution strategy}
Table \ref{tab:diff_similar} shows the result on the \textit{NIPS} dataset when varying word substitution strategies. Note that the performances of the baselines are not dependent on the word substitute strategies, and therefore their performances are the same as in Figure~\ref{fig:nips_diff_level} and we do not repeat them here. We can see that word embedding strategy (EM) achieves much better CoR and CoPS results than word synonyms strategy (Syno), while Syno achieves better AWD and ASSD results. This is intuitive because the solution space of Syno is significantly smaller than EM, as words usually have few synonyms. The mixture of the two (Mix), achieves slightly better results in all metrics compared with EM. This indicates that Syno, while performing worst in terms of effectiveness by its own, can be a good complementary to EM.
Table \ref{tab:ldac_diff_similar} shows the result on the \textit{AP} dataset. Similarly, EM has better performance than Syno. A mix of the two has the best performance in CoR and CoPS, despite that AWD and ASSD of mix are almost the same as EM.


\noindent\textbf{Perturbation threshold} 
Figure \ref{nips_diff_thre} shows performance comparisons of EvaLDA with baselines under varying perturbation threshold values, on the \textit{NIPS} dataset. We can see that in general, as the threshold increases, both CoR and CoPS increase. Similarly, EvaLDA is far superior to the baselines under all threshold values in terms of CoR and CoPS and is comparable to the baselines in  AWD and ASSD. Be reminded that the best performance of AWD and ASSD of \textbf{B4} is due to the fact that it always selects the nearest neighbor of the target word in vector space. Interestingly, we see that the average sentence level evasiveness value (ASSD) of EvaLDA decreases when threshold value increases. Our conjecture for this is when threshold increases (i.e., solution space is larger), the newly found target-replacement word pairs by EvaLDA are closer.
Figure \ref{ldac_diff_thre} shows the result on the \textit{AP} dataset. In terms of effectiveness (CoR and CoPS), EvaLDA is superior to the baseline methods under almost all threshold values. Note that \textbf{B3} also performs well under different perturbation thresholds. This is because \textbf{B3} selects a random replacement word for the target word and it is more likely to shift the target topic. However, such attacks are impractical as it is highly detectable due to the large word-level and sentence variations, which are indicated by the very large AWD and ASSD values.

\noindent\textbf{Original rank of target topic}
Intuitively, promotion attacks on topics that are ranked lower should be more effective than topics that are ranked higher. Results in both Figures \ref{nips_diff_rank} and \ref{ldac_diff_rank} follow this intuition. In addition, since EvaLDA does not consider the rank of the original topic in the optimization process, it works constantly in CoPS regardless of the original rank. Moreover, we can see that the AWD values of EvaLDA are comparable with \textbf{B1}-\textbf{B3} and slightly worse than \textbf{B4}. 
As shown in Figure \ref{ldac_diff_rank}, the overall results on the \textit{AP} dataset are similar to those of the \textit{NIPS} dataset. The lower ranked topics are easier to promote, even if topics with higher rankings can get more score promotion, its ranks are still relatively stable. As we can see, EvaLDA performs better in CoR and CoPS, and also competitive in AWD and ASSD.

\subsection{Demotion Attack}
As shown in Eq.\eqref{eq:demote}, rank demotion attack aims at decreasing the rank of a target topic. Figure \ref{fig:nips_demotion} shows the result of rank demotion attack with varying original rank of the target topic, on the \textit{NIPS} dataset. We can see that, intuitively, the CoPS value decreases when the original rank is lower. However, the CoR metric does not decrease w.r.t. the original rank. This seems to be counter-intuitive, but in fact is correct because although the attack to originally higher ranked topics achieves higher CoPS values, the original probability scores of these highly ranked topics are also higher. Therefore it is difficult to change it from a high score to a very low score, rendering it difficult to demote.
In terms of AWD and ASSD, EvaLDA performs even better than \textbf{B4}. This is because the solution space is larger than that of promotion attack, so EvaLDA can achieve aggressive goals while ensuring similarity.
Figure \ref{ldac_demotion} shows the topic demotion attack performance on the \textit{AP} dataset. We can see that EvaLDA beats the baselines by a very large margin for CoR and CoPS, while having the best AWD and ASSD values too, which are even smaller than \textbf{B4}.

\begin{figure}
	\centering
	\includegraphics[width=0.99\columnwidth]{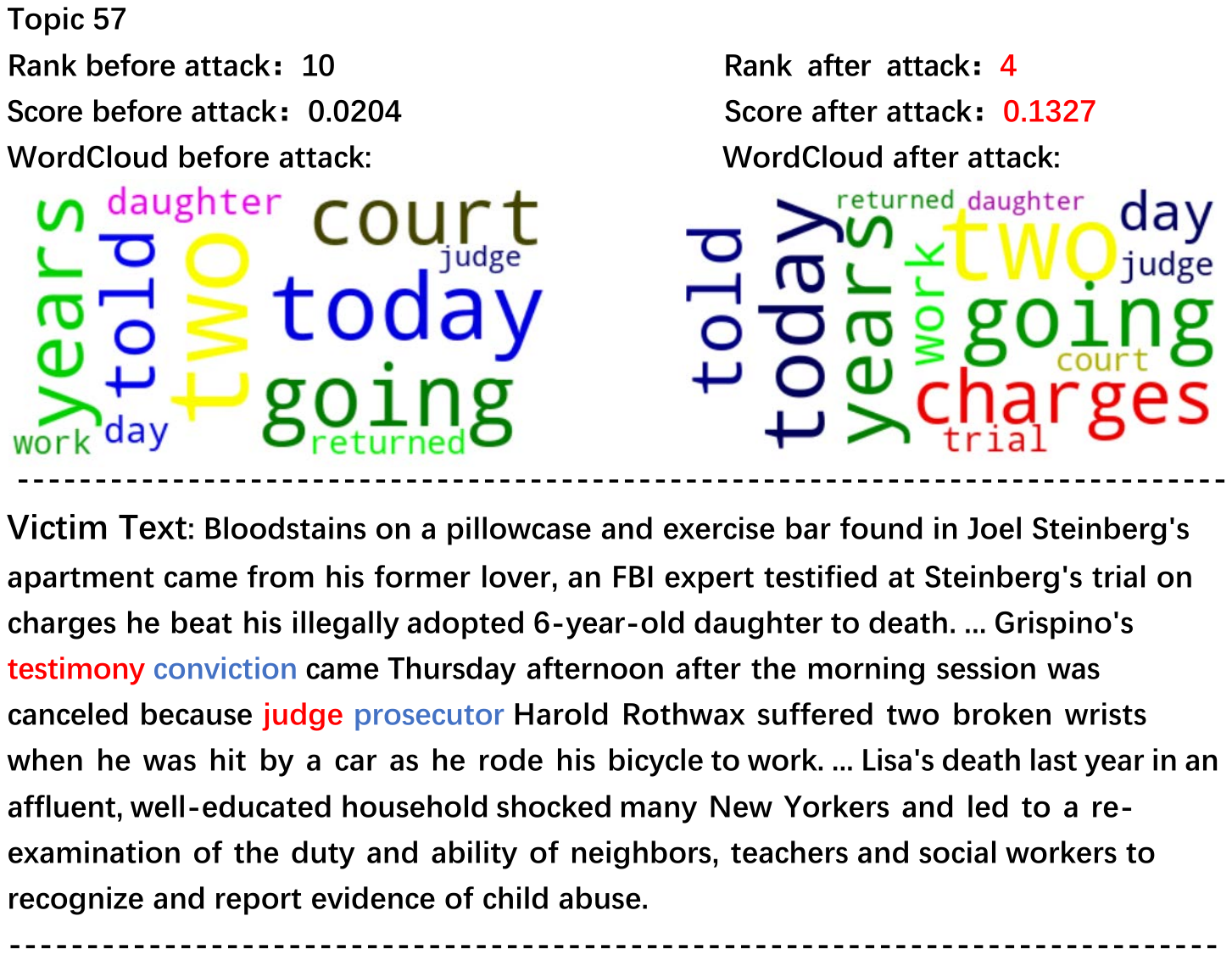}
	\caption{An adversarial sample generated by EvaLDA. Target and replacement words are resp. in red and blue. By replacing 2 words ``testimony" and ``judge" with ``conviction" and ``prosecutor", the rank of the topic is greatly promoted.}
	
	\label{fig:case}
	
\end{figure}
\subsection{Case Study}
We now show via a concrete example how EvaLDA works. Because documents from the AP dataset are shorter and the contents are news from Associated Press, we select a victim document sample from the AP dataset. The length of the document is 111 words. We generate an adversarial sample using EvaLDA with $t=4$ and perturbation threshold $\kappa= 0.02$ -- which means we can perturb only 2 words. Figure \ref{fig:case} shows a piece of the sample document, where target and replacement words are respectively in red and blue. We choose topic $57$ as the target topic, which ranks $10$th originally. We can see that by replacing the words "testimony" and "judge" with reasonably close words "conviction" and "prosecutor", the rank of the topic is promoted to $4$ and the topic probability score increases from $0.0204$ to $0.1327$. At the top of the text, the topic WordCloud before and after the attack also shows that the two new words "charges" and "trial" appear with evidently greater weights.


\section{Conclusion, Limitations, and Future Work}
This work is the first to study evasion attacks to LDA models. The formulated optimization problem, which is provably \textit{NP-hard}, is solved via our proposed novel algorithm EvaLDA. EvalDA consists of an efficient estimate of the topic-word distribution via a surrogate CGS-based inference procedure and a greedy target word selection and replacement procedure. Via extensive experimental evaluations on two distinct datasets, we show that EvaLDA achieves superb performances in both topic rank promotion and demotion attacks under various settings.
Despite the advantages, EvaLDA does have a few \textit{limitations}. First, it is limited to attacks to CGS-based LDA models. It remains unclear whether it works on VI-based LDA models. A potential idea is via attack strategy transferred from a CGS-based LDA model. Second, EvaLDA assumes a white-box setting and acts more like a proof-of-concept analysis in the worst-case scenario. It is interesting to see whether attacks can be effective in a black-box setting. Another critical future work, which is also the ultimate goal of studying adversarial attacks to LDA models, is to design effective defense strategies to against such attacks.

\section{Acknowledgements}
The paper is supported by the National Natural Science
Foundation of China (Grant No. 62002092). H. C. is funded by the Harvard Center for Research on Computation and Society. 

\section{Ethical Impact}
We aim at providing insights into the power and limitations of adversarial attacks towards LDA models. By exposing this work, users and service providers of LDA models can be alerted the potential risks of such models. Moreover, design defense strategies towards such attacks can be designed to improve the robustness of the models. The potential downside is that malicious hackers may be aware of the vulnerability of the LDA models and make exploits of it. However, without being aware of the vulnerabilities would pose critical risk of their product and systems. Therefore, exposing such vulnerabilities of LDA models is the necessary first step.

\newpage

\bibstyle{aaai21}
\bibliography{lda_attack}

\end{document}


\linenumbers
\maketitle

\section{Proof of Theorem 1}
\begin{theorem}
Given an oracle that tells the algorithm the explicit value of $Q(\mathcal{W},\mathcal{W}')$ for an attack strategy $(\mathcal{W},\mathcal{W}')$, the \textit{Attack-LDA} problem formulated above is $\mathcal{NP}$-\textit{hard}.
\end{theorem}
\begin{proof}
The key idea is to reduce the (binary) combinatorial optimization problem with cardinality constraints (COPCC, which is proven to be $\mathcal{NP}$-\textit{hard}~\cite{bruglieri2006annotated}) to our defined \textit{Attack-LDA} problem. 

An \textit{arbitrary} instance of COPCC can be expressed as:
\begin{align}\label{eq:copcc-instance}
\min_x  f(x) \ \ 
\text{s.t.} \ \  x\in \{0,1\}^d : ||x||_0\leq C,
\end{align} where $x$ is a $d$-dimensional binary indicator vector which corresponds to the selection of items. That is, $x_i=1$ indicates the $i$-th item is selected and otherwise not. $|| x||_0$ denotes $l$-0 norm of $x$. 

We then construct a \textit{special} instance of the \textit{Attack-LDA} problem as follows. We first let 
\begin{align*}
Q^*(\mathcal{W}) &= \max_{\mathcal{W}'} Q(\mathcal{W},\mathcal{W}')\\
\text{s.t.} \quad &D(w,w')\leq \sigma, \quad \forall (w,w') \in (\mathcal{W}, \mathcal{W}').
\end{align*} The \textit{Attack-LDA} problem is then transformed as 
\begin{align}\label{eq:attack-lda-instance}
\max_{\mathcal{W}} \quad Q^*(\mathcal{W}) \ \ 
\text{s.t.} \ \  | \mathcal{W}| \leq |\mathbf{w}^{vic}| \cdot  \kappa
\end{align}
Because finding the optimal $w'$ for a given $w$ requires enumerating the vocabulary space only once, therefore finding $Q^*(\mathcal{W})$ is polynomial ($\mathcal{O}(|\mathbf{w}^{vic}|\cdot |\mathcal{V}^{vocab}|)$) in the document size $|\mathbf{w}^{vic}|$ and the vocabulary size $|\mathcal{V}^{vocab}|$. 

We construct the correspondence as: $d \xleftrightarrow\ |\mathbf{w}^{vic}|$, $x_i=1 \xleftrightarrow\ w_i\in \mathcal{W}$, $f(x) \xleftrightarrow\ Q^*(\mathcal{W})$, $C \xleftrightarrow\ |\mathbf{w}^{vic}| \cdot  \kappa$. In this sense, we can easily get that if $x$ is an "yes" instance of Eq.\eqref{eq:copcc-instance}, then the corresponding $\mathcal{W}$ is an "yes" instance of Eq.\eqref{eq:attack-lda-instance} and vice-versa.
\end{proof}

\section{Proof of Lemma 1}
\begin{lemma}\label{lemma:recursive_theta}
In the above designed surrogate inference procedure, when $\alpha\rightarrow 0$,\footnote{Be reminded that the hyper-parameter $\alpha$ of the Dirichlet distribution can be interpreted as a regularization term of the document-topic distribution $\theta_m$ based on prior knowledge. In practice, $\alpha$ is a very small value that is approximately equal to $1/K$ where $K$ is the number of topics. Therefore the assumption is reasonably made.} there exists a recursive definition of the topic distribution $\theta_k^t$ for each topic $k=1,\dots,K$:
\begin{align}\label{eq:recursive_theta}
	\theta_{k}^t = \frac{\theta_k^{t-1}}{N} \sum_{v\in \mathcal{V}}\!n_v \varphi_{kv},
\end{align} where $t$ is the number of iterations in the CGS procedure, and $N=|\mathbf{w}_m|$ is the number of words in the test document $\mathbf{w}_m$.
\end{lemma}	
\begin{proof}
At iteration $t$, denote the full conditional probability of sampling a topic $k$ for word $v$ as $p_{kv}^t$, then $p_{kv}^t$ is the same at each of the $n_v$ sampling operations:
\begin{align*}
p_{kv}^t=\varphi_{kv} \cdot \frac{N_{k}^{t-1}+\alpha}{N^{t-1}+K\alpha}
\end{align*}
Here since only the test document $\mathbf{w}_m$ is involved, we have omitted the sub-script $m$ for clarity of notation. 

Because the surrogate inference procedure goes over the entire document at each iteration $t$, the sum of topic count always equals the total number of words $N$ in the test document. When $\alpha\rightarrow 0$ and $K\alpha \rightarrow 1$, the above equation is re-written as:
\begin{align*}
p_{kv}^t=\varphi_{kv} \cdot \frac{N_{k}^{t-1}}{N}
\end{align*}
The approximation on the denominator holds as $N\gg 1$.

When sampling repeats $n_v$ times, the \textit{expected} count of topic $k$ being sampled at word $v$ is $N_{kv}^t=n_v p_{kv}^t$.
Note that we have omitted the expectation symbol for clarify of notation. Thus, the total \textit{expected} count of topic $k$ for the test document is:
\begin{align*}
N_{k}^t=\sum_{v\in \mathcal{V}}n_v p_{kv}^t
\end{align*}
Combining the above two equations,
\begin{align*}
	N_{k}^t=\!\sum_{v\in \mathcal{V}}\!n_v \varphi_{kv}  \frac{N_{k}^{t-1}}{N}\! =\! \frac{N_{k}^{t-1}}{N} \!\sum_{v\in \mathcal{V}}\!n_v \varphi_{kv}   
\end{align*}
The second equation holds because $\frac{N_{k}^{t-1}}{N}$ does not depend on $v$.
According to Eq.(2) in the main text,
\begin{align*}
\theta_k = \frac{N_k^t+\alpha}{N + K\alpha} \rightarrow \frac{N_k^t}{N}
\end{align*} when $\alpha \rightarrow 0$. We can thus derive the recursive equation in Eq.\eqref{eq:recursive_theta}.
\end{proof}

\section{Results on the AP Dataset}
In this section, we will show the attack performence of different algorithms on AP dataset. Without otherwise specified, the main parameters are the same as the experiments on the NIPS dataset. 

\paragraph{Summary of results.} Overall, the patterns of the results for the AP dataset are very similar to that for the NIPS dataset. In terms of effectiveness (CoR and CoPS), the  margins of EvaLDA against the baselines in promotion attack are smaller, while in demotion attack the margins are larger. In terms of evasiveness (AWD and ASSD), the performances of EvaLDA are very close to that of the NIPS dataset. \textbf{This shows that the performance of EvaLDA is stable across datasets.}

\subsection{Promotion Attack}
For the AP dataset, the different experiment settings for promotion attack are the same as before:i) Approximation level $t$ (default 4); ii) Word substituion strategies (default mix); iii) Perturbation threshold $\kappa$, which ranges over $[1\%, 2\%, 3\%, 4\%]$ (default $\kappa=2\%$).
(iv) Original rank of target topic, which ranges over $[4, 6, 8, 10]$ (default $6$). 
As shown in Table 1 in the main text, documents in the AP data are much shorter (192 on average) than the NIPS dataset (3211 on average), using a perturbation threshold of $1\%$ means for some documents whose length is less than 200, only one word can be perturbed. This makes it too challenging for the attack algorithms including EvaLDA. Therefore we set larger default perturbation thresholds for the experiments on the AP dataset. Similarly, because the documents are much shorter in the AP dataset, the number of topics with non-zero probability scores are also much smaller than that in the NIPS dataset. Therefore, the original rank of the target topic is set higher and denser than that in the NIPS dataset.
We still randomly choose $50$ samples and set word distance threshold $\sigma = 0.6$.

\paragraph{Approximation level $t$}
Figure \ref{ldac_diff_level} shows the results with different $t$ values from $1$ to $6$. For all the metrics, similar patterns are observed compared with experiments on the NIPS dataset. 
Note that we set the perturbation threshold as $2\%$ perturbation, original rank $6$ and mixed word substitution strategy in this setting.
Based on this set of experiments, we choose $t=4$ in the following experiments.

\paragraph{Word Substitution Strategiesl}
Table \ref{tab:ldac_diff_similar} shows the result when varying word substitution strategies. Similarly, EM has better performance than Syno. A mix of the two has the best performance in CoR and CoPS, despite that AWD and ASSD of mix are almost the same as EM.

\paragraph{Perturbation Threshold} 
Figure \ref{ldac_diff_thre} shows performance comparisons of EvaLDA w.r.t. baselines with varying perturbation threshold. In terms of effectiveness (CoR and CoPS), EvaLDA is superior to the baseline methods under almost all threshold values. Note that \textbf{B3} also performs well under different perturbation thresholds. This is because \textbf{B3} selects a random replacement word for the target word and it is more likely to shift the target topic. However, such attacks are impractical as it is very detectable due to the large word-level and sentence variations, which are indicated by the very large AWD and ASSD values.

\paragraph{Original Ranking of Target Topic}
As shown in Figure \ref{ldac_diff_rank}, the overall results are similar to those of the NIPS dataset. The lower ranked topics are easier to promote, even if topics with higher rankings can get more score promotion, its ranks are still relatively stable. As we can see, EvaLDA performs better in CoR and CoPS, and also competitive in AWD and ASSD.

\paragraph{Original Ranking of Target Topic}
Figure \ref{ldac_demotion} shows the topic demotion attack performance of different methods under different orginal rank of the target topic. In demotion attack, we can see that EvaLDA beats the baselines by a very large margin for CoR and CoPS, while having the best AWD and ASSD values as well, which are even smaller than \textbf{B4}. 

\begin{figure*}
	\centering
	\subfigure{
		\begin{minipage}[t]{0.49\columnwidth}
			\centering
			\includegraphics[width=\textwidth]{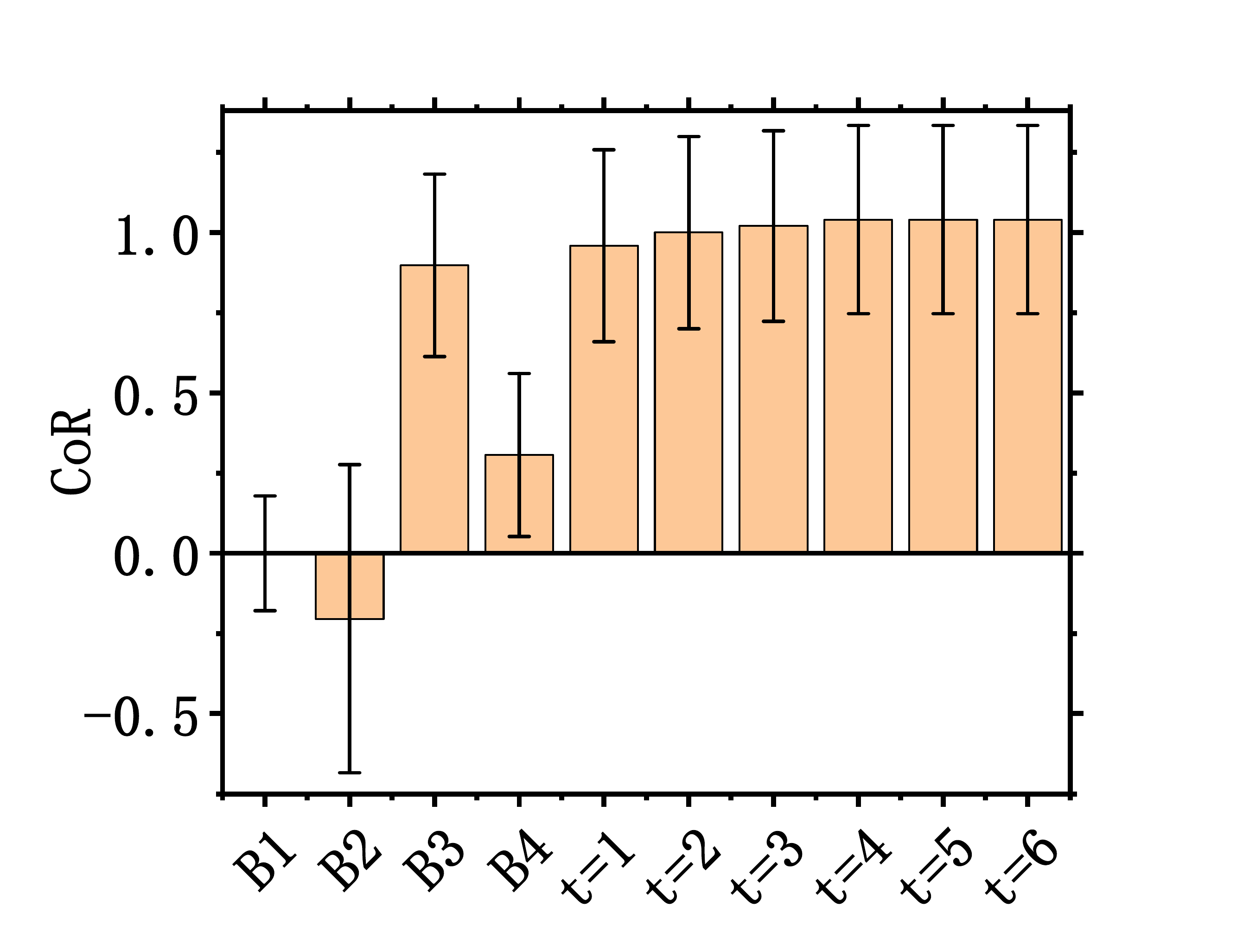}
		\end{minipage}%
	}%
	\subfigure{
		\begin{minipage}[t]{0.49\columnwidth}
			\centering
			\includegraphics[width=\textwidth]{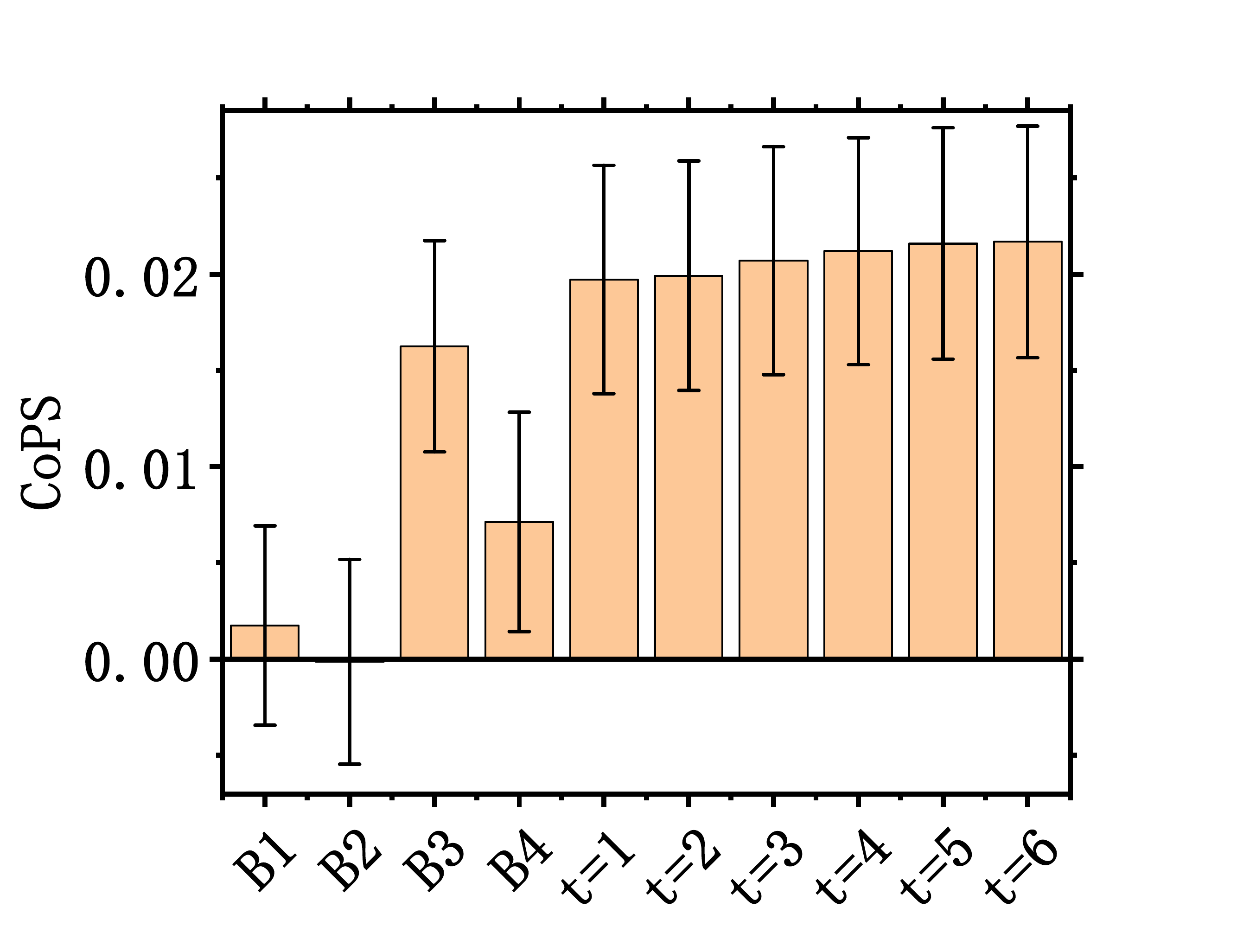}
		\end{minipage}%
	}%
	\subfigure{
		\begin{minipage}[t]{0.49\columnwidth}
			\centering
			\includegraphics[width=\textwidth]{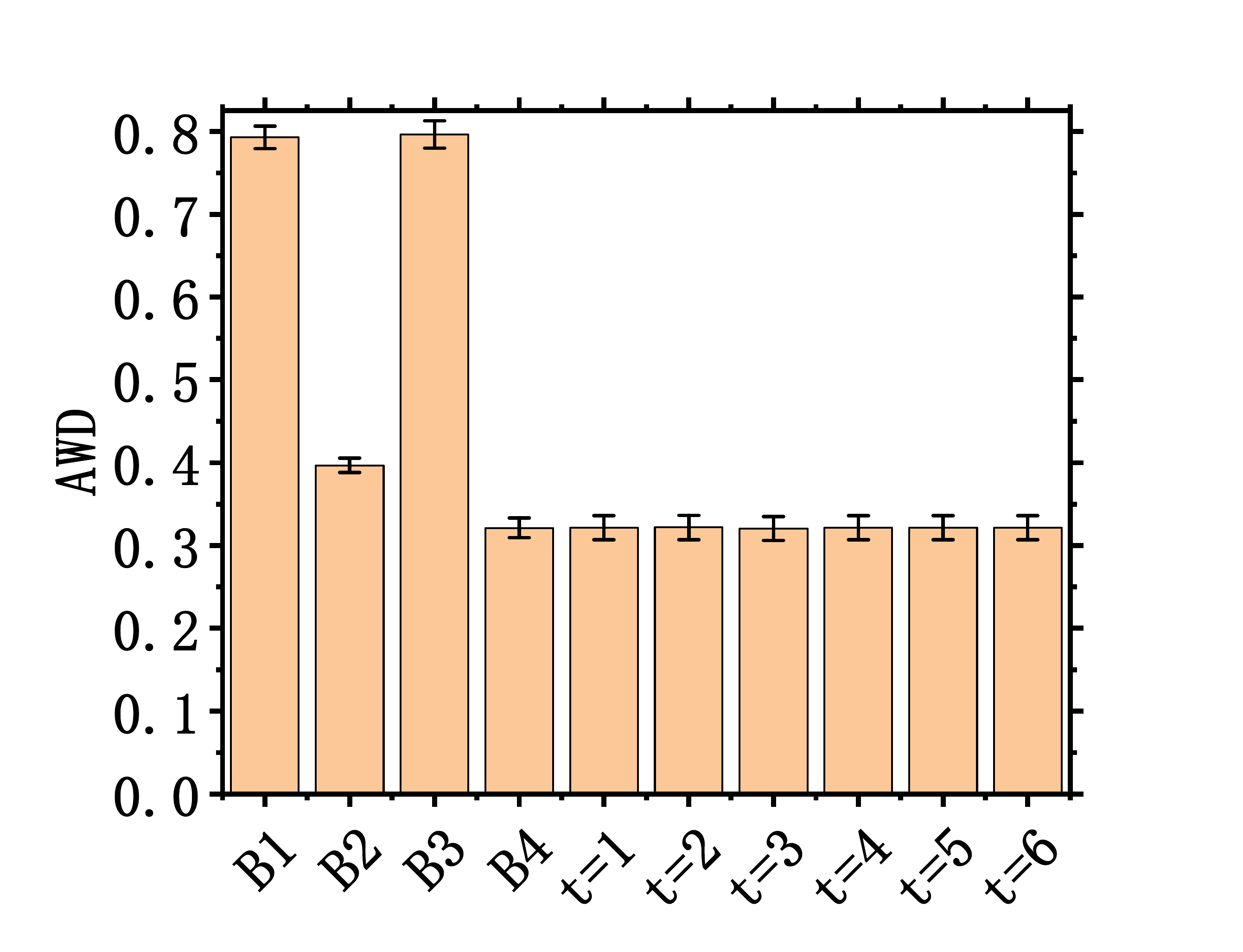}
		\end{minipage}
	}%
	\subfigure{
		\begin{minipage}[t]{0.49\columnwidth}
			\centering
			\includegraphics[width=\textwidth]{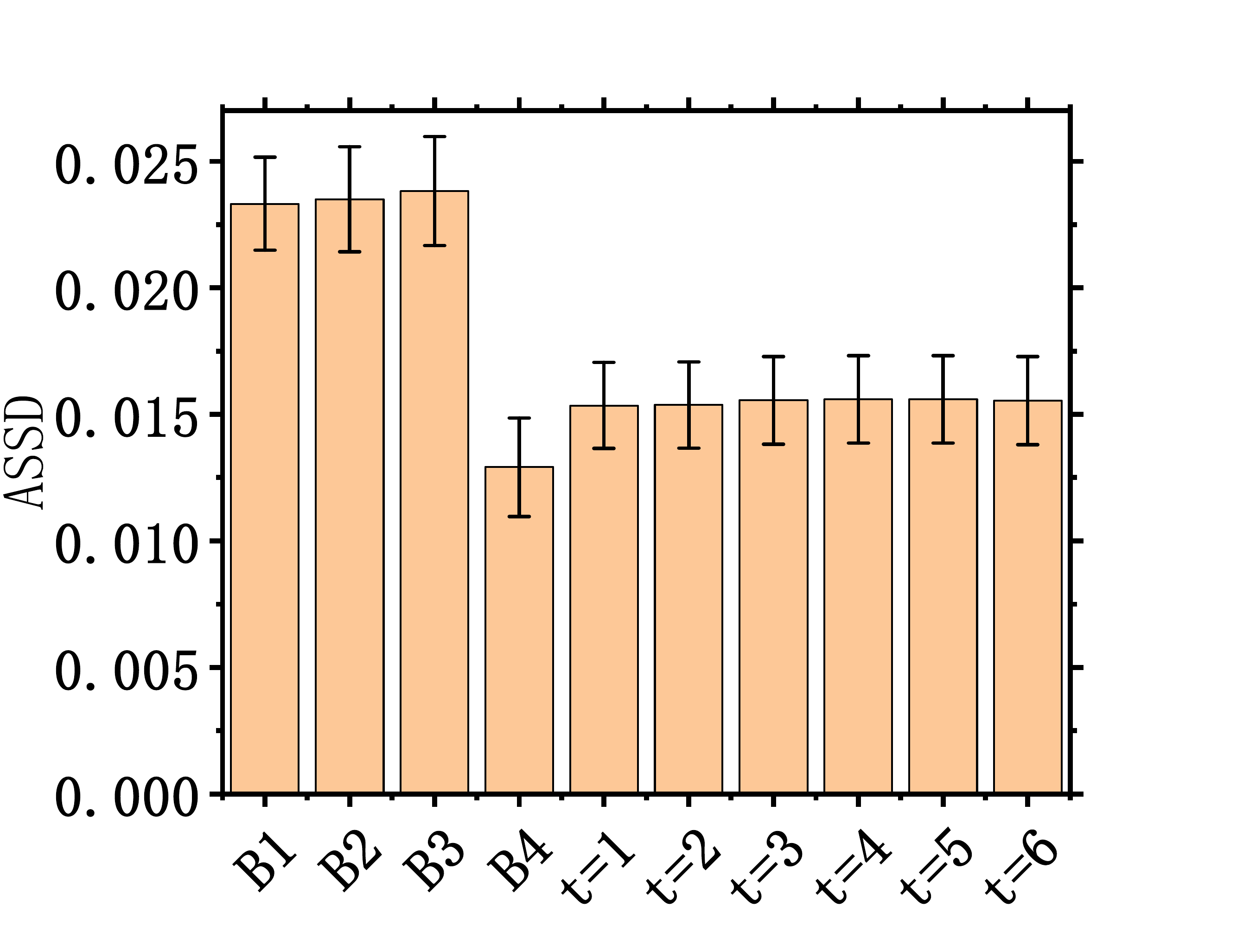}
		\end{minipage}
	}%
	
	\centering
	\caption{Promotion attack with varying approximate levels, on the AP dataset, showing 95\% confidence interval. }
	\label{ldac_diff_level}
\end{figure*}

\begin{figure*}
	\centering
	\subfigure{
		\begin{minipage}[t]{0.49\columnwidth}
			\centering
			\includegraphics[width=\textwidth]{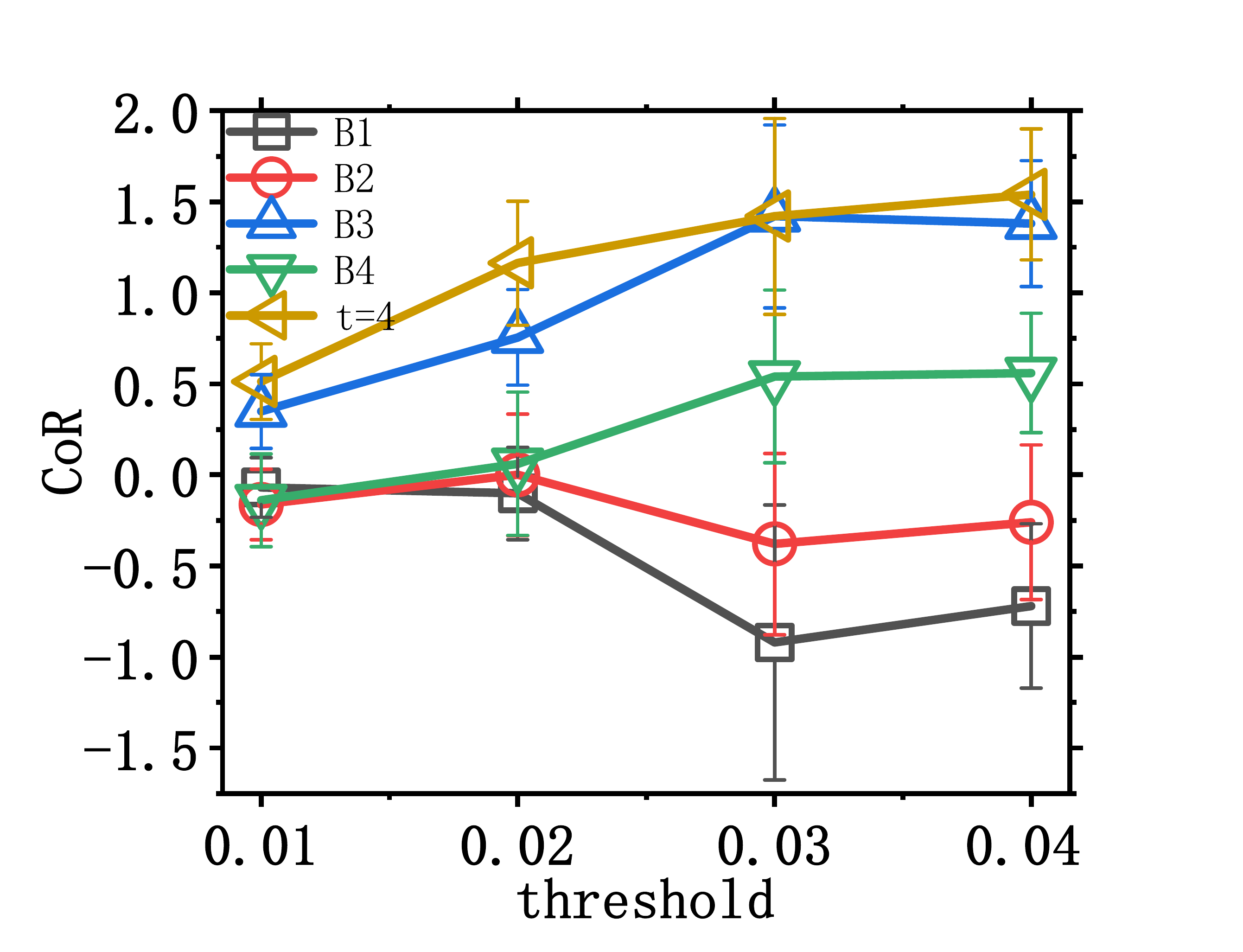}
		\end{minipage}%
	}%
	\subfigure{
		\begin{minipage}[t]{0.49\columnwidth}
			\centering
			\includegraphics[width=\textwidth]{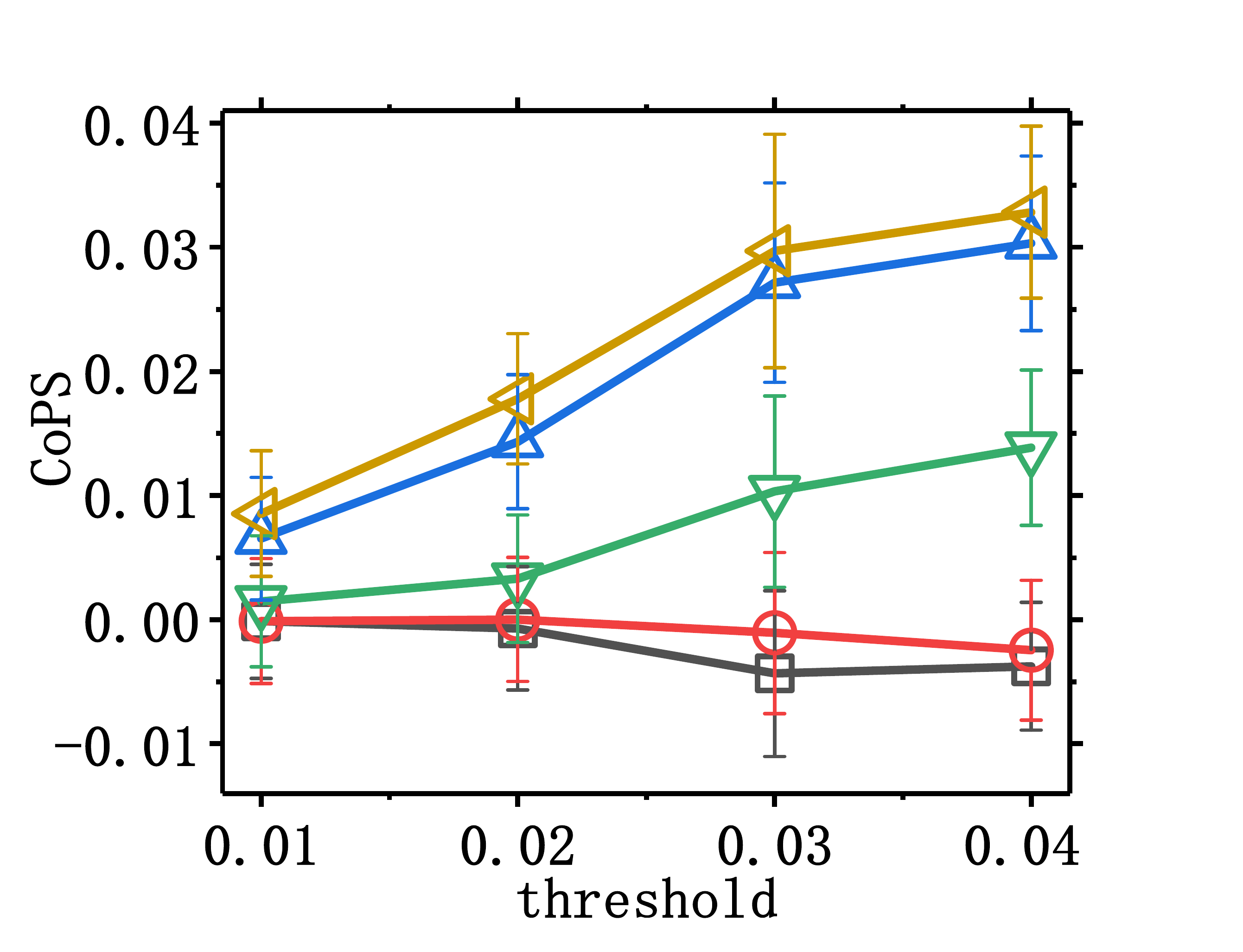}
		\end{minipage}%
	}%
	\subfigure{
		\begin{minipage}[t]{0.49\columnwidth}
			\centering
			\includegraphics[width=\textwidth]{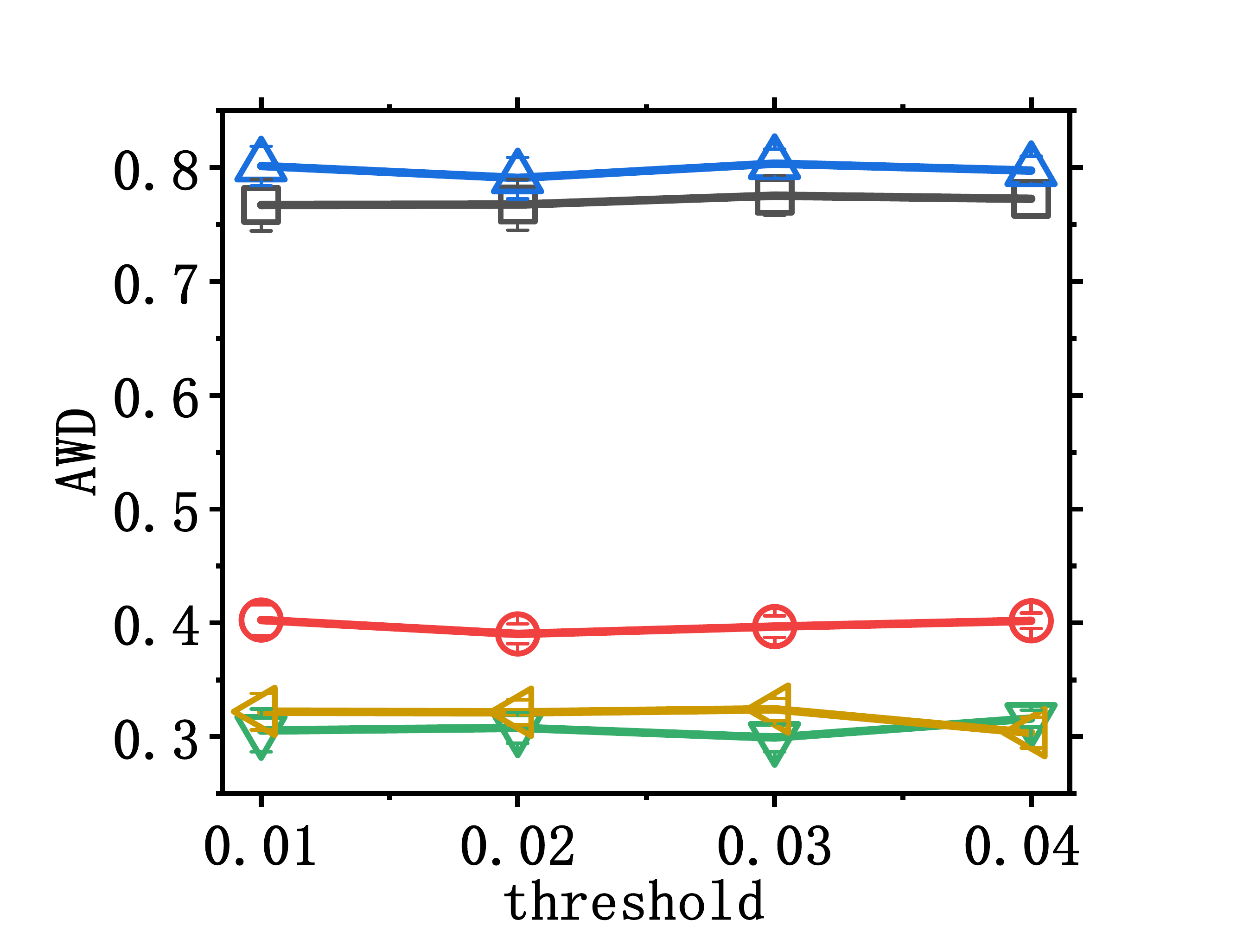}
		\end{minipage}
	}%
	\subfigure{
		\begin{minipage}[t]{0.49\columnwidth}
			\centering
			\includegraphics[width=\textwidth]{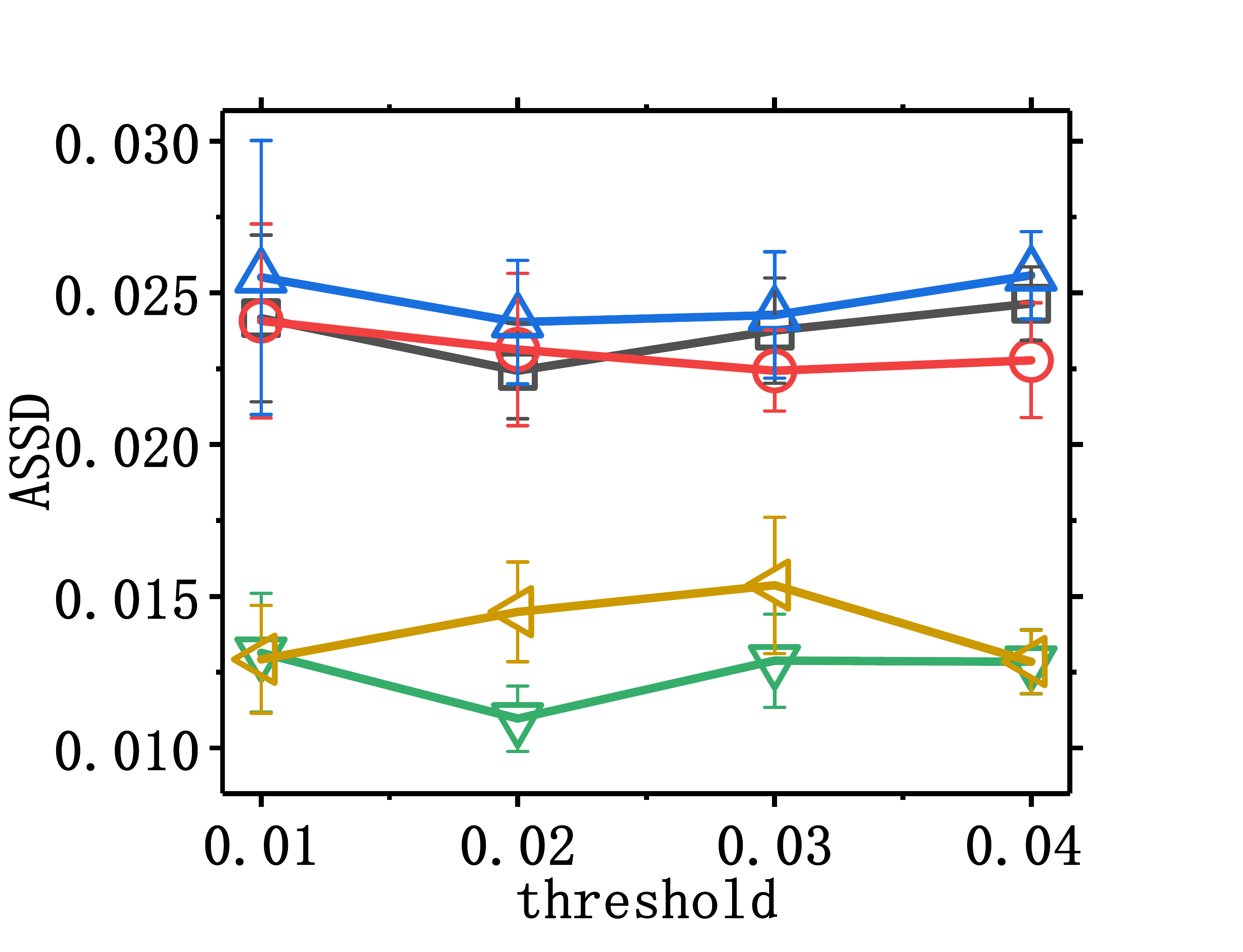}
		\end{minipage}
	}%
	
	\centering
	\caption{Promotion attack with varying perturbation threshold, on the AP dataset, showing 95\% confidence interval.}
	\label{ldac_diff_thre}
\end{figure*}
	
\begin{figure*}
	\centering
	\subfigure{
		\begin{minipage}[t]{0.49\columnwidth}
			\centering
			\includegraphics[width=\textwidth]{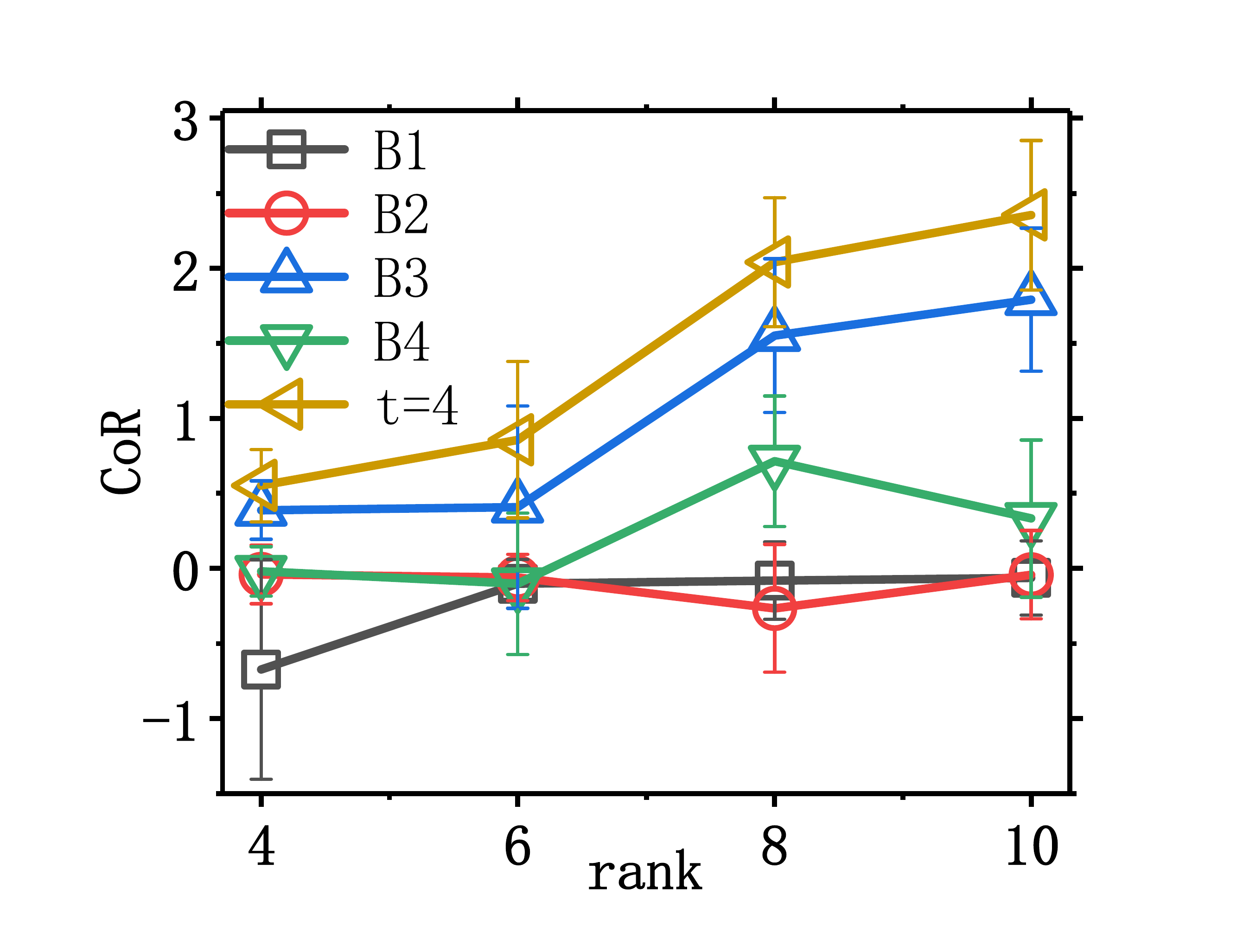}
		\end{minipage}%
	}%
	\subfigure{
		\begin{minipage}[t]{0.49\columnwidth}
			\centering
			\includegraphics[width=\textwidth]{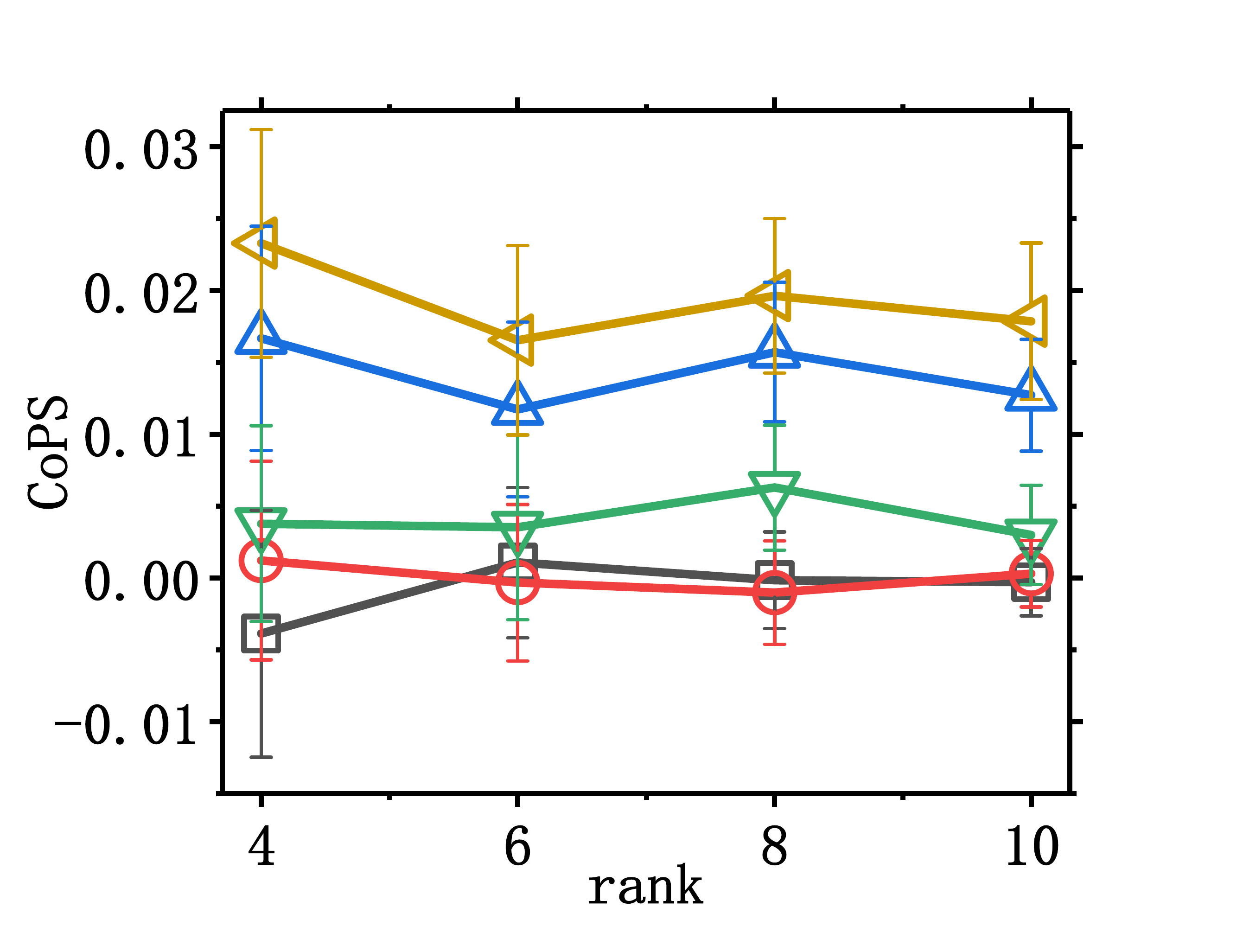}
		\end{minipage}%
	}%
	\subfigure{
		\begin{minipage}[t]{0.49\columnwidth}
			\centering
			\includegraphics[width=\textwidth]{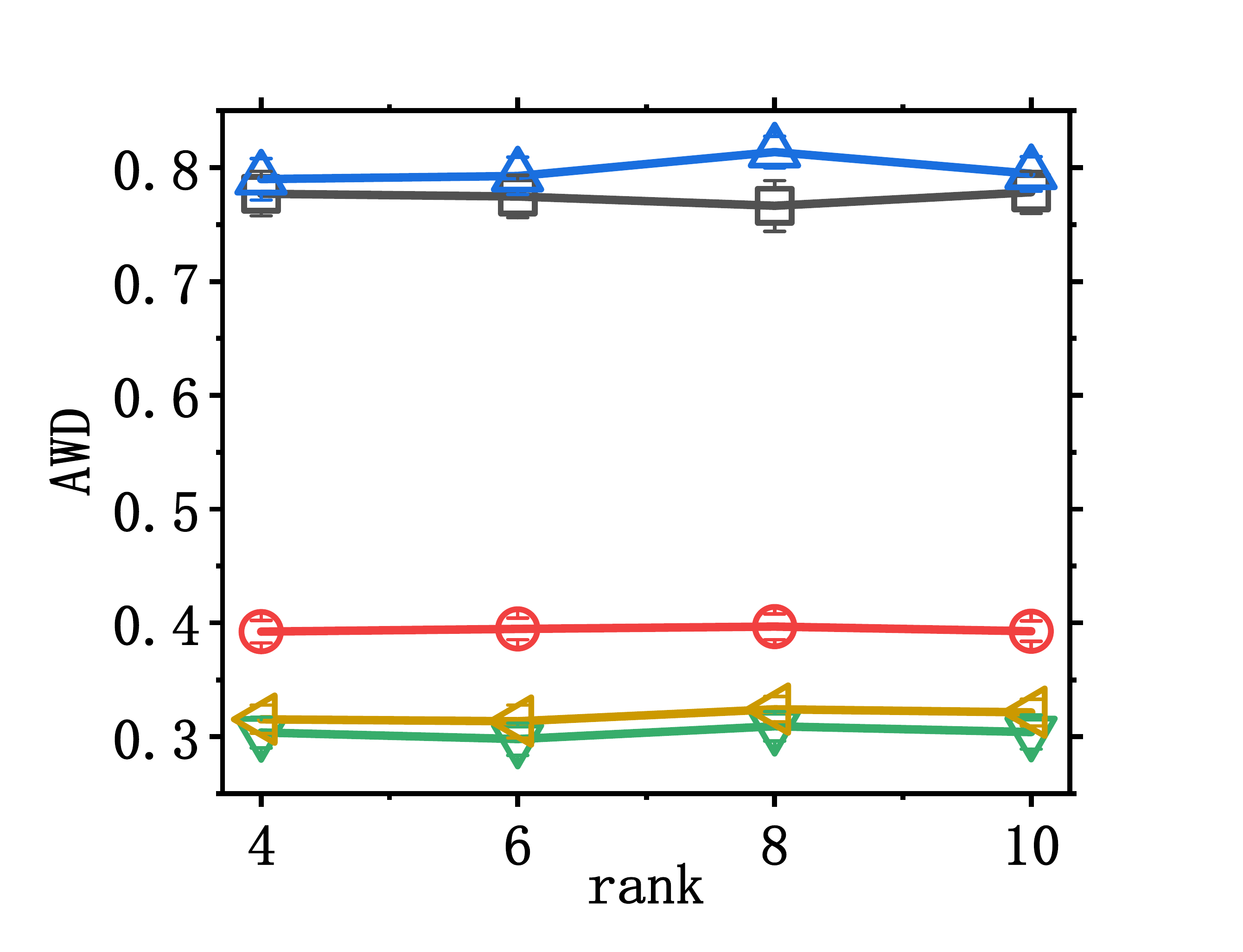}
		\end{minipage}
	}%
	\subfigure{
		\begin{minipage}[t]{0.49\columnwidth}
			\centering
			\includegraphics[width=\textwidth]{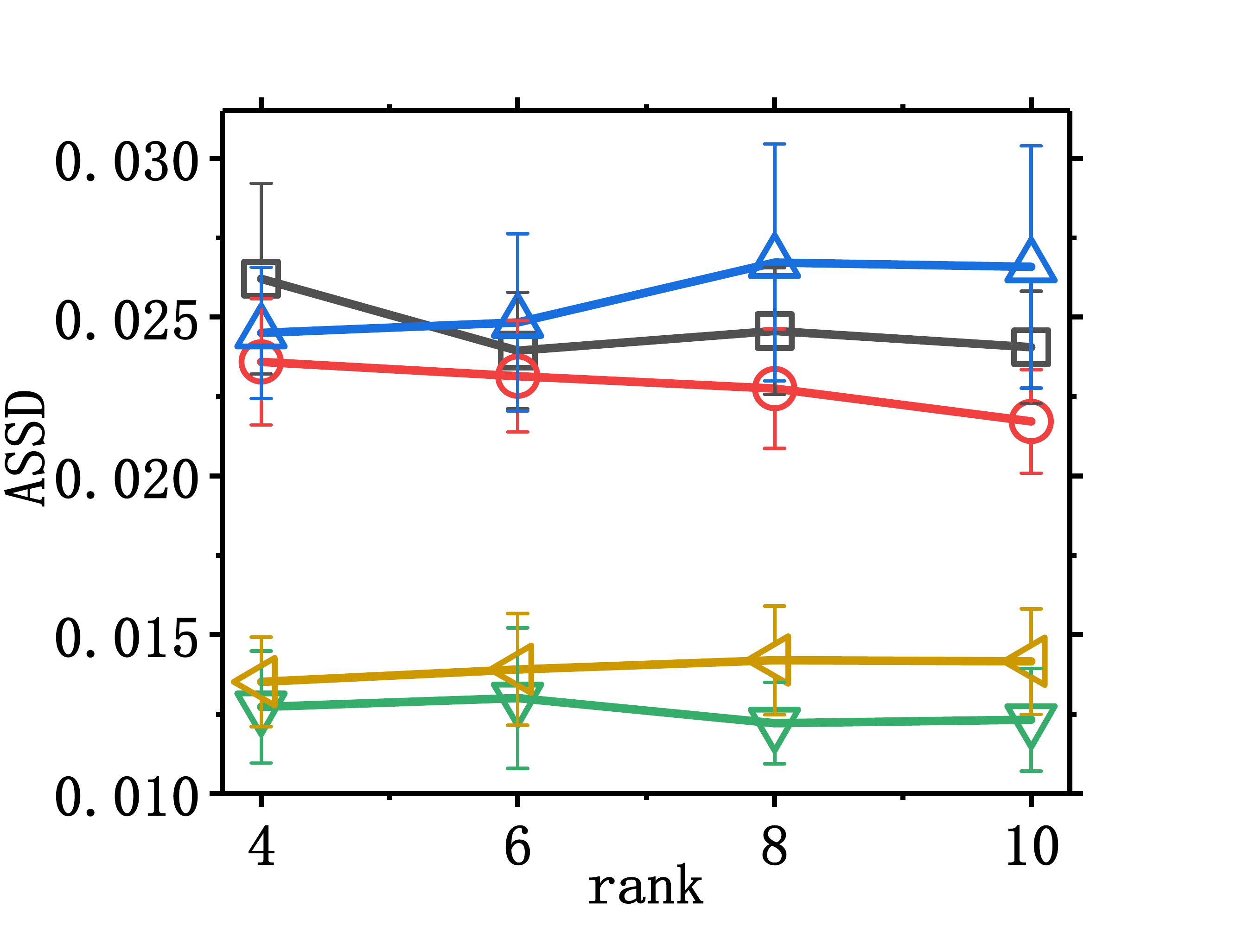}
		\end{minipage}
	}%
	
	\centering
	\caption{Promotion attack with varying original rank of the target topic, on the AP dataset, showing 95\% confidence interval.}
	\label{ldac_diff_rank}
\end{figure*}

\begin{figure*}
	\centering
	\subfigure{
		\begin{minipage}[t]{0.49\columnwidth}
			\centering
			\includegraphics[width=\textwidth]{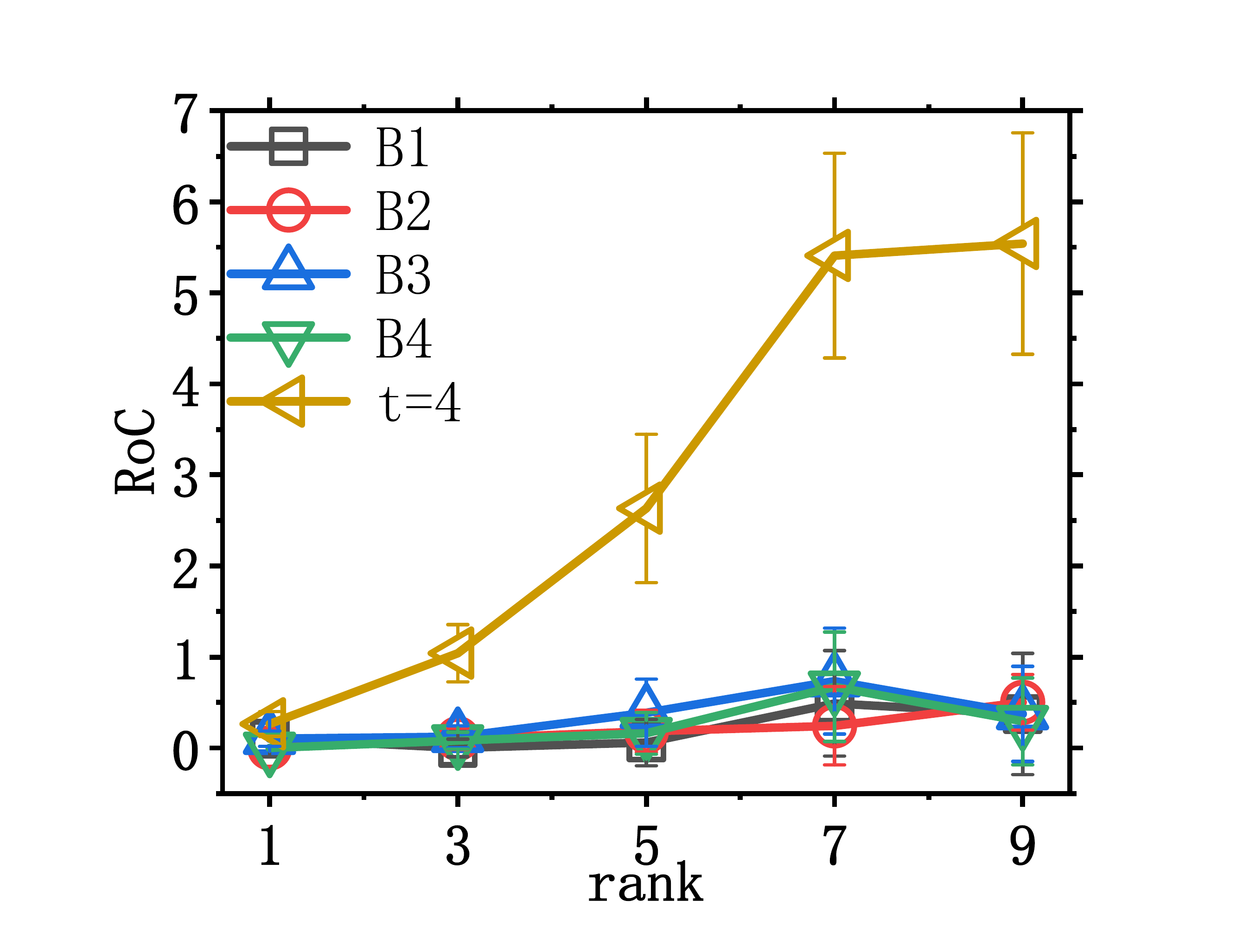}
		\end{minipage}%
	}%
	\subfigure{
		\begin{minipage}[t]{0.49\columnwidth}
			\centering
			\includegraphics[width=\textwidth]{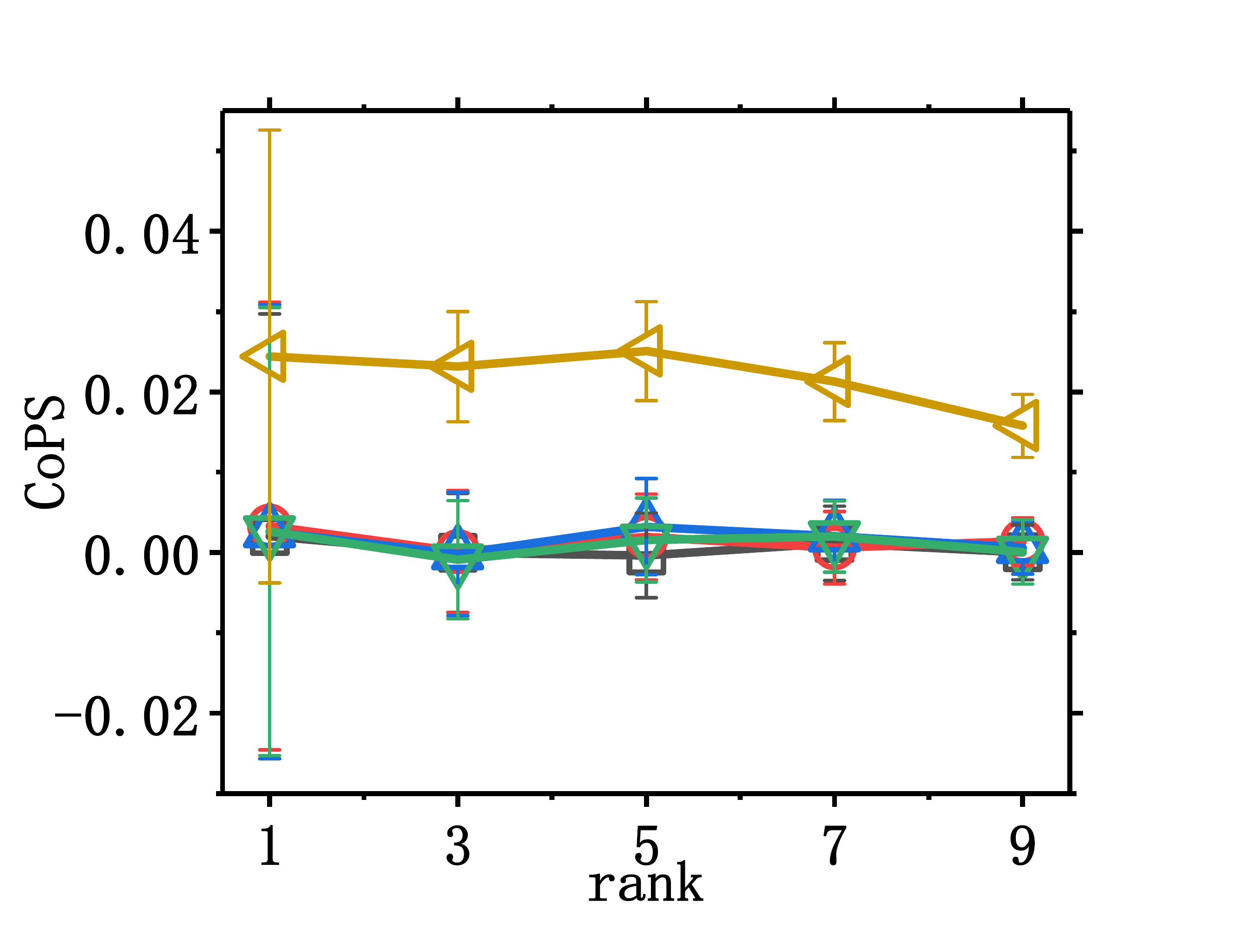}
		\end{minipage}%
	}%
	\subfigure{
		\begin{minipage}[t]{0.49\columnwidth}
			\centering
			\includegraphics[width=\textwidth]{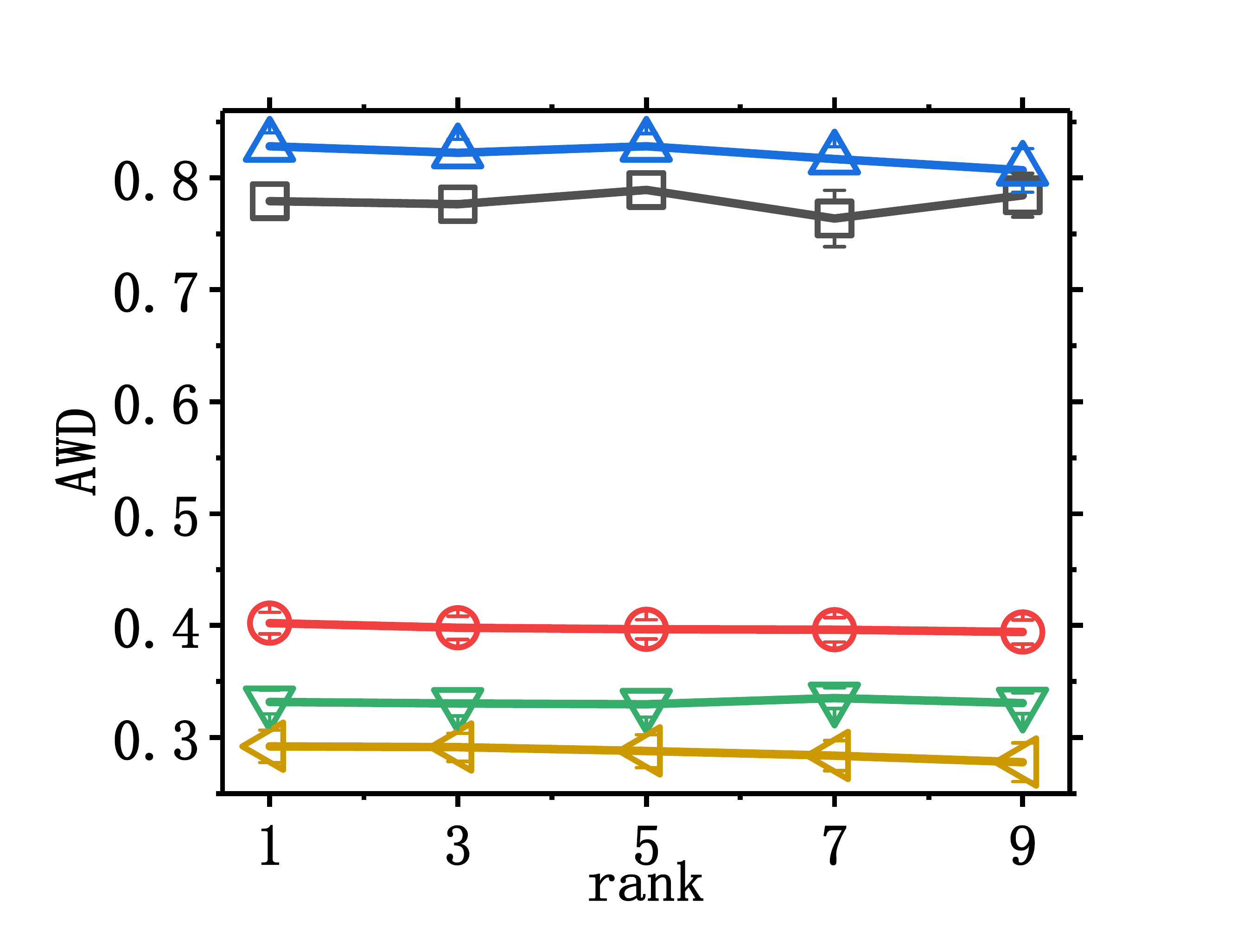}
		\end{minipage}
	}%
	\subfigure{
		\begin{minipage}[t]{0.49\columnwidth}
			\centering
			\includegraphics[width=\textwidth]{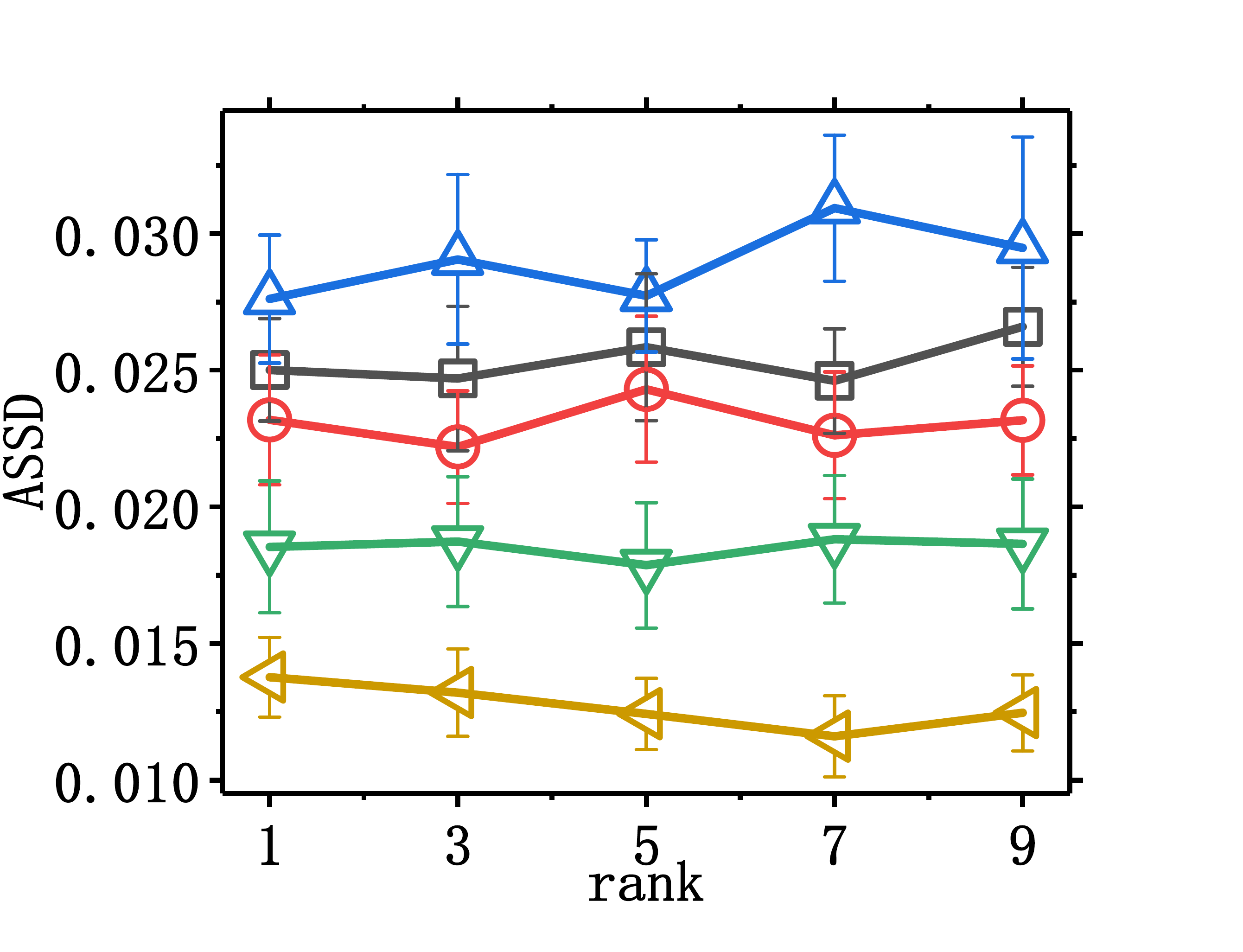}
		\end{minipage}
	}%
	
	\centering
	\caption{Demotion attack with varying original rank of target topic, on the AP dataset, showing 95\% confidence interval.}
	\label{ldac_demotion}
\end{figure*}

\begin{table}[htp]
\caption{The result of different word replacement strategies on the AP dataset. EM, Syno, and Mix respectively stand for embedding, synonyms using WordNet and a mixture of the two. The best result is highlighted in bold fonts.}\label{tab:ldac_diff_similar}
\centering\scriptsize
\begin{tabular}{lllll}
\hline
 & CoR & CoPS & AWD & ASSD \\ \hline
EM & 0.8913$\pm$0.2642 & 0.0225$\pm$0.0073 & 0.3147$\pm$0.0176 & 0.0142$\pm$0.0017 \\
Syno & 0.4348$\pm$0.2459 & 0.0116$\pm$0.0077 & \textbf{0.2602$\pm$0.0183} & \textbf{0.0104$\pm$0.0013} \\
Mix & \textbf{0.9348$\pm$0.2722} & \textbf{0.0232$\pm$0.0075} & 0.3149$\pm$0.0180 & 0.0143$\pm$0.0017 \\ \hline
\end{tabular}
\end{table}

\bibstyle{aaai21}
\bibliography{lda_attack}